%% file: OnlineMaxNorm_arxiv.tex
\documentclass[oneside,11pt]{article} % For LaTeX2e
\renewenvironment{abstract}
  {{\centering\large\bfseries Abstract\par}\vspace{0.7ex}%
    \bgroup
       \leftskip 20pt\rightskip 20pt\small\noindent\ignorespaces}%
  {\par\egroup\vskip 0.25ex}

\newenvironment{keywords}
{\bgroup\leftskip 20pt\rightskip 20pt \small\noindent{\bfseries
Keywords:} \ignorespaces}%
{\par\egroup\vskip 0.25ex}

\usepackage{times}
\usepackage{algorithm,algorithmic}
\usepackage{xcolor}
\usepackage{url}
\usepackage{graphicx,subfig}
\usepackage{appendix}
\usepackage{enumitem}
\usepackage{amsfonts,amssymb,mathtools,amsthm}

\usepackage[colorlinks,
            linkcolor=blue,
            citecolor=blue,
            urlcolor=magenta,
            linktocpage,
            plainpages=false]{hyperref}

\numberwithin{equation}{section}

\theoremstyle{plain}
\newtheorem{theorem}{Theorem}
\newtheorem{lemma}[theorem]{Lemma}
\newtheorem{proposition}[theorem]{Proposition}
\newtheorem{corollary}[theorem]{Corollary}
\theoremstyle{definition}

\newtheorem{remark}[theorem]{Remark}

\renewcommand{\(}{\left(}
\renewcommand{\)}{\right)}

\newcommand{\lV}{\left\Vert}
\newcommand{\rV}{\right\Vert}
\newcommand{\lv}{\left\vert}
\newcommand{\rv}{\right\vert}

\DeclareMathOperator*{\argmin}{arg\,min}

\DeclareMathOperator*{\tr}{Tr}

\DeclareMathOperator*{\st}{\mathrm{s.t.}}

\newcommand{\EXP}{\mathbb{E}}
\newcommand{\trans}{^{\top}}
\newcommand{\defeq}{\stackrel{\text{def}}{=}}

\newcommand{\bx}{\boldsymbol{x}}
\newcommand{\bz}{\boldsymbol{z}}
\newcommand{\be}{\boldsymbol{e}}
\newcommand{\bl}{\boldsymbol{l}}
\newcommand{\br}{\boldsymbol{r}}
\newcommand{\bu}{\boldsymbol{u}}
\newcommand{\bv}{\boldsymbol{v}}
\newcommand{\ba}{\boldsymbol{a}}
\newcommand{\bb}{\boldsymbol{b}}
\newcommand{\bm}{\boldsymbol{m}}
\newcommand{\Rq}{\mathbb{R}^q}
\newcommand{\Rp}{\mathbb{R}^p}

\newcommand{\Rpd}{\mathbb{R}^{p\times d}}

\newcommand{\Rnd}{\mathbb{R}^{n\times d}}
\newcommand{\Rpn}{\mathbb{R}^{p\times n}}

\newcommand{\twonorm}[1]{\left\lVert #1 \right\rVert_{2}}
\newcommand{\twoonenorm}[1]{\left\lVert #1 \right\rVert_{2,1}}
\newcommand{\onenorm}[1]{\left\lVert #1 \right\rVert_{1}}
\newcommand{\twoinfnorm}[1]{\left\lVert #1 \right\rVert_{2,\infty}}
\newcommand{\maxnorm}[1]{\left\lVert #1 \right\rVert_{\max}}
\newcommand{\fronorm}[1]{\left\lVert #1 \right\rVert_{F}}

\newcommand{\fractwo}[1]{\frac{#1}{2}}
\newcommand{\lambdat}{ \frac{\lambda_1}{2t}}

\newcommand{\tl}{\tilde{\ell}}
\renewcommand{\th}{\tilde{h}}

\newcommand{\tildel}{\frac{1}{2} \twonorm{ \bz - L \br - \be }^2 + \lambda_2 \th(\be)}
\newcommand{\tildeli}{\frac{1}{2} \twonorm{ \bz_i - L \br_i - \be_i }^2 + \lambda_2 \th(\be_i)}

\newcommand{\hg}{\hat{g}}

\newcommand{\tinf}{{2, \infty}}

\graphicspath{{fig/}}

\begin{document}

\title{\vspace{-0.5in}Online Optimization for Large-Scale Max-Norm Regularization}

%\author{\name Jie Shen \email js2007@rutgers.edu \\
%       \addr Department of Computer Science\\
%       Rutgers University\\
%       Piscataway, NJ 08854, USA\\
%       \AND
%       \name Huan Xu \email mpexuh@nus.edu.sg \\
%       \addr Department of Mechanical Engineering\\
%       National University of Singapore\\
%       Singapore 117575, Singapore\\
%       \AND
%       \name Ping Li \email pingli@stat.rutgers.edu\\
%       \addr Department of Statistics and Biostatistics\\
%       Department of Computer Science\\
%       Rutgers University\\
%       Piscataway, NJ 08854, USA\\
%       }
%
%\editor{}

\author{
{\bf Jie Shen}\\
Dept. of Computer Science\\
Rutgers University\\
Piscataway, NJ 08854, USA\\
\texttt{js2007@rutgers.edu}
\and
{\bf Huan Xu}\\
Dept. of Industrial \& Sys. Engineering\\
National University of Singapore\\
Singapore 117576, Singapore\\
\texttt{isexuh@nus.edu.sg}
\and
{\bf Ping Li}\\
Dept. of Statistics \& Biostatistics\\
Department of Computer Science\\
Rutgers University\\
Piscataway, NJ 08854, USA\\
\texttt{pingli@stat.rutgers.edu}
}
\date{}

\maketitle

\begin{abstract}
Max-norm regularizer has been extensively studied in the last decade as it promotes an effective low-rank estimation for the underlying data. However, such max-norm regularized problems are typically formulated and solved in a batch manner, which prevents it from processing big data due to  possible memory budget. In this paper, hence, we propose an online algorithm that is scalable to large-scale setting. Particularly, we consider the matrix decomposition problem as an example, although a simple variant of the algorithm and analysis can be adapted to other important problems such as matrix completion. The crucial technique in our implementation is to reformulating the max-norm to an equivalent matrix factorization form, where the factors consist of a (possibly overcomplete) basis component and a coefficients one. In this way, we may maintain the basis component in the memory and optimize over it and the coefficients for each sample alternatively. Since the memory footprint of the basis component is independent of the sample size, our algorithm is appealing when manipulating a large collection of samples. We prove that the sequence of the solutions (i.e., the basis component) produced by our algorithm converges to a stationary point of the expected loss function asymptotically. Numerical study demonstrates encouraging results for the efficacy and robustness of our algorithm compared to the widely used nuclear norm solvers.

\vspace{0.05in}

\end{abstract}

\section{Introduction}
%\subsection{Motivation}
%\subsubsection{why use max-norm}
In the last decade, estimating low-rank matrices has attracted increasing attention in the machine learning community owing to its  successful applications in a wide range of fields including subspace clustering~\cite{liu2010lrr}, collaborative filtering~\cite{srebro2012matrix} and robust dimensionality reduction~\cite{candes2011rpca}, to name a few. Suppose that we are given an observed data matrix $Z$ in $\Rpn$, i.e., $n$ observations in $p$ ambient dimensions, we aim to learn a prediction matrix $X$ with a low-rank structure so as to approximate the observation. This problem, together with its many variants, typically involves minimizing a weighted combination of the residual error and a penalty for the matrix rank.

Generally speaking, it is intractable to optimize a matrix rank~\cite{recht2010guaranteed}. To tackle this challenge, researchers suggested alternative convex relaxations to the matrix rank. The two most widely used convex surrogates are the nuclear norm\footnote{Also known as the trace norm, the Ky-Fan $n$-norm and the Schatten $1$-norm.}~\cite{recht2010guaranteed} and the max-norm (a.k.a. $\gamma_2$-norm)~\cite{srebro2004mmmf}. The nuclear norm is defined as the sum of the matrix singular values. Like the $\ell_1$ norm in the vector case that induces sparsity, the nuclear norm was proposed as a rank minimization heuristic and was able to be formulated as a semi-definite programming (SDP) problem~\cite{fazel2001rank}. By combining the SDP formulation and the matrix factorization technique,~\cite{srebro2004mmmf} showed that the collaborative filtering problem can be effectively solved by optimizing a soft margin based program. Another interesting work of the nuclear norm comes from the data compression community. In real-world applications, due to possible sensor failure and background clutter, the underlying data can be easily corrupted. In this case, estimation produced by Principal Component Analysis (PCA) may be deviated far from the true subspace~\cite{jolliffe2005principal}. To handle the (gross) corruption, in the seminal work of~\cite{candes2011rpca}, Cand{\`{e}}s et al. proposed a new formulation called Robust PCA (RPCA), and proved that under mild conditions, solving a convex optimization problem consisting of a nuclear norm regularization and a weighted $\ell_1$ norm penalty can {exactly} recover the low-rank component of the underlying data even if a constant fraction of the entries are arbitrarily corrupted. Notably, they also provided a range of the trade-off parameter which guarantees the exact recovery.

The max-norm variant was developed as another convex relaxation to the rank function~\cite{srebro2004mmmf}, where Srebro et al. formulated the max-norm regularized problem as an SDP and empirically showed the superiority to the nuclear norm. The main theoretical study on the max-norm comes from~\cite{srebro2005colt}, where Srebro and Shraibman considered collaborative filtering as an example and proved that the max-norm schema enjoys a lower generalization error than the nuclear norm. Following these theoretical foundations,~\cite{srebro2012clustering} improved the error bound for the clustering problem. Another important contribution from~\cite{srebro2012clustering} is that they partially characterized the subgradient of the max-norm, which is a hard mathematical entity and cannot be fully understood to date. However, since SDP solver is not scalable, there is a large gap between the theoretical progress and the practical applicability of the max-norm. To bridge the gap, a number of follow-up works attempted to design efficient algorithms to solve max-norm regularized or constrained problems. For example,~\cite{srebro2005fast} devised a gradient-based optimization method and empirically showed promising results on large collaborative filtering datasets.~\cite{srebro2010practical} presented large-scale optimization methods for max-norm constrained and max-norm regularized problems and showed a convergence to stationary point.

%\subsubsection{why in an online manner}
Nevertheless, algorithms presented in prior works~\cite{srebro2004mmmf,srebro2005fast,srebro2010practical,srebro2012prisma} require to access all the data when the objective function involves a max-norm regularization. In the large-scale setting, the applicability of such batch optimization methods will be hindered by the memory bottleneck. In this paper, henceforth, we propose an online algorithm to solve max-norm regularized problems. The main advantage of online algorithms is that the memory cost is independent of the sample size, which makes it a good fit for the \emph{big data} era.

To be more detailed, we are interested in a general max-norm regularized matrix decomposition (MRMD) problem. Assume that the observed data matrix $Z$ can be decomposed into a low-rank component $X$ and some structured noise $E$, we aim to simultaneously and accurately estimate the two components, by solving the following convex program:
\begin{equation}
\text{(MRMD)}\quad  \min_{X, E}\quad \fractwo{1} \fronorm{Z-X-E}^2 + \fractwo{\lambda_1} \maxnorm{X}^2 +  \lambda_2 h(E).
  \label{eq:primal}
\end{equation}
Here, $\fronorm{\cdot}$ denotes the Frobenius norm which is a commonly used metric for evaluating the residual, $\maxnorm{\cdot}$ is the max-norm (which promotes low-rankness), and $\lambda_1$ and $\lambda_2$ are two non-negative parameters. $h(E)$ is some (convex) regularizer that can be adapted to various kinds of noise. We require that it can be represented as a summation of column norms. Formally, there exists some regularizer $\th(\cdot)$, such that
\begin{equation}
\label{eq:e req}
h(E) = \sum_{i=1}^{n} \th(\be_i),
\end{equation}
where $\be_i$ is the $i$th column of $E$. Admissible examples include:
\begin{itemize}
\item $\onenorm{E}$. That is, the $\ell_1$ norm of the matrix $E$ seen as a long vector, which is used to handle sparse corruption. In this case, $\th(\cdot)$ is the $\ell_1$ vector norm. Note that when equipped with this norm, the above problem reduces to the well-known RPCA formulation~\cite{candes2011rpca}, but with the nuclear norm replaced by the max-norm.

%\item $\fronorm{E}^2$: This is a popular regularizer when the data matrix contains Gaussian noise. For this case, $\th(\cdot)$ is the squared $\ell_2$ norm.

\item $\twoonenorm{E}$. This is defined as the summation of the $\ell_2$ column norms, which is effective when a small fraction of the samples are contaminated (recall that each column of $Z$ is a sample). Here, $\th(\cdot)$ is the $\ell_2$ norm. The matrix $\ell_{2,1}$ norm is typically used to handle outliers and interestingly, the above program becomes Outlier PCA~\cite{xu2013outlier} in this case.

\item $\fronorm{E}^2$ or $E=0$. The formulation of~\eqref{eq:primal} works as a large margin based program, with the hinge loss replaced by the squared loss~\cite{srebro2004mmmf}.
\end{itemize}
Hence, (MRMD)~\eqref{eq:primal} is general enough and our algorithmic and theoretical results hold for such general form, covering important problems including max-norm regularized RPCA, max-norm regularized Outlier PCA and large maximum margin matrix factorization. Furthermore, with a careful design, the above formulation~\eqref{eq:primal} can be extended to address the matrix completion problem~\cite{candes2009exact}, as we will show in Section~\ref{sec:mc}.

\subsection{Contributions}
In summary, our {main contributions} are two-fold: {\bf 1)} We are the first to develop an online algorithm to solve a family of max-norm regularized problems~\eqref{eq:primal}, which finds a wide range of applications in machine learning. We also show that our approach can be used to solve other popular max-norm regularized problems such as matrix completion. {\bf 2)} We prove that the solutions produced by our algorithm converge to a stationary point of the expected loss function asymptotically (see Section~\ref{sec:main results}). 

Compared to our earlier work~\cite{shen2014online}, the formulation~\eqref{eq:primal} considered here is more general and a complete proof is provided. In addition, we illustrate by an extensive study on the subspace recovery task to confirm the conjecture that the max-norm always performs better than the nuclear norm in terms of convergence rate and robustness.

\subsection{Related Works}
\label{sec:related work}
Here we discuss some relevant works in the literature. Most previous works on max-norm focused on showing that it is empirically  superior to the nuclear norm in real-world problems, such as collaborative filtering~\cite{srebro2004mmmf}, clustering~\cite{srebro2012clustering} and hamming embedding~\cite{srebro2014hamming}. Other works, for instance,~\cite{srebro2010non-uniform}, studied the influence of data distribution with the max-norm regularization and observed good performance even when the data are sampled non-uniformly. There are also interesting works which investigated the connection between the max-norm and the nuclear norm. A comprehensive study on this problem, in the context of collaborative filtering, can be found in~\cite{srebro2005colt}, which established and compared the generalization bound for the nuclear norm regularization and the max-norm, showing that the latter one results in a tighter bound. More recently,~\cite{srebro2012matrix} attempted to unify them to gain insightful perspective.

Also in line with this work is matrix decomposition. As we mentioned, when we penalize the noise $E$ with $\ell_1$ matrix norm, it reverts to the well known RPCA formulation~\cite{candes2011rpca}. The only difference is that~\cite{candes2011rpca} analyzed the RPCA problem with the nuclear norm, while~\eqref{eq:primal} employs the max-norm. Owing to the explicit form of the subgradient of the nuclear norm,~\cite{candes2011rpca} established a dual certificate for the success of their formulation, which facilitates their theoretical analysis. In contrast, the max-norm is a much harder mathematical entity (even its subgradient has not been fully characterized). Henceforth, it still remains challenging to understand the behavior of the max-norm regularizer in the general setting~\eqref{eq:primal}. Studying the conditions for the exact recovery of MRMD is out of the scope of this paper. We leave this as a future work.

From a high level, the goal of this paper is similar to that of~\cite{feng2013online}. Motivated by the celebrated RPCA problem~\cite{candes2011rpca,xu2013outlier,xu2012robust},~\cite{feng2013online} developed an online implementation for the nuclear-norm  regularized matrix decomposition. Yet, since the max-norm is a more complicated mathematical entity, new techniques and insights are needed in order to develop online methods for the max-norm regularization. {For example, after converting the max-norm to its matrix factorization form, the data are still coupled and we propose to transform the problem to a constrained one for stochastic optimization.}

The main technical contribution of this paper is converting  max-norm regularization to an appropriate matrix factorization problem amenable to online implementation. Compared to~\cite{mairal2010online} which also studies online matrix factorization, our formulation contains an additional structured noise that brings the benefit of robustness to contamination. Some of our proof techniques are also different. For example, to prove the convergence of the dictionary and to well define their problem,~\cite{mairal2010online} assumed that the magnitude of the learned dictionary is constrained. In contrast, we {\em prove} that the optimal basis is uniformly bounded, and hence our problem is naturally well-defined.

\subsection{Roadmap}
The rest of the paper is organized as follows. Section~\ref{sec:setup} begins with some basic notation and problem definition, followed by reformulating the MRMD problem which turns out to be amenable for online optimization. Section~\ref{sec:algorithm} then elaborates the online implementation of MRMD and Section~\ref{sec:main results} establishes the convergence guarantee under some mild assumptions. In Section~\ref{sec:mc}, we show that our framework can easily be extended to other max-norm regularized problems, such as matrix completion. Numerical performance of the proposed algorithm is presented in Section~\ref{sec:exp}. Finally, we conclude this paper in Section~\ref{sec:conclude}. All the proofs are deferred to the appendix.

\section{Problem Setup}\label{sec:setup}
{\bf Notation.} \ 
We use lower bold letters to denote vectors. The $\ell_1$ norm and $\ell_2$ norm of a vector $\bv$ are denoted by $\onenorm{\bv}$ and $\twonorm{\bv}$, respectively. Capital letters, such as $M$, are used to denote matrices. In particular, the letter $I_n$ is reserved for the identity matrix with the size of $n$ by $n$. For a matrix $M$, the $i$th row and $j$th column are written as $\bm(i)$ and $\bm_j$ respectively, and the $(i, j)$-th entry is denoted by $M_{ij}$. There are four matrix norms that will be heavily used in the paper: $\fronorm{M}$ for the Frobenius norm, $\onenorm{M}$ for the $\ell_1$ matrix norm seen as a long vector, $\maxnorm{M}$ for the max-norm induced by the product of $\ell_{\tinf}$ norm on the factors of $M$. Here, the $\ell_{\tinf}$ norm is defined as the maximum $\ell_2$ row norm. The trace of a square matrix $M$ is denoted as $\tr(M)$. Finally, for a positive integer $n$, we use $[n]$ to denote the integer set $\{1, 2, \cdots, n\}$.

We are interested in developing an online algorithm for the MRMD problem~\eqref{eq:primal}. To this end, we note that the max-norm~\cite{srebro2004mmmf} is defined as follows:
\begin{equation}
\label{eq:max-def}
\maxnorm{X} \defeq \min_{L, R}\ \Big\{ \twoinfnorm{L} \cdot \twoinfnorm{R}:\ X = L R\trans, L \in \Rpd, R \in \Rnd \Big\},
\end{equation}
where $d$ is an upper bound on the intrinsic dimension of the underlying data. Plugging the above into~\eqref{eq:primal}, we obtain an equivalent form:
\begin{equation}
  \min_{L, R, E}\quad \fractwo{1} \fronorm{Z-LR\trans -E}^2 + \fractwo{\lambda_1} \twoinfnorm{L}^2 \twoinfnorm{R}^2 + \lambda_2 h(E).
  \label{eq:min_lre}
\end{equation}
In this paper, if not specified, ``equivalent'' means we do not change the optimal value of the objective function. Intuitively, the variable $L$ serves as a (possibly overcomplete) basis for the clean data while correspondingly, the variable $R$ works as a coefficients matrix with each row being the coefficients for each sample (recall that we organize the observed samples in a column-wise manner). In order to make the new formulation~\eqref{eq:min_lre} equivalent to MRMD~\eqref{eq:primal}, the quantity of $d$ should be sufficiently large due to~\eqref{eq:max-def}.

\vspace{0.05in}
\noindent{\bf Challenge.} \ 
At a first sight, the problem can only be optimized in a batch manner for which the memory cost is prohibitive. To see this, note that we are considering the regime of $d < p \ll n$. Hence, the basis component $L$ is eligible for memory storage since its size is independent of the number of samples. As $E$ is column-wisely regularized (see~\eqref{eq:e req}), we are able to update each column of $E$ by only accessing one sample. Nevertheless, the size of the coefficients $R$ is proportional to $n$. In order to optimize the above program over the variable $R$, we have to compute the gradient with respect to it. Recall that the $\ell_{\tinf}$ norm counts the largest $\ell_2$ row norm of $R$, hence coupling all the samples (each row of $R$ associates with a sample). 

Fortunately, we have the following proposition that alleviates the inter-dependency among the rows of $R$, hence facilitating an online algorithm where the rows of $R$ can be optimized sequentially.

\begin{proposition}
\label{prop:constrained}
Problem~\eqref{eq:min_lre} is equivalent to the following constrained program:
\begin{equation}
\begin{split}
  \min_{{L}, {R}, E}&\quad \fractwo{1} \fronorm{Z-LR\trans-E} + \fractwo{\lambda_1} \twoinfnorm{L}^2 + \lambda_2 h(E),\\
  \st&\quad \twoinfnorm{R}^2 \leq 1.
  \end{split}
  \label{eq:batch-max}
\end{equation}
Moreover, there exists an optimal solution $(L^*, R^*, E^*)$ attained at the boundary of the feasible set, i.e., $\twoinfnorm{R^*}^2$ is equal to the unit.
\end{proposition}
\begin{proof}
Let us denote $k = \twoinfnorm{R}$. We presume $k$ is positive. Otherwise, the low-rank component $X$ we aim to recover is a zero matrix, which is of little interest. Now we construct two auxiliary variables $\bar{L} = kL \in \Rpd$ and $\bar{R} = \frac{1}{k} R \in \Rnd$. Replacing $L$ and $R$ with $\frac{1}{k}\bar{L}$ and $k\bar{R}$ in~\eqref{eq:min_lre} respectively, we have:
\begin{align*}
  \min_{\bar{L}, \bar{R}, E}\ \fractwo{1} \fronorm{Z - \(\frac{1}{k}\bar{L}\)\(k\bar{R}\)\trans - E}^2 + \fractwo{\lambda_1} \twoinfnorm{\frac{1}{k}\bar{L}}^2 \twoinfnorm{k\bar{R}}^2 + \lambda_2 h(E).
\end{align*}
That is, we are to solve
\begin{align*}
  \min_{\bar{L}, \bar{R}, E}\ \fractwo{1} \fronorm{ Z - \bar{L}\bar{R}\trans - E}^2 + \fractwo{\lambda_1} \twoinfnorm{\bar{L}}^2 \twoinfnorm{\bar{R}}^2 + \lambda_2 h(E).
\end{align*}
The fact that $\bar{R} = \frac{1}{k} R$ and $k$ is the maximum of the $\ell_2$ row norm of $R$ implies $\twoinfnorm{\bar{R}} = 1$. Therefore, we can reformulate our MRMD problem as a constrained program:
\begin{align*}
  \min_{\bar{L}, \bar{R}, E}\ \fractwo{1} \fronorm{Z - \bar{L}\bar{R}\trans - E}^2 + \fractwo{\lambda_1} \twoinfnorm{ \bar{L}}^2 + \lambda_2 h(E),\quad  \st\ \twoinfnorm{\bar{R}}^2 = 1.
\end{align*}
To see why the above program is equivalent to~\eqref{eq:batch-max}, we only need to show that each optimal solutions $(L^*, R^*, E^*)$ of~\eqref{eq:batch-max} must satisfy $\twoinfnorm{R^*}^2 = 1$. Suppose that $k = \twoinfnorm{R^*} < 1$. Let $L' = kL^*$ and $R' = \frac{1}{k}R^*$. Obviously, $(L', R', E^*)$ are still feasible. However, the objective value becomes
\begin{align*}
\begin{split}
&\ \fractwo{1} \fronorm{Z - L'R^{'\top} - E^*}^2 + \fractwo{\lambda_1}\twoinfnorm{L'}^2 + \lambda_2 h(E^*)\\
=&\ \fractwo{1} \fronorm{Z - L^*R^{*\top} - E^*}^2 +  \fractwo{\lambda_1} \cdot k^2 \twoinfnorm{L^*}^2 + \lambda_2 h(E^*) \\
<&\ \fractwo{1} \fronorm{Z - L^*R^{*\top} - E^*}^2 + \fractwo{\lambda_1} \twoinfnorm{L^*}^2 + \lambda_2 h(E^*),
\end{split}
\end{align*}
which contradicts the assumption that $(L^*, R^*, E^*)$ is the optimal solution. Thus we complete the proof.
\end{proof}

\begin{remark}
Proposition~\ref{prop:constrained} is crucial for the online implementation. It states that our primal MRMD problem~\eqref{eq:primal} can be transformed to an equivalent constrained program~\eqref{eq:batch-max} where the coefficients of {\em each individual} sample (i.e., a row of the matrix $R$) is \emph{uniformly and separately} constrained.
\end{remark}

Consequently, we can, equipped with Proposition~\ref{prop:constrained}, rewrite the original problem in an online fashion, with each sample being separately processed:
\begin{align}
  \min_{L, R, E}\ \fractwo{1} \sum_{i=1}^n \twonorm{\bz_i - L\mathbf{r}_i - \be_i}^2 + \fractwo{\lambda_1} \twoinfnorm{L}^2 + \lambda_2 \sum_{i=1}^n \th(\be_i),\quad  \st\ \twonorm{\br_i}^2 \leq 1,\ \forall\ i \in [n],
\end{align}
where $\bz_i$ is the $i$th observation, $\br_i$ is the coefficients and $\be_i$ is some structured error penalized by the (convex) regularizer $\th(\cdot)$ (recall that we require $h(E)$ can be decomposed column-wisely). Merging the first and third term above gives a compact form:
\begin{equation}
\begin{split}
  \min_{L}\ \min_{R,E}&\quad  \sum_{i=1}^n \tilde{\ell}(\bz_i, L, \br_i, \be_i) + \fractwo{\lambda_1} \twoinfnorm{L}^2,\\
  \st&\quad \twonorm{\br_i}^2 \leq 1,\ \forall i \in [n],
  \end{split}
\end{equation}
where
\begin{equation}
\label{eq:tildel}
\tilde{\ell}(\bz, L, \br, \be) \defeq \tildel.
\end{equation}
This is indeed equivalent to optimizing (i.e., minimizing) the empirical loss function:
\begin{equation}
\min_L\ f_n(L),
\label{eq:min f_n(L)}
\end{equation}
where
\begin{equation}
f_n(L) \defeq \frac{1}{n} \sum_{i=1}^n \ell( \bz_i, L) + \frac{\lambda_1}{2n} \twoinfnorm{L}^2,
\label{eq:f_n(L)}
\end{equation}
and
\begin{equation}
  \ell(\bz, L ) = \min_{\br, \be, \twonorm{\br}^2 \leq 1} \tilde{\ell}(\bz, L, \br, \be).\\
  \label{eq:l(z,L)}
\end{equation}
Note that by Proposition~\ref{prop:constrained}, as long as the quantity of $d$ is sufficiently large, the program~\eqref{eq:min f_n(L)} is equivalent to the primal formulation~\eqref{eq:primal}, in the sense that both of them could attain the same minimum. Compared to MRMD~\eqref{eq:primal}, which is solved in a batch manner by prior works, the formulation~\eqref{eq:min f_n(L)} paves a way for stochastic optimization procedure since all the samples are decoupled.

\section{Algorithm}
\label{sec:algorithm}
Based on the derivation in the preceding section, we are now ready to present our online algorithm to solve the MRMD problem~\eqref{eq:primal}. The implementation is outlined in Algorithm~\ref{alg:all}. Here we first briefly explain the underlying intuition. We optimize the coefficients $\br$, the structured noise $\be$ and the basis $L$ in an alternating manner, with only the basis $L$ and two accumulation matrices being kept in memory. At the $t$-th iteration, given the basis $L_{t-1}$ produced by the previous iteration, we can optimize~\eqref{eq:l(z,L)} by examining the Karush Kuhn Tucker (KKT) conditions. To obtain a new iterate $L_t$, we then minimize the following objective function:
\begin{equation}
  g_t(L) \defeq \frac{1}{t}\sum_{i=1}^t \tilde{\ell}(\bz_i, L, \br_i, \be_i) + \frac{\lambda_1}{2t} \twoinfnorm{L}^2,
\label{eq:g_t(L)}
\end{equation}
where $\{\br_i\}_{i=1}^t$ and $\{\be_i\}_{i=1}^t$ are already on hand. It can be verified that~\eqref{eq:g_t(L)} is a surrogate function of the empirical loss $f_t(L)$~\eqref{eq:f_n(L)}, since the obtained $\br_i$'s and $\be_i$'s are suboptimal. Interestingly, instead of recording all the past $\br_i$'s and $\be_i$'s, we only need to store two accumulation matrices whose sizes are independent of $n$, as shown in Algorithm~\ref{alg:all}. In the sequel, we elaborate each step in Algorithm~\ref{alg:all}.

\begin{algorithm}[t]
\caption{Online Max-Norm Regularized Matrix Decomposition}
\label{alg:all}
\begin{algorithmic}[1]
    \REQUIRE $Z \in \Rpn$ (observed samples), parameters $\lambda_1$ and $\lambda_2$, $L_0 \in \Rpd$ (initial basis), zero matrices $A_0 \in \mathbb{R}^{d \times d}$ and $B_0 \in \mathbb{R}^{p \times d}$.
    \ENSURE Optimal basis $L_n$.
    \FOR{$t=1$ to $n$}
      \STATE Access the $t$-th sample $\bz_t$.
      \STATE Compute the coefficient and noise:
        \begin{align*}
          \{\br_t, \be_t \} = \argmin_{\br, \be, \twonorm{\br}^2 \leq 1} \tl(\bz_t, L_{t-1}, \br, \be).
        \end{align*}
      \STATE Compute the accumulation matrices $A_t$ and $B_t$:
        \begin{align*}
          A_t \longleftarrow&\ A_{t-1} + \br_t^{} \br_t\trans,\\
          B_t \longleftarrow&\ B_{t-1} + \(\bz_t^{} - \be_t^{} \)\br_t\trans.
        \end{align*}

      \STATE Compute the basis $L_t$ by optimizing the surrogate function~\eqref{eq:g_t(L)}:
        \begin{align*}
        \begin{split}
          L_t &= \argmin_L \frac{1}{t}\sum_{i=1}^t \tilde{\ell}(\bz_i, L, \br_i, \be_i) + \frac{\lambda_1}{2t} \twoinfnorm{L}^2 \\
          &= \argmin_L \frac{1}{t} \(\frac{1}{2} \tr \( L\trans L A_t \) - \tr \( L\trans B_t \) \) + \frac{\lambda_1}{2t} \twoinfnorm{L}^2.
          \end{split}
        \end{align*}
    \ENDFOR
\end{algorithmic}
\end{algorithm}

\subsection{Update the coefficients and noise}
Given a sample $\bz$ and a basis $L$, we are able to estimate the optimal coefficients $\br$ and noise $\be$ by minimizing $\tl(\bz, L, \br, \be)$. That is, we are to solve the following program:
\begin{equation}
\begin{split}
\min_{\br, \be}&\quad \fractwo{1} \twonorm{\bz - L \br - \be}^2 + \lambda_2 \th(\be),\\
\st&\quad \twonorm{\br} \leq 1.
\end{split}
\label{eq:solve_re}
\end{equation}
We notice that the constraint only involves the variable $\br$, and in order to optimize $\br$, we only need to consider the residual term in the objective function. This motivates us to employ a block coordinate descent algorithm. Namely, we alternatively optimize one variable with the other fixed, until some stopping criteria is fulfilled. In our implementation, when the difference between the current and the previous iterate is smaller than $10^{-6}$, or the number of iterations exceeds 100, our algorithm will stop and return the optima.

\subsubsection{Optimize the coefficients $\br$}
Now it remains to show how to compute a new iterate for one variable when the other one is fixed. According to~\cite{bertsekas1999nonlinear}, when the objective function is strongly convex with respect to (w.r.t.) each block variable, it can guarantee the convergence of the alternating minimization procedure. In our case, we observe that such condition holds for $\be$ but not necessary for $\br$. In fact, the strong convexity for $\br$ holds if and only if the basis $L$ is with full rank. When $L$ is not full rank, we may compute the Moore Penrose pseudo inverse to solve $\br$. However, for computational efficiency, we append a small jitter $\frac{\epsilon}{2}\twonorm{\br}^2$ to the objective if necessary, so as to guarantee the convergence ($\epsilon=0.01$ in our experiments). In this way, we obtain a \emph{potentially} admissible iterate for $\br$ as follows:
\begin{equation}
\br_0^{} = (L\trans L + \epsilon I_d )^{-1} L\trans (\bz - \be).
\label{eq:r_0}
\end{equation}
Here, $\epsilon$ is set to be zero if and only if $L$ is full rank.

Next, we examine if $\br_0^{}$ violates the inequality constraint in~\eqref{eq:solve_re}. If it happens to be a feasible solution, i.e., $\twonorm{\br_0^{}} \leq 1$, we have found the new iterate for $\br$. Otherwise, we conclude that the optima of $\br$ must be attained on the boundary of the feasible set, i.e., $\twonorm{\br}=1$, for which the minimizer can be found by the method of Lagrangian multipliers:
\begin{equation}
\begin{split}
\max_{\eta} \min_{\br}&\quad \fractwo{1} \twonorm{\bz - L \br - \be}^2 + \fractwo{\eta}\( \twonorm{\br}^2 - 1 \),\\
\st&\quad \eta > 0,\quad \twonorm{\br} = 1.
\end{split}
\label{eq:solve r eta}
\end{equation}
By differentiating the objective function with respect to $\br$, we have
\begin{equation}
\br = \( L\trans L + \eta I \)^{-1} L\trans (\bz - \be).
\label{eq:lagrange_r}
\end{equation}
In order to facilitate the computation, we make the following argument.

\begin{proposition}
\label{prop:r}
Let $\br$ be given by~\eqref{eq:lagrange_r}, where $L$, $\bz$ and $\be$ are assumed to be fixed. Then, the $\ell_2$ norm of $\br$ is strictly monotonically decreasing with respect to the quantity of $\eta$.
\end{proposition}
\begin{proof}
For simplicity, let us denote
\begin{align*}
\br(\eta) = \( L\trans L + \eta I \)^{-1} \bb,
\end{align*}
where $\bb = L\trans (\bz - \be)$ is a fixed vector. Suppose we have a full singular value decomposition (SVD) on $L=USV\trans$, where the singular values $\{s_1, s_2, \cdots, s_p\}$ (i.e., the diagonal elements in $S$) are arranged in a decreasing order and at most $d$ number of them are non-zero. Substituting $L$ with its SVD, we obtain the squared $\ell_2$ norm for $\br(\eta)$:
\begin{align*}
\twonorm{\br(\eta)}^2 =&\ \bb\trans \( VS^2V\trans + \eta I \)^{-2} \bb \\
=&\ \bb\trans V S_{\eta} V\trans \bb,
\end{align*}
where $S_{\eta}$ is a diagonal matrix whose $i$th diagonal element equals $(s_i^2 + \eta)^{-2}$.

For any two entities $\eta_1 > \eta_2$, it is easy to see that the matrix $S_{\eta_1} - S_{\eta_2}$ is negative definite. Hence, it always holds that
\begin{align*}
\twonorm{\br(\eta_1)}^2 - \twonorm{\br(\eta_2)}^2 = \bb\trans V (S_{\eta_1} - S_{\eta_2}) V\trans \bb < 0,
\end{align*}
which concludes the proof.
\end{proof}

The above proposition offers an efficient computational scheme, i.e., bisection method, for searching the optimal $\br$ as well as the dual variable $\eta$. To be more detailed, we can maintain a lower bound $\eta_1$ and an upper bound $\eta_2$, such that $\twonorm{\br(\eta_1)} \geq 1$ and $\twonorm{\br(\eta_2)} \leq 1$. According to the monotonic property shown in Proposition~\ref{prop:r}, the optimal $\eta$ must fall into the interval $[\eta_1, \eta_2]$. By evaluating the value of $\twonorm{\br}$ at the middle point $(\eta_1 + \eta_2)/2$, we can sequentially shrink the interval until $\twonorm{\br}$ is close to one. Note that we can initialize $\eta_1$ with zero (since $\twonorm{\br_0^{}}$ is larger than one implying the optimal $\eta^* > \epsilon \geq 0$). The bisection routine is summarized in Algorithm~\ref{alg:solve r eta}.

\begin{algorithm}[t]
\caption{Bisection Method for Problem~\eqref{eq:solve r eta}}
\label{alg:solve r eta}
\begin{algorithmic}[1]
\REQUIRE $L \in \Rpd$, $\bz \in \Rp$, $\be \in \Rp$.
\ENSURE Optimal primal and dual pair $(\br, \eta)$.
\STATE Initialize the lower bound $\eta_1 = 0$ and the upper bound $\eta_2$ large enough such that $\twonorm{\br(\eta_2)} \leq 1$.

\REPEAT
\STATE Compute the middle point:
\begin{align*}
\eta \leftarrow \frac{1}{2}(\eta_1 + \eta_2).
\end{align*}
  	
\IF{$\twonorm{\br(\eta)} < 1$}
\STATE Update $\eta_2$:
\begin{align*}
\eta_2 \leftarrow \eta.
\end{align*}
\ELSE
\STATE Update $\eta_1$:
\begin{align*}
\eta_1 \leftarrow {\eta}.
\end{align*}
\ENDIF
    
\UNTIL{$\twonorm{\br} = 1$}
\end{algorithmic}
\end{algorithm}

\subsubsection{Optimize the Noise $\be$}
We have clarified the technique used for solving $\br$ in Problem~\eqref{eq:solve_re} when $\be$ is fixed. Now let us turn to the phase where $\br$ is fixed and we want to find the optimal $\be$. Since $\be$ is an unconstrained variable, generally speaking, it is much easier to solve, although one may employ different strategies for various regularizers $\th(\cdot)$. Here, we discuss the solutions for popular choices of the regularizer.
\begin{enumerate}
\item $\th(\be) = \onenorm{\be}$. The $\ell_1$ regularizer results in a closed form solution for $\be$ as follows:
\begin{align}
\be = \mathcal{S}_{\lambda_2}[\bz - L\br],
\end{align}
where $\mathcal{S}_{\lambda_2}[\cdot]$ is the soft-thresholding operator~\cite{hale2008fixed}.

%\item $\th(\be) = \fractwo{1} \twonorm{\be}^2$. Again, we can derive a closed form solution for such squared $\ell_2$ regularizer:
%\begin{align*}
%\be = \frac{1}{\lambda_2 + 1} (\bz - L\br).
%\end{align*}

\item $\th(\be) = \twonorm{\be}$. The solution in this case can be characterized as follows (see, for example,~\cite{liu2010lrr}):
\begin{align}
\be =
\begin{cases}
\frac{\twonorm{\bz - L\br}}{\twonorm{\bz - L\br} - \lambda_2} (\bz - L\br),\quad &\text{if}\ \lambda_2 < \twonorm{\bz - L\br},\\
\mathbf{0},\quad &\text{otherwise}.
\end{cases}
\end{align}
\end{enumerate}

Finally, for completeness, we summarize the routine for updating the coefficients and the noise in Algorithm~\ref{alg:re}. The readers may refer to the preceding paragraphs for details.

\begin{algorithm}[t]
\caption{The Coefficients and Noise Update (Problem~\eqref{eq:solve_re})}
\label{alg:re}
\begin{algorithmic}[1]
    \REQUIRE $L \in \Rpd$, $\bz \in \Rp$, parameter $\lambda_2$ and a small jitter $\epsilon$.
    \ENSURE Optimal $\br$ and $\be$.
    \STATE Initialize $\be = \mathbf{0}$.
    \REPEAT
        \STATE Compute the potential solution $\br^{}_0$ given in~\eqref{eq:r_0}.
        
        \IF{$\lV \br^{}_0 \rV_2 \leq 1$}
        \STATE Update $\br$ with 
        	\begin{align*}
        	\br = \br^{}_0,
      	    \end{align*}
        \ELSE
        \STATE Update $\br$ by Algorithm~\ref{alg:solve r eta}.
        \ENDIF
            
        \STATE Update the noise $\be$.
    \UNTIL{convergence}
\end{algorithmic}
\end{algorithm}

\subsection{Update the basis}
With all the past filtration $\mathcal{F}_t = \{ \bz_i, \br_i, \be_i \}_{i=1}^{t}$ on hand, we are able to compute a new basis $L_t$ by minimizing the surrogate function~\eqref{eq:g_t(L)}. That is, we are to solve the following program:
\begin{equation}
\label{eq:solve g_t primal}
\min_L\quad \frac{1}{t}\sum_{i=1}^t \tl(\bz_i, L, \br_i, \be_i) + \frac{\lambda_1}{2t} \twoinfnorm{L}^2.
\end{equation}
By a simple expansion, for any $i \in [t]$, we have
\begin{align}
\tl(\bz_i, L, \br_i, \be_i) = \fractwo{1} \tr\( L\trans L \br_i \br_i\trans \) - \tr\( L\trans (\bz_i - \be_i ) \br_i\trans \) + \fractwo{1} \twonorm{\bz_i - \be_i }^2 + \lambda_2 \th(\be_i).
\end{align}
Substituting back into Problem~\eqref{eq:solve g_t primal}, putting $A_t = \sum_{i=1}^{t} \br_i \br_i\trans$, $B_t = \sum_{i=1}^{t} (\bz_i - \be_i ) \br_i\trans$ and removing constant terms, we obtain
\begin{equation}
\label{eq:alg:all:solve L}
L_t = \argmin_L \frac{1}{t} \(\frac{1}{2} \tr \( L\trans L A_t \) - \tr \( L\trans B_t \) \) + \frac{\lambda_1}{2t} \twoinfnorm{L}^2.
\end{equation}

In order to derive the optimal solution, firstly, we need to characterize the subgradient of the squared $\ell_{\tinf}$ norm. In fact, let $Q$ be a positive semi-definite diagonal matrix, such that $\tr(Q)=1$. Denote the set of row index which attains the maximum $\ell_2$ row norm of $L$ by $\mathcal{I}$. In this way, the subgradient of $\fractwo{1}\twoinfnorm{L}^2$ can be formalized as follows:
\begin{equation}
\label{eq:subgradient}
\partial \( \fractwo{1} \twoinfnorm{L}^2 \) = QL,\ Q_{ii} \neq 0\ \text{if and only if}\ i \in \mathcal{I},\ Q_{ij} = 0\ \text{for}\ i \neq j.
\end{equation}

Equipped with the subgradient, we may apply block coordinate descent to update each column of $L$ sequentially. We assume that the objective function~\eqref{eq:alg:all:solve L} is strongly convex w.r.t. $L$, implying that the block coordinate descent scheme can always converge to the global optimum~\cite{bertsekas1999nonlinear}.

We summarize the update procedure in Algorithm~\ref{alg:L}. In practice, we find that after revealing a large number of samples, performing one-pass updating for each column of $L$ is sufficient to guarantee a desirable accuracy, which matches the observation in~\cite{mairal2010online}.

\begin{algorithm}
\caption{The Basis Update}
\label{alg:L}
\begin{algorithmic}[1]
    \REQUIRE $L \in \Rpd$ in the previous iteration, accumulation matrix $A$ and $B$, parameter $\lambda_1$.
    \ENSURE Optimal basis $L$ (updated).
    \REPEAT
      \STATE Compute the subgradient of $\fractwo{1} \twoinfnorm{L}^2$:
      \begin{align*}
        U = \partial \(\fractwo{1}\twoinfnorm{L}^2 \).
      \end{align*}
      
      \FOR{$j = 1$ to $d$}
      \STATE Update the $j$th column:
      \begin{align*}
        \bl_j \leftarrow \bl_j - \frac{1}{A_{jj}} \( L \ba_j - \bb_j + \lambda_1 \bu_j \)
      \end{align*}
      \ENDFOR
    \UNTIL{convergence}
\end{algorithmic}
\end{algorithm}

\subsection{Memory and Computational Cost}
As one of the main contributions of this paper, our OMRMD algorithm (i.e., Algorithm~\ref{alg:all}) is appealing for large-scale problems (the regime $d < p \ll n$) since the memory cost is independent of $n$. To see this, note that when computing the optimal coefficients and noise, only $\bz_t$ and $L_{t-1}$ are accessed, which cost $O(pd)$. To store the accumulation matrix $A_t$, we need $O(d^2)$ memory while that for $B_t$ is $O(pd)$. Finally, we find that only $A_t$ and $B_t$ are needed for the computation of the new iterate $L_t$. Therefore, the total memory cost of OMRMD is $O(pd)$, i.e., independent of $n$. In contrast, the SDP formulation introduced by~\cite{srebro2004mmmf} requires $O((p+n)^2)$ memory usage, the local-search heuristic algorithm~\cite{srebro2005fast} needs $O(d(p+n))$ and no convergence guarantee was derived. Even for a recently proposed algorithm~\cite{srebro2010practical}, they require to store the entire data matrix and thus the memory cost is $O(pn)$.

In terms of computational efficiency, our algorithm can be fast, although this is not the main point of this work. One may have noticed that the computation is dominated by solving Problem~\eqref{eq:solve_re}. The computational complexity of~\eqref{eq:lagrange_r} involves an inverse of a $d\times d$ matrix followed by a matrix-matrix and a matrix-vector multiplication, totally $O(pd^2)$. For the basis update, obtaining a subgradient of the squared $\ell_{\tinf}$ norm is $O(pd)$ since we need to calculate the $\ell_2$ norm for all rows of $L$ followed by a multiplication with a diagonal matrix (see~\eqref{eq:subgradient}). A one-pass update for the columns of $L$, as shown in Algorithm~\ref{alg:L} costs $O(pd^2)$. Thus, the computational complexity of OMRMD is $O(pd^2)$. Note that the quadratic dependence on $d$ is acceptable in most cases since $d$ is the estimated rank and hence typically much smaller than $p$.

\section{Theoretical Analysis and Proof Sketch}
\label{sec:main results}
In this section we present our main theoretical result regarding the validity of the proposed algorithm. We first discuss some necessary assumptions.

\subsection{Assumptions}
\begin{enumerate}[label=$(A\arabic*)$]
\item The observed samples are independent and identically distributed (i.i.d.) with a compact support $\mathcal{Z}$. This is a very common scenario  in real-world applications.
\label{as:z}

\item The surrogate functions $g_t(L)$ in~\eqref{eq:g_t(L)} are strongly convex. In particular, we assume that the smallest singular value of the positive semi-definite matrix $\frac{1}{t}A_t$ defined in Algorithm~\ref{alg:all} is not smaller than some positive constant $\beta_1$. %Note that we can easily enforce this assumption by adding a term $\frac{\beta_1}{2} \fronorm{L}^2$ to $g_t(L)$.
\label{as:g_t(L)}

\item The minimizer for Problem~\eqref{eq:l(z,L)} is unique. Notice that $\tilde{\ell}(\bz, L, \br, \be)$ is strongly convex w.r.t. $\be$ and convex w.r.t. $\br$. We can enforce this assumption by adding a jitter $\fractwo{\epsilon} \Vert \br \Vert_2^2$ to the objective function, where $\epsilon$ is a small positive constant. \label{as:unique}
\end{enumerate}

\subsection{Main Results}
It is easy to see that Algorithm~\ref{alg:all} is devised to optimize the empirical loss function~\eqref{eq:f_n(L)}. In stochastic optimization, we are mainly interested in the expected loss function, which is defined as the averaged loss incurred when the number of samples goes to infinity. If we assume that each sample is independently and identically distributed (i.i.d.), we have
\begin{equation}
  f(L) \defeq \lim_{n \rightarrow \infty}f_n(L) = \EXP_{\bz}[\ell( \bz, L)]. %+ \lambdan  \twoinfnorm{L}^2.
  \label{eq:f(L)}
\end{equation}
The main theoretical result of this work is stated as follows.
\begin{theorem}[Convergence to a stationary point of the expected loss function]
\label{thm:stationary}
Let $\{L_t\}_{t=1}^{\infty}$ be the sequence of solutions produced by Algorithm~\ref{alg:all}. Then, the sequence converges to a stationary point of the expected loss function~\eqref{eq:f(L)} when $t$ tends to infinity.
\end{theorem}
\begin{remark}
The theorem establishes the validity of our algorithm. Note that on one hand, the transformation~\eqref{eq:max-def} facilitates an amenable way for the online implementation of the max-norm. On the other hand, due to the non-convexity of our new formulation~\eqref{eq:batch-max}, it is generally hard to desire a local, or even a global minimizer~\cite{bertsekas1999nonlinear}. Although Burer and Monteiro~\cite{burer2005local} showed that any local minimum of an SDP is also the global optimum under some conditions (note that the max-norm problem can be transformed to an SDP~\cite{srebro2004mmmf}), it is hard to determine if a solution is a local optima or a stationary point. From the empirical study in Section~\ref{sec:exp}, we find that the solutions produced by our algorithm always converge to the global optima when the samples are independently and identically drawn from a Gaussian distribution. We leave further analysis on the rationale as our future work.
\end{remark}

\subsection{Proof Outline}
The essential tools for our analysis are from stochastic approximation~\cite{bottou1998online} and asymptotic statistics~\cite{van2000asymptotic}. There are four key stages in our proof and one may find the full proof in Appendix~\ref{supp:sec:proof}. 

\vspace{0.05in}
\noindent{\bf Stage I.} \ 
We first show that all the stochastic variables $\{ L_t, \br_t, \be_t \}_{t=1}^{\infty}$ are uniformly bounded. The property is crucial because it justifies that the problem we solve is well-defined. Also, the uniform boundedness will be heavily used for deriving subsequent important results (e.g., the Lipschitz of the surrogate) to establish our main theorem.

\begin{proposition}[Uniform bound of all stochastic variables]
	\label{prop:bound:reABL}
	Let $\{\br_t, \be_t, L_t\}_{t=1}^{\infty}$ be the sequence of optimal solutions produced by Algorithm~\ref{alg:all}. Then,
	
	\begin{enumerate}
		\item For any $t > 0$, the optimal solutions $\br_t$ and $\be_t$  are uniformly bounded.
		\item  For any $t > 0$, the accumulation matrices $\frac{1}{t}A_t$ and $\frac{1}{t}B_t$ are uniformly bounded.
		\item There exists a compact set $\mathcal{L}$, such that for any $t > 0$, we have $L_t \in \mathcal{L}$.
	\end{enumerate}
\end{proposition}
\begin{proof}
(Sketch) The uniform bound of $\be_t$ follows by constructing a trivial solution $(\mathbf{0}, \mathbf{0})$ for ~\eqref{eq:tildel}, which results in an upper bound for the optimum of the objective function. Notably, the upper bound here only involves a quantity on $\twonorm{\bz_t}$, which is assumed to be uniformly bounded. Since $\br_t$ is always upper bounded by the unit, the first claim follows. The second claim follows immediately by combining the first claim and Assumption~\ref{as:z}. In order to show $L_t$ is uniformly bounded, we utilize the first order optimality condition of the surrogate~\eqref{eq:g_t(L)}. Since $\frac{1}{t}A_t$ is positive definite, we can represent $L_t$ in terms of $\frac{1}{t}B_t$, $U_t$ and the inverse of $\frac{1}{t}A_t$, where $U_t$ is the subgradient, whose Frobenius norm is in turn bounded by that of $L_t$. Hence, it follows that $L_t$ can be uniformly bounded.
\end{proof}
\begin{remark}
Note that in~\cite{mairal2010online,feng2013online}, both of them assume that the dictionary (or basis) is uniformly bounded. In contrast, we prove that such condition naturally holds in our problem.
\end{remark}

\begin{corollary}[Uniform bound and Lipschitz of the surrogate]
\label{coro:bound l lip gt}
Following the notation in Proposition~\ref{prop:bound:reABL}, we have for all $t > 0$,
\begin{enumerate}
	\item  $\tilde{\ell}\(\bz_t, L_t, \br_t, \be_t\)$~\eqref{eq:tildel} and $\ell\(\bz_t, L_t\)$~\eqref{eq:l(z,L)} are both uniformly bounded.
		
	\item The surrogate function, i.e., $g_t(L)$ defined in~\eqref{eq:g_t(L)} is uniformly bounded over $\mathcal{L}$.
		
	\item Moreover, $g_t(L)$ is uniformly Lipschitz over the compact set $\mathcal{L}$.
\end{enumerate}
\end{corollary}

\vspace{0.05in}
\noindent{\bf Stage II.} \ 
We next present that the positive stochastic process $\{g_t(L_t)\}_{t=1}^{\infty}$ converges almost surely. To establish the convergence, we verify that $\{g_t(L_t)\}_{t=1}^{\infty}$ is a quasi-martingale~\cite{bottou1998online} that converges almost surely. To this end, we show that the expectation of the discrepancy of $g_{t+1}(L_{t+1})$ and $g_t(L_t)$ can be upper bounded by a family of functions $\ell(\cdot, L)$ indexed by $L \in \mathcal{L}$. Then we show that the family of the functions is P-Donsker~\cite{van2000asymptotic}, the summands of which concentrate around its expectation within an $O(1/\sqrt{n})$ ball almost surely. Therefore, we conclude that $\{g_t(L_t)\}_{t=1}^{\infty}$ is a quasi-martingale and converges almost surely.

\begin{proposition}
	\label{prop:l:Lipschtiz}
	Let $L \in \mathcal{L}$ and denote the minimizer of $\tilde{\ell}(\bz, L, \br, \be)$ as:
	\begin{align*}
	\{\br^*, \be^*\} = \argmin_{\br, \be, \twonorm{\br} \leq 1} \tildel.
	\end{align*}
	Then, the function $\ell(\bz, L)$ defined in Problem~\eqref{eq:l(z,L)} is continuously differentiable and
	\begin{align*}
	\nabla_L \ell(\bz, L) = (L\br^* + \be^* - \bz) \br^{*\top}.
	\end{align*}
	Furthermore, $\ell(\bz, \cdot)$ is uniformly Lipschitz over the compact set $\mathcal{L}$.
\end{proposition}
\begin{proof}
The gradient of $\ell(\bz, \cdot)$ follows from Lemma~\ref{lem:1}. Since each term of $\nabla_L \ell(\bz, L)$ is uniformly bounded, we conclude the uniform Lipschitz property of $\ell(\bz, L)$ w.r.t. $L$.

\end{proof}

\begin{corollary}[Uniform bound and Lipschitz of the empirical loss]
\label{coro:bound lip ft}
Let $f_t(L)$ be the empirical loss function defined in~\eqref{eq:f_n(L)}. Then $f_t(L)$ is uniformly bounded and Lipschitz over the compact set $\mathcal{L}$.
\end{corollary}

\begin{corollary}[P-Donsker of $\ell(\bz, L)$]
\label{coro:l donsker}
The set of measurable functions $\{ \ell(\bz, L),\ L \in \mathcal{L} \}$ is P-Donsker~(see definition in Lemma~\ref{lem:donsker}).
\end{corollary}

\begin{proposition}[Concentration of the empirical loss]
\label{prop:concentration of ft}
Let $f_t(L)$ and $f(L)$ be the empirical and expected loss functions we defined in~\eqref{eq:f_n(L)} and~\eqref{eq:f(L)}. Then we have
\begin{align*}
\EXP[\sqrt{t}\lV f_t - f \rV_{\infty}] = O(1).
\end{align*}
\end{proposition}
\begin{proof}
Since $\ell(\bz, L)$ is uniformly upper bounded (Corollary~\ref{coro:bound l lip gt}) and is always non-negative, its square is uniformly upper bounded, hence its expectation. Combining Corollary~\ref{coro:l donsker}, Lemma~\ref{lem:donsker} applies.

\end{proof}

\begin{theorem}[Convergence of the surrogate]
\label{thm:convergence gt(Lt)}
The sequence $\{g_t(L_t)\}_{t=1}^{\infty}$ we defined in~\eqref{eq:g_t(L)} converges almost surely, where $\{L_t\}_{t=1}^{\infty}$ is the solution produced by Algorithm~\ref{alg:all}. Moreover, the infinite summation $\sum_{t=1}^{\infty} \lvert \EXP[g_{t+1}(L_{t+1}) - g_t(L_t) \mid \mathcal{F}_t] \rvert$ is bounded almost surely.
\end{theorem}
\begin{proof}
The theorem follows by showing that the sequence of $\{ g_t(L_t) \}_{t=1}^{\infty}$ is a quasi-martingale, and hence converges almost surely. To see this, we note that for any $t > 0$, the expectation of the difference $g_{t+1}(L_{t+1}) - g_t(L_t)$ conditioned on the past information $\mathcal{F}_t$ is bounded by $\sup_L (f(L) - f_t(L)) / (t+1)$, which is of order $O(1/(\sqrt{t}(t+1)))$ due to Proposition~\ref{prop:concentration of ft}. Hence, Lemma~\ref{lem:bottou} applies.
\end{proof}

\vspace{0.05in}
\noindent{\bf Stage III.} \ 
Then we prove that the sequence of the empirical loss function, $\{f_t(L_t)\}_{t=1}^{\infty}$ defined in~\eqref{eq:f_n(L)} converges almost surely to the same limit of its surrogate $\{g_t(L_t)\}_{t=1}^{\infty}$. According to the central limit theorem, we assert that $f_t(L_t)$ also converges almost surely to the expected loss $f(L_t)$ defined in~\eqref{eq:f(L)}, implying that $g_t(L_t)$ and $f(L_t)$ converge to the same limit almost surely.

We first show the numerical convergence of the basis sequence $\{L_t\}_{t=1}^{\infty}$, based on which we show the convergence of $\{ f_t(L_t) \}_{t=1}^{\infty}$ by applying Lemma~\ref{lem:mairal}.
\begin{proposition}[Numerical convergence of the basis component]
\label{prop:diff_L}
Let $\{L_t\}_{t=1}^{\infty}$ be the basis sequence produced by the Algorithm~\ref{alg:all}. Then,
\begin{equation}
\fronorm{L_{t+1} - L_t} = O\(\frac{1}{t}\).
\end{equation}
\end{proposition}

\begin{theorem}[Convergence of the empirical and expected loss]
\label{thm:convergence of f(L_t)}
Let $\{f(L_t)\}_{t=1}^{\infty}$ be the sequence of the expected loss where $\{L_t\}_{t=1}^{\infty}$ be the sequence of the solutions produced by the Algorithm~\ref{alg:all}. Then, we have
\begin{enumerate}
	\item The sequence of the empirical loss $\{f_t(L_t)\}_{t=1}^{\infty}$ converges almost surely to the same limit of the surrogate.
		
	\item The sequence of the expected loss $\{f(L_t)\}_{t=1}^{\infty}$ converges almost surely to the same limit of the surrogate.
\end{enumerate}
\end{theorem}
\begin{proof}
Let $b_t = g_t(L_t) - f_t(L_t)$. We show that infinite series $\sum_{t=1}^{\infty} b_t / (t+1)$ is bounded by applying the central limit theorem to $f(L_t) - f_t(L_t)$ and the result of Theorem~\ref{thm:convergence gt(Lt)}. We further prove that $\lvert b_{t+1} - b_t \rvert$ can be bounded by $O(1/t)$, due to the uniform boundedness and Lipschitz of $g_t(L_t)$, $f_t(L_t)$ and $\ell(\bz_t, L_t)$. According to Lemma~\ref{lem:mairal}, we conclude the convergence of $\{b_t\}_{t=1}^{\infty}$ to zero. Hence the first claim. The second claim follows immediately owing to the central limit theorem.
\end{proof}

\vspace{0.05in}
\noindent{\bf Final Stage.} \ 
According to Claim 2 of Theorem~\ref{thm:convergence of f(L_t)} and the fact that 0 belongs to the subgradient of $g_t(L)$ evaluated at $L = L_t$, we are to show the gradient of $f(L)$ taking at $L_t$ vanishes as $t$ tends to infinity, which establishes Theorem~\ref{thm:stationary}. To this end, we note that since $\{L_t\}_{t=1}^{\infty}$ is uniformly bounded, the non-differentiable term $\frac{1}{2t} \twoinfnorm{L}^2$ vanishes as $t$ goes to infinity, implying the differentiability of $g_{\infty}(L_{\infty})$, i.e. $\nabla g_{\infty}(L_{\infty})=0$. On the other hand, we show that the gradient of $f(L)$ and that of $g_t(L)$ are always Lipschitz on the compact set $\mathcal{L}$, implying the existence of their second order derivative even when $t \rightarrow \infty$. Thus, by taking a first order Taylor expansion and let $t$ go to infinity, we establish the main theorem.

\section{Connection to Matrix Completion}
\label{sec:mc}

While we mainly focus on the matrix decomposition problem, our method can  be extended to the matrix completion (MC) problem~\cite{cai2010singular,candes2009exact} with max-norm regularization~\cite{cai2014max}~--~another popular topic in machine learning and signal processing. The max-norm regularized MC problem can be described as follows:
\begin{align*}
\min_X\ \frac{1}{2}\fronorm{\mathcal{P}_{\Omega}\(Z - X\)}^2 + \fractwo{\lambda} \maxnorm{X}^2,
\end{align*}
where $\Omega$ is the set of indices of observed entries in $Z$ and $\mathcal{P}_{\Omega}(M)$ is the orthogonal projection onto the span of matrices vanishing outside of $\Omega$ so that the $(i, j)$-th entry of $\mathcal{P}_{\Omega}(M)$ is equal to $M_{ij}$ if $(i, j) \in \Omega$ and zero otherwise. Interestingly, the max-norm regularized MC problem can be cast into our framework. To see this, let us introduce an auxiliary matrix $M$, with $M_{ij} = c > 0$ if $(i, j) \in \Omega$ and $M_{ij}=1/c$ otherwise. The reformulated MC problem,
\begin{equation}
\min_{X, E}\ \fractwo{1} \fronorm{Z-X-E}^2 + \fractwo{\lambda} \maxnorm{X}^2 + \onenorm{M \circ E},
\label{eq:mc}
\end{equation}
where ``$\circ$'' denotes the entry-wise product, is comparable to our MRMD formulation~\eqref{eq:primal}. And it is easy to show that when $c$ tends to infinity, the reformulated problem converges to the original MC problem.

\subsection{Online Implementation}
We now derive a stochastic implementation for the max-norm regularized MC problem. Note that the only difference between the Problem~\eqref{eq:mc} and Problem~\eqref{eq:primal} is the $\ell_1$ regularization on $E$, which results a new penalty on $\be$ for $\tilde{\ell}(\bz, L, \br, \be)$ (which is originally defined in~\eqref{eq:tildel}):
\begin{equation}
\tilde{\ell}(\bz, L, \br, \be) = \frac{1}{2} \twonorm{\bz- L\br - \be}^2 + \onenorm{\bm \circ \be}.
\label{eq:new tildel}
\end{equation}
Here, $\bm$ is a column of the matrix $M$ in~\eqref{eq:mc}. According to the definition of $M$, $\bm$ is a vector with element value being either $c$ or $1/c$. Let us define two support sets as follows:
\begin{align*}
\begin{split}
\Omega_1 \defeq&\ \{i \mid m_i = c, 1 \leq i \leq p \},\\
\Omega_2 \defeq&\ \{i \mid m_i = 1/c, 1 \leq i \leq p\},
\end{split}
\end{align*}
where $m_i$ is the $i$th element of vector $\bm$. In this way, the newly defined $\tilde{\ell}(\bz, L, \br, \be)$ can be written as
\begin{equation}
\label{eq:tildel mc}
\begin{split}
\tilde{\ell}(\bz, L, \br, \be) =& \( \frac{1}{2} \lV \bz^{}_{\Omega_1}- (L\br)^{}_{\Omega_1} - \be^{}_{\Omega_1} \rV_2^2 + c \lV \be^{}_{\Omega_1} \rV_1 \) \\
&+ \( \frac{1}{2} \lV \bz^{}_{\Omega_2}- (L\br)^{}_{\Omega_2} - \be^{}_{\Omega_2} \rV_2^2 + \frac{1}{c} \lV \be^{}_{\Omega_2} \rV_1 \).
\end{split}
\end{equation}
Notably, as $\Omega_1$ and $\Omega_2$ are disjoint, given $\bz$, $L$ and $\br$, the variable $\be$ in~\eqref{eq:tildel mc} can be optimized by soft-thresholding in a separate manner:
\begin{equation}
\begin{split}
\be^{}_{\Omega_1} =&\ \mathcal{S}^{}_{c}[\bz^{}_{\Omega_1}- (L\br)^{}_{\Omega_1}],\\
\be^{}_{\Omega_2} =&\ \mathcal{S}^{}_{1/c}[\bz^{}_{\Omega_2}- (L\br)^{}_{\Omega_2}].
\end{split}
\label{eq:mc:solve_e}
\end{equation}

With this rule on hand, we propose Algorithm~\ref{alg:mc} for the online max-norm regularized matrix completion (OMRMC) problem. The update rule for $\br$ is the same as we described in Algorithm~\ref{alg:re} and that for $\be$ is given by~\eqref{eq:mc:solve_e}. Note that we can use Algorithm~\ref{alg:L} to update $L$ as usual.
\begin{algorithm}
\caption{Online Max-Norm Regularized Matrix Completion}
\label{alg:mc}
\begin{algorithmic}[1]
    \REQUIRE $Z \in \Rpn$ (observed samples), parameters $\lambda_1$ and $\lambda_2$, $L_0 \in \Rpd$ (initial basis), zero matrices $A_0 \in \mathbb{R}^{d \times d}$ and $B_0 \in \mathbb{R}^{p \times d}$
    \ENSURE optimal basis $L_t$
    \FOR{$t=1$ to $n$}
      \STATE Access the $t$-th sample $\bz_t$.
      \STATE Compute the coefficient and noise:
        \begin{align*}
        \label{supp:eq:mc:solve_re}
          \{\br_t, \be_t \} =& \argmin_{\br, \be, \Vert \br \Vert_2^2 \leq 1} \tilde{\ell}(\bz_t, L_{t-1}, \br, \be)\\
          =& \argmin_{\br, \be, \Vert \br \Vert_2^2 \leq 1} \left( \frac{1}{2} \Vert \bz_t - L_{t-1} \br - \be \Vert_2^2 + \Vert \bm_t \circ \be \Vert_1 \right).
        \end{align*}
      \STATE Compute the accumulation matrices $A_t$ and $B_t$:
        \begin{align*}
          A_t &\leftarrow A_{t-1} + \br_t \br_t\trans,\\
          B_t &\leftarrow B_{t-1} + \left(\bz_t - \be_t \right)\br_t\trans.
        \end{align*}

      \STATE Compute the basis $L_t$ by optimizing the surrogate function~\eqref{eq:g_t(L)}:
        \begin{align*}
          L_t &= \argmin_L \frac{1}{t}\sum_{i=1}^t \tilde{\ell}(\bz_i, L, \br_i, \be_i) + \lambdat \twoinfnorm{L}^2 \\
          &= \argmin_L \frac{1}{t}\sum_{i=1}^t \( \frac{1}{2} \lV \bz_i - L \br_i - \be_i \rV_2^2 + \lV \bm_i \circ \be_i \rV_1 \) + \lambdat \twoinfnorm{L}^2 \\
          &= \argmin_L \frac{1}{t}\sum_{i=1}^t \( \frac{1}{2} \lV \bz_i - L \br_i - \be_i \rV_2^2\) + \lambdat \twoinfnorm{L}^2 \\
          &= \argmin_L \frac{1}{t} \(\frac{1}{2} \tr \( L\trans L A_t \) - \tr \( L\trans B_t \) \) + \lambdat \twoinfnorm{L}^2.
        \end{align*}
    \ENDFOR
\end{algorithmic}
\end{algorithm}

Since we have clarified the algorithm for OMRMC, we move to the theoretical analysis. We argue that {\em all} the results for OMRMD apply to OMRMC, which can be trivially justified.

\input{exp.tex}

\section{Conclusion}
\label{sec:conclude}
In this paper, we have developed an online algorithm for {max-norm} regularized matrix decomposition problem. Using the matrix factorization form of the max-norm, we converted the original problem to a constrained one which facilitates an online implementation for solving the batch  problem. We have established theoretical guarantees that the solutions will converge to a stationary point of the expected loss function asymptotically.  Moreover, we empirically compared our proposed algorithm with OR-PCA, which is a recently proposed online algorithm for nuclear-norm based matrix decomposition. The simulation results have suggested that the proposed algorithm is more robust than OR-PCA, in particular for hard tasks (i.e., when a large fraction of entries are corrupted). We also have investigated the convergence rate for both OMRMD and OR-PCA, and have shown that OMRMD converges much faster than OR-PCA even in large-scale problems. When acquiring sufficient samples, we observed that our algorithm converges to the batch method PCP, which is a state-of-the-art formulation for subspace recovery. Our experiments, to an extent,  suggest that the max-norm might be a tighter relaxation of the rank function compared to the nuclear norm.

\clearpage
\appendix

\input{proof.tex}

\clearpage
{
\bibliographystyle{alpha}
\bibliography{max-norm}
}
\end{document}

%% file: exp.tex
\section{Experiments}
\label{sec:exp}
In this section, we report numerical results on synthetic data to demonstrate the effectiveness and robustness of our online max-norm regularized matrix decomposition (OMRMD) algorithm. Some experimental settings are used throughout this section, as elaborated below.

\vspace{0.05in}
\noindent{\bf Data Generation.} \ 
The simulation data are generated  by following a similar procedure in~\cite{candes2011rpca}. The clean data matrix $X$ is produced by  $X = UV\trans$, where $U\in \mathbb{R}^{ p \times d}$ and $V\in  \mathbb{R}^{ n \times d}$. The entries of $U$ and $V$ are i.i.d. sampled from the normal distribution $\mathcal{N}(0, 1)$. We choose sparse corruption in the experiments, and introduce a parameter $\rho$ to control the sparsity of the corruption matrix $E$, \emph{i.e.},  a $\rho$-fraction of the entries  are non-zero and following an i.i.d. uniform distribution over  $[-1000, 1000]$. Finally, the observation matrix $Z$ is produced by  $Z=X+E$.

\vspace{0.05in}
\noindent{\bf Baselines.} \ 
We mainly compare with two methods: Principal Component Pursuit~(PCP) and online robust PCA~(OR-PCA). PCP is the state-of-the-art batch method for subspace recovery, which was presented as a robust formulation of PCA in~\cite{candes2011rpca}. OR-PCA is an online implementation of PCP,\footnote{Strictly speaking, OR-PCA is an online version of stable PCP~\cite{zhou2010stable}.} which also achieves state-of-the-art performance over the online subspace recovery algorithms. Sometimes, to show the robustness, we will also report the results of online PCA~\cite{artac2002incremental}, which incrementally learns the principal components without taking the noise into account.

\vspace{0.05in}
\noindent{\bf Evaluation Metric.} \ 
Our goal is to estimate the correct subspace for the underlying data. Here, we evaluate the fitness of our estimated subspace basis $L$ and the ground truth basis $U$ by the Expressed Variance (EV)~\cite{xu2010principal}:
\begin{equation}
\text{EV}(U, L) \defeq \frac{\tr(L\trans UU\trans L)}{\tr(UU\trans)}.
\end{equation}
The values of EV range in $[0, 1]$ and a higher value indicates a more accurate recovery.

\vspace{0.05in}
\noindent{\bf Other Settings.} \
Throughout the experiments, we set the ambient dimension $p=400$ and the total number of samples $n=5000$ unless otherwise specified. We fix the tunable parameter $\lambda_1 = \lambda_2 = 1 / \sqrt{p}$, and use default parameters for all baselines we compare with. Each experiment is repeated 10 times  and we report the averaged EV as the result.

\begin{figure*}[!htb]
	\centering
	\subfloat[OMRMD]{
		\includegraphics[width=0.33\linewidth]{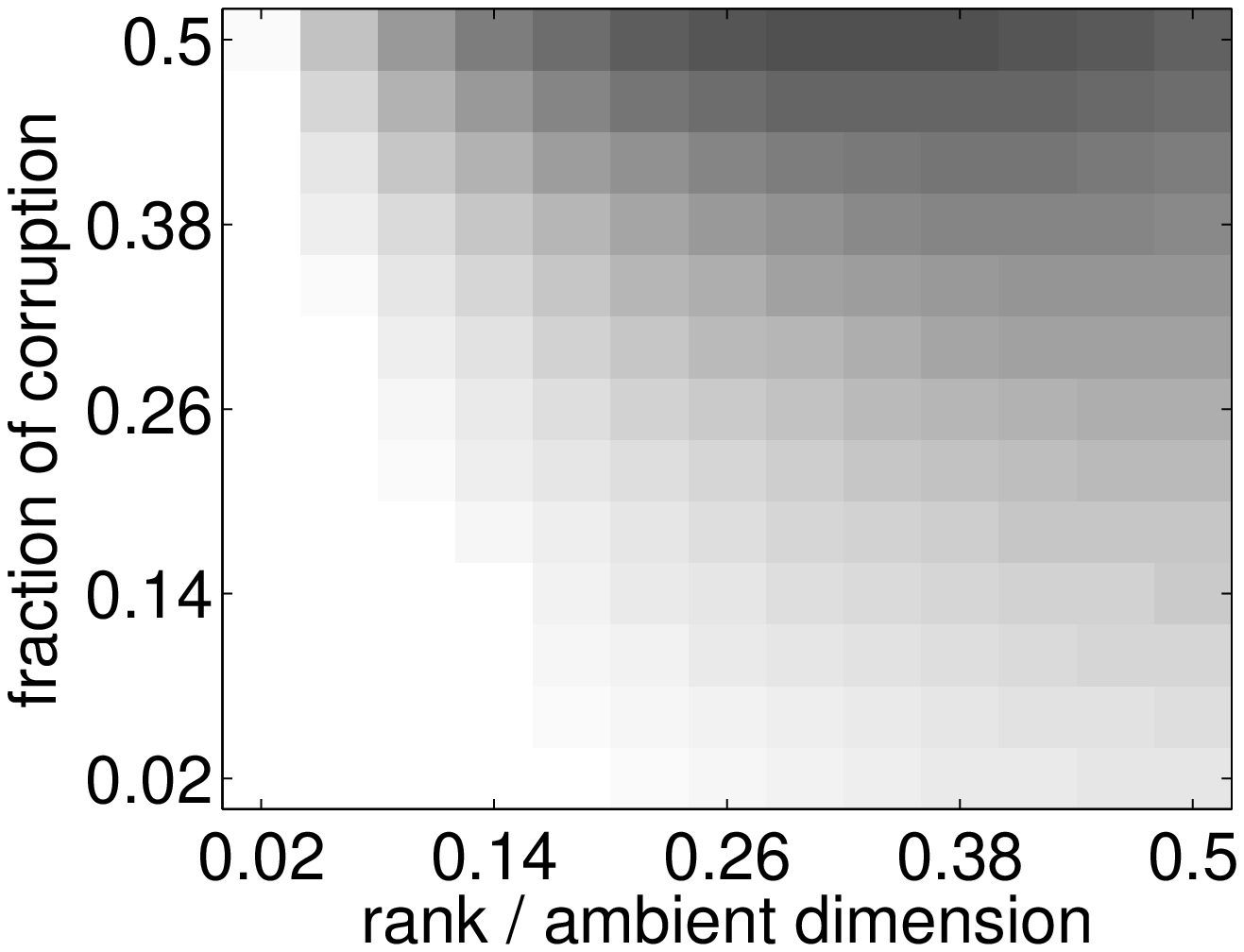}
	}
	\subfloat[OR-PCA]{
		\includegraphics[width=0.33\linewidth]{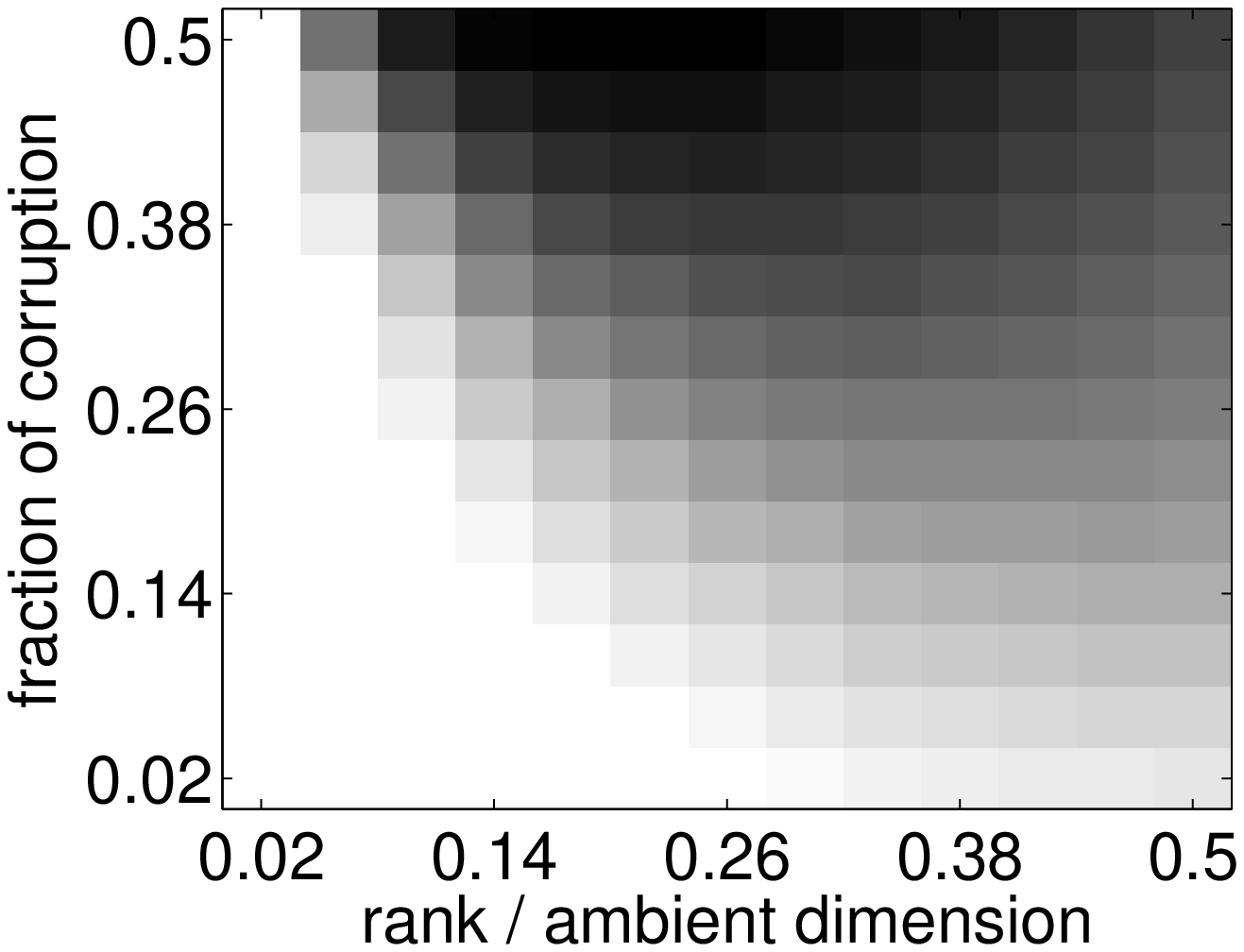}
	}
	\subfloat[PCP]{
		\includegraphics[width=0.33\linewidth]{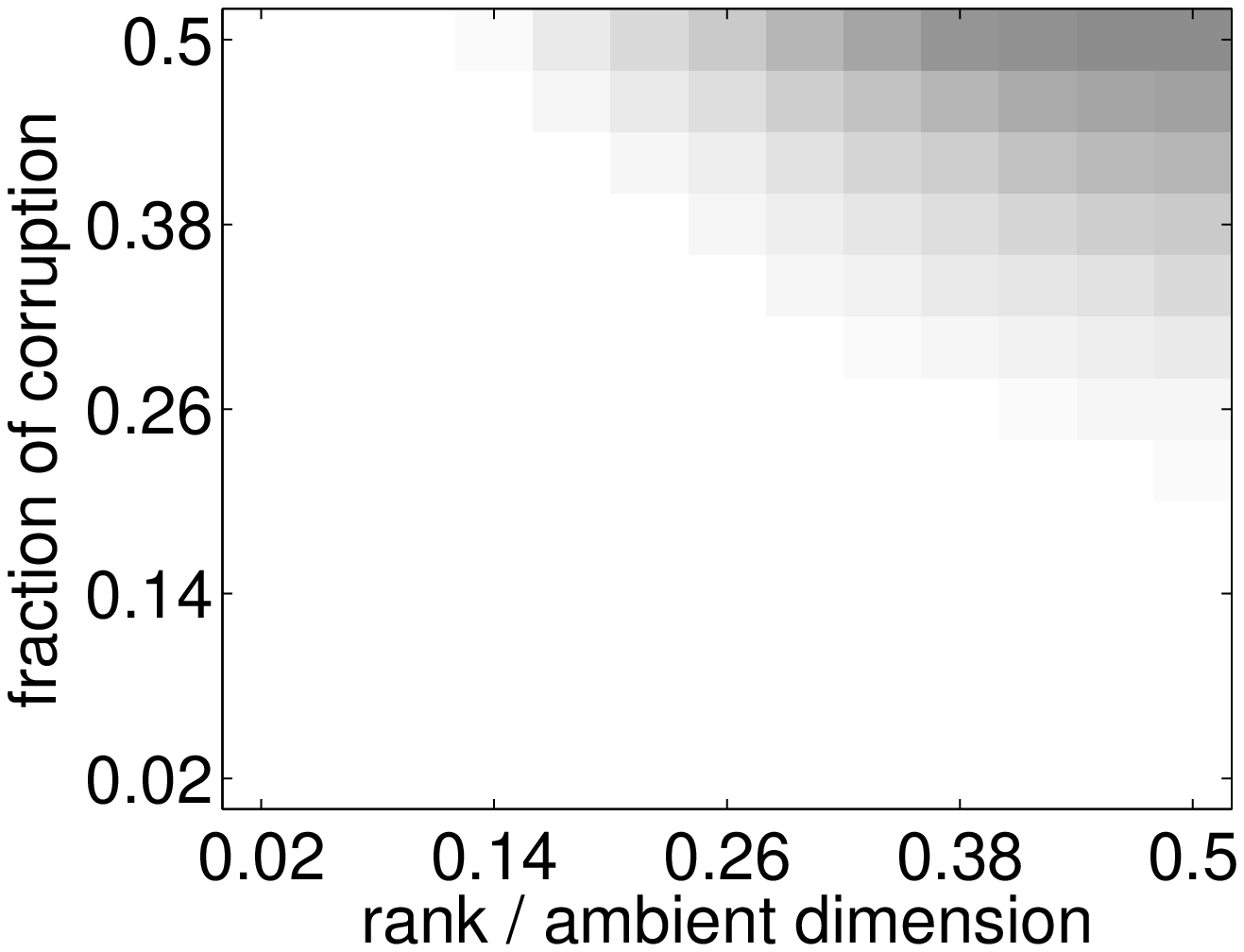}
	}
	\caption{Performance of subspace recovery under different rank and corruption fraction. Brighter color means better performance. As observed, the max-norm based algorithm OMRMD always performs comparably or outperforms OR-PCA which is based on nuclear norm formulation. Since PCP is a batch method, it always achieves the best recovery performance.
	}
	\label{fig:diff_rank_rho}
\end{figure*}

\begin{figure*}[!htb]
	\centering
	\subfloat[]{
		\includegraphics[width=0.32\linewidth]{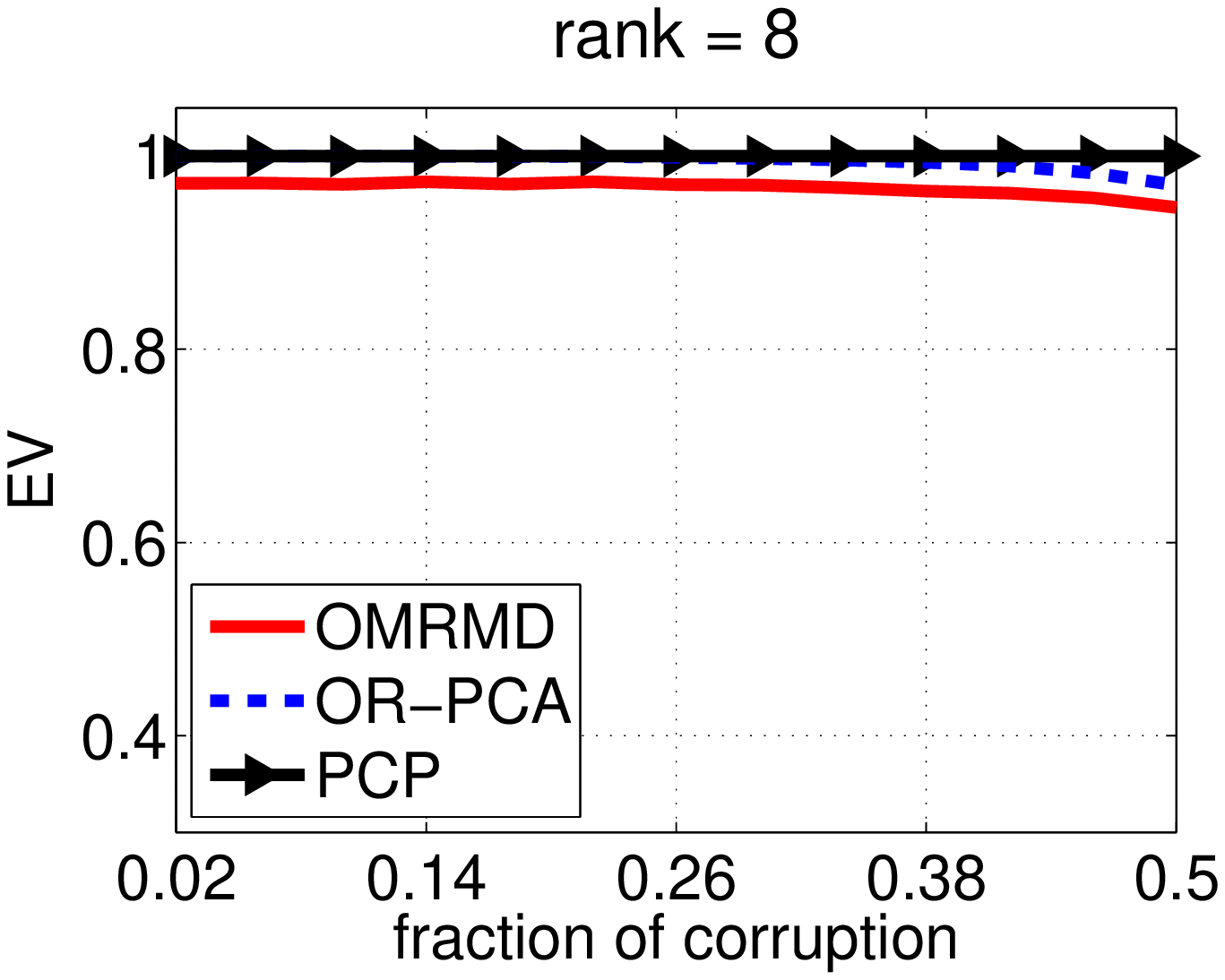}
		\label{fig:rank 8}
	}
	\subfloat[]{
		\includegraphics[width=0.32\linewidth]{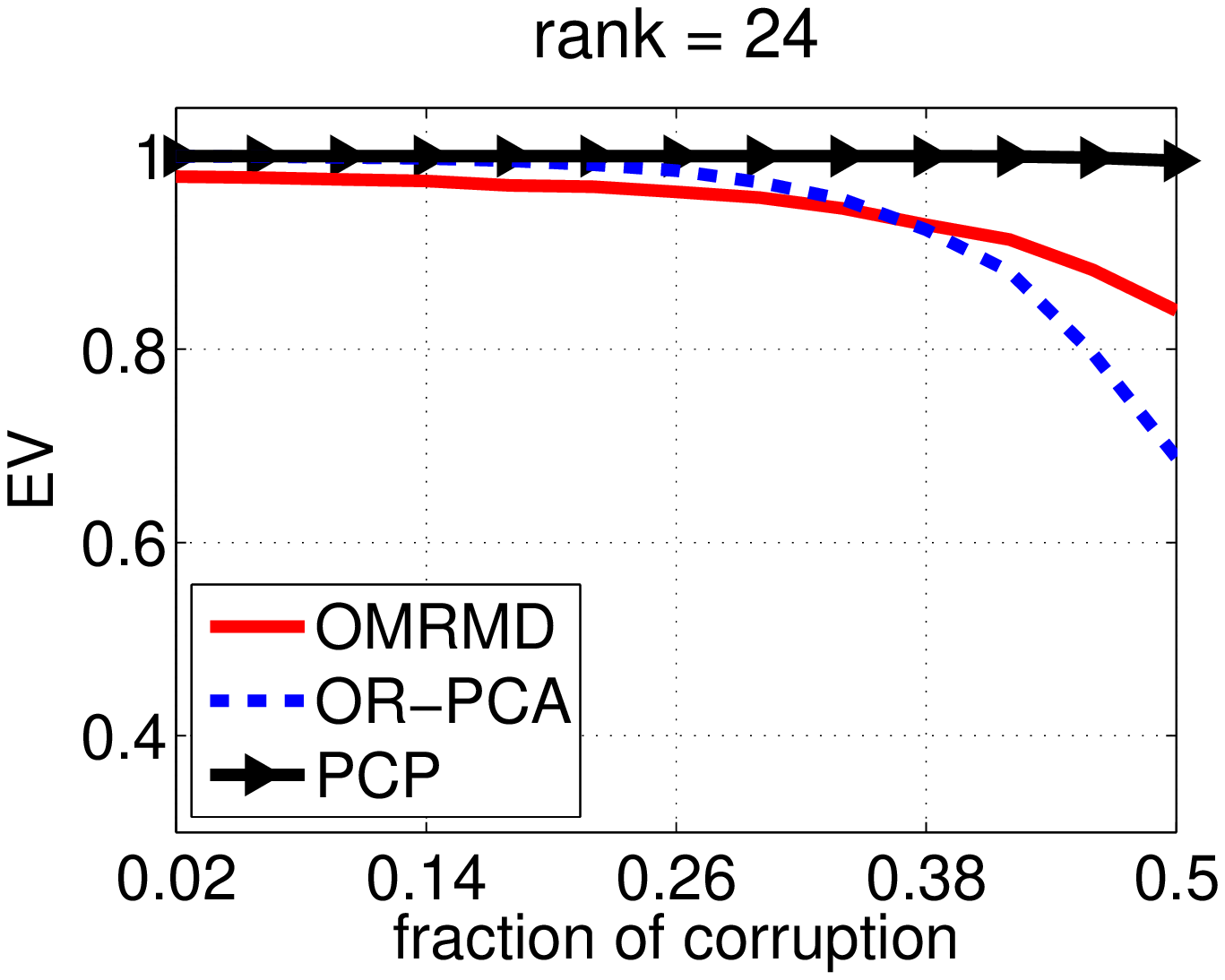}
		\label{fig:rank 24}
	}
	\subfloat[]{
		\includegraphics[width=0.32\linewidth]{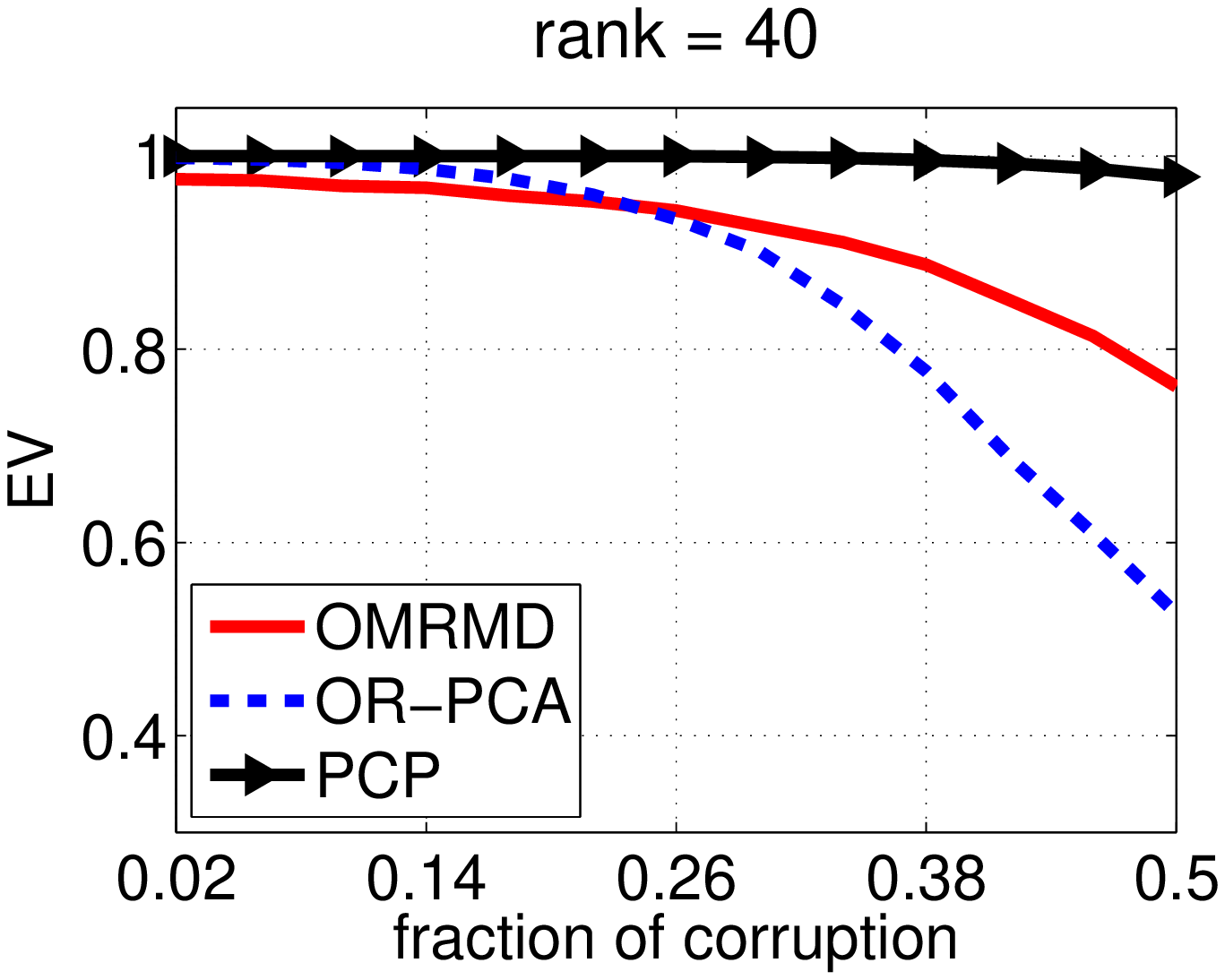}
		\label{fig:rank 40}
	}
	
	\subfloat[]{
		\includegraphics[width=0.32\linewidth]{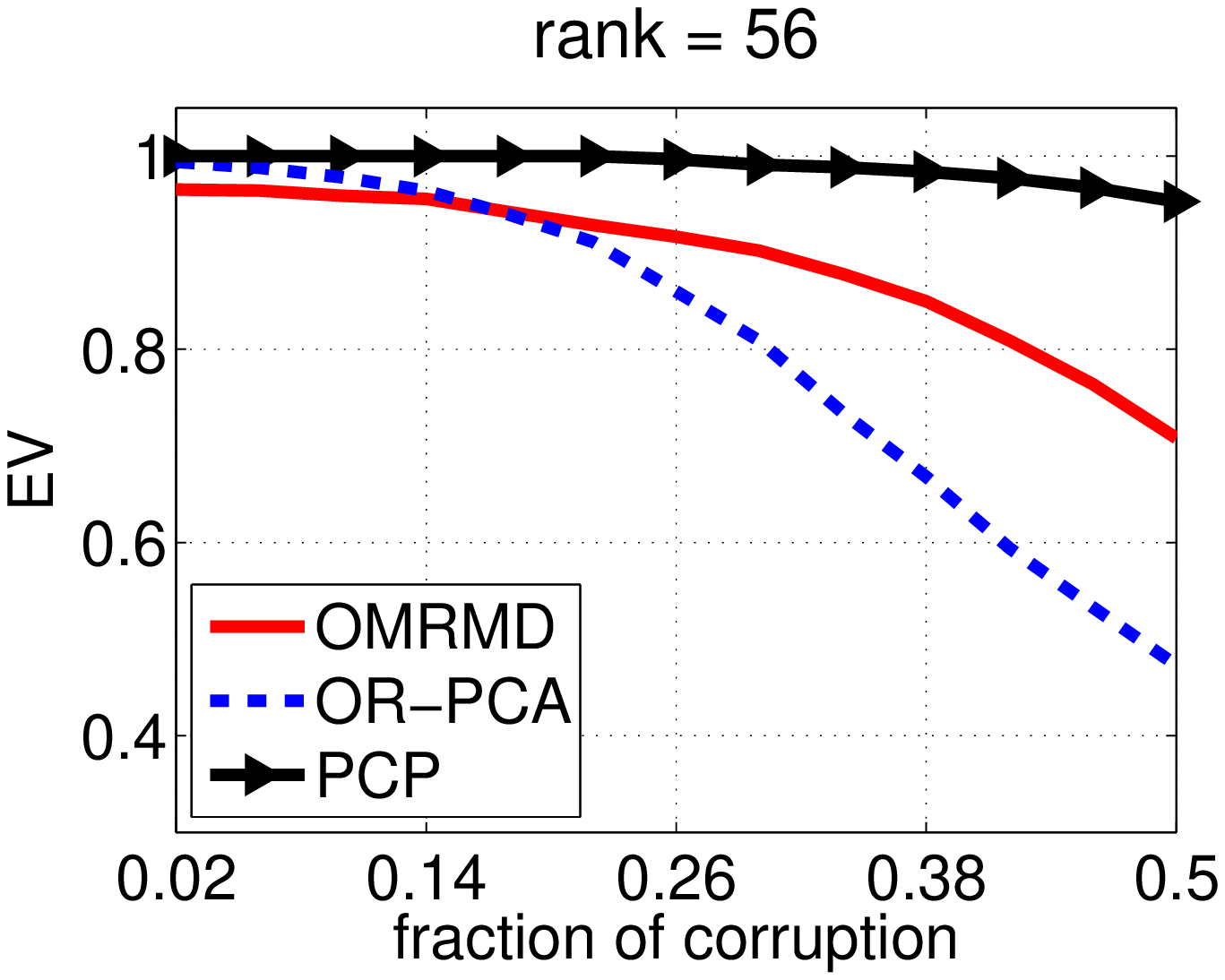}
		\label{fig:rank 56}
	}
	\subfloat[]{
		\includegraphics[width=0.32\linewidth]{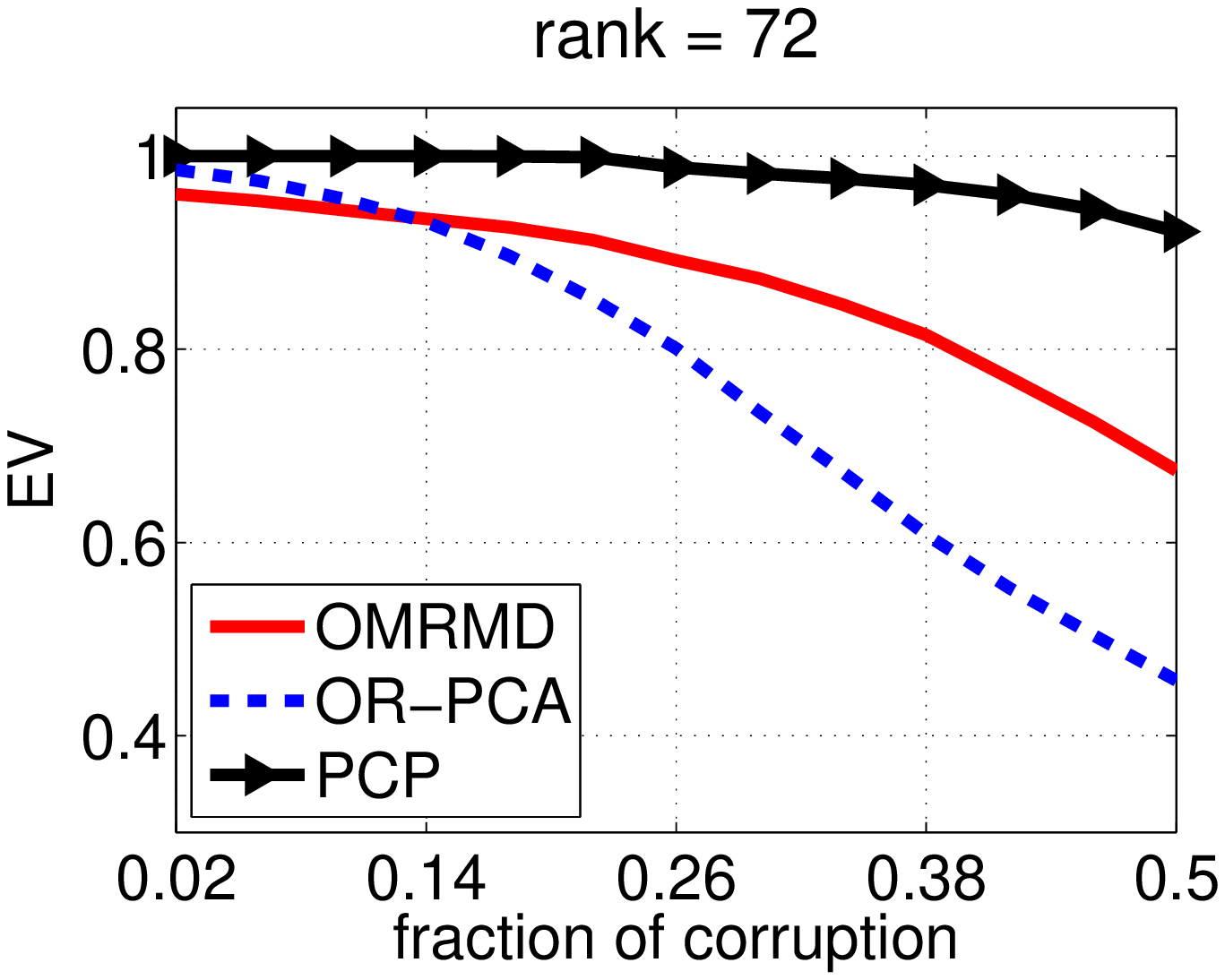}
		\label{fig:rank 72}
	}
	\subfloat[]{
		\includegraphics[width=0.32\linewidth]{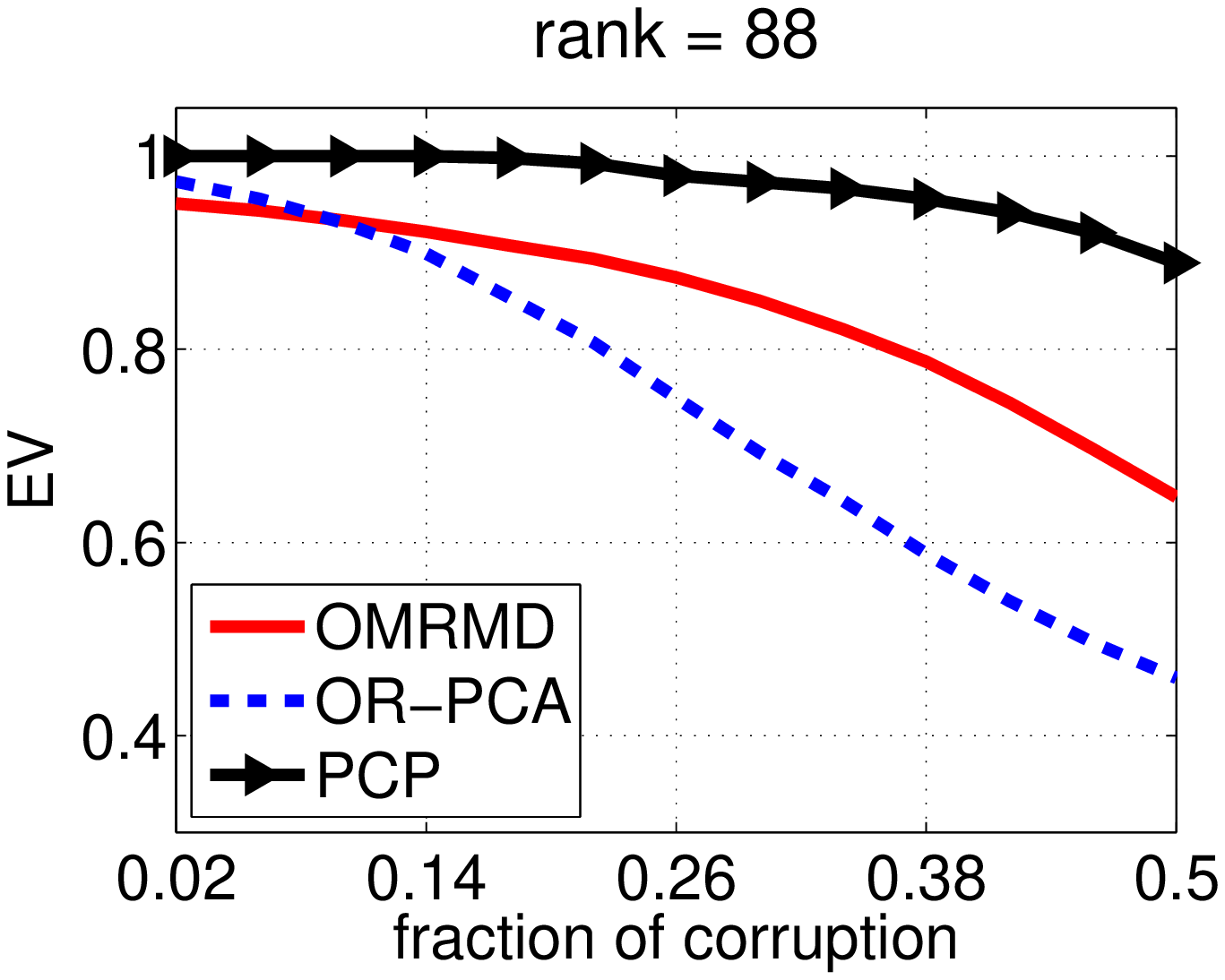}
		\label{fig:rank 88}
	}
	\caption{EV value against corruption fractions when the matrix has a relatively low rank (note that the ambient dimension $p$ is 400). The EV value is computed by the obtained basis after accessing the last sample. When the rank is extremely low (rank = 8), OMRMD and OR-PCA works comparably. In other cases, OMRMD is always better than OR-PCA addressing a large fraction of corruption.}
	\label{fig:diff_rank_rho_lowrank}
\end{figure*}

\begin{figure*}[!htb]
	\centering
	\subfloat[]{
		\includegraphics[width=0.32\linewidth]{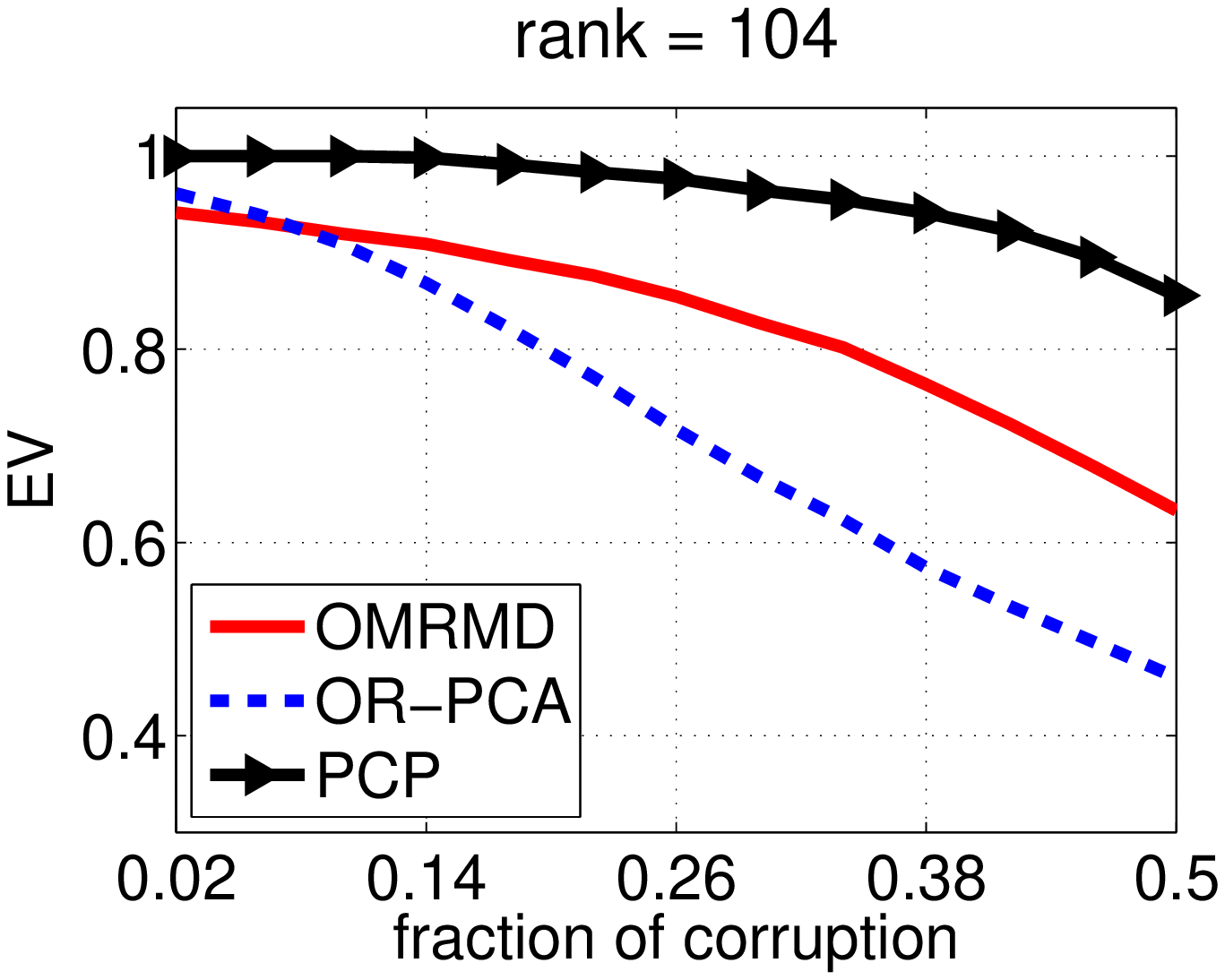}
		\label{fig:rank 104}
	}
	\subfloat[]{
		\includegraphics[width=0.32\linewidth]{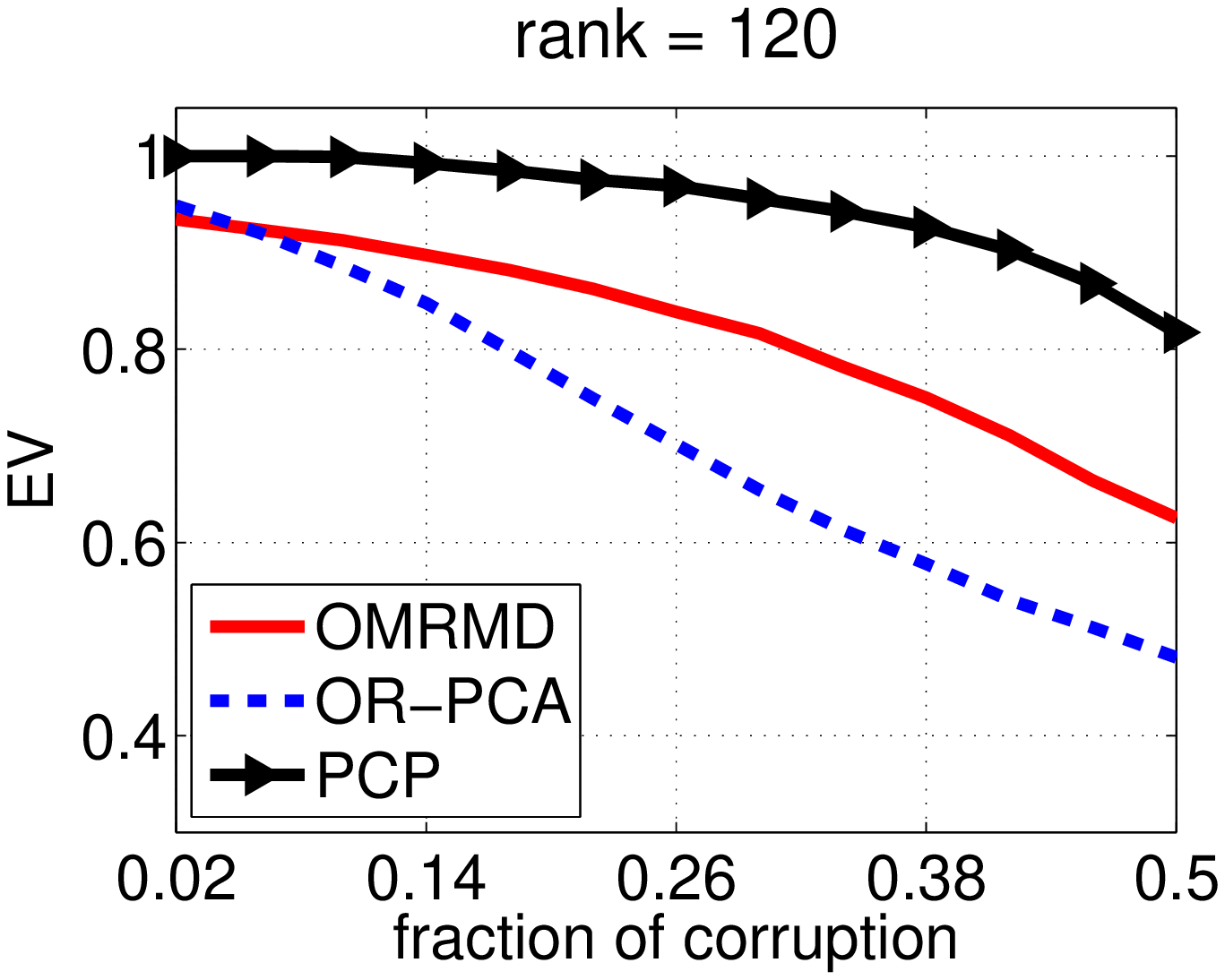}
		\label{fig:rank 120}
	}
	\subfloat[]{
		\includegraphics[width=0.32\linewidth]{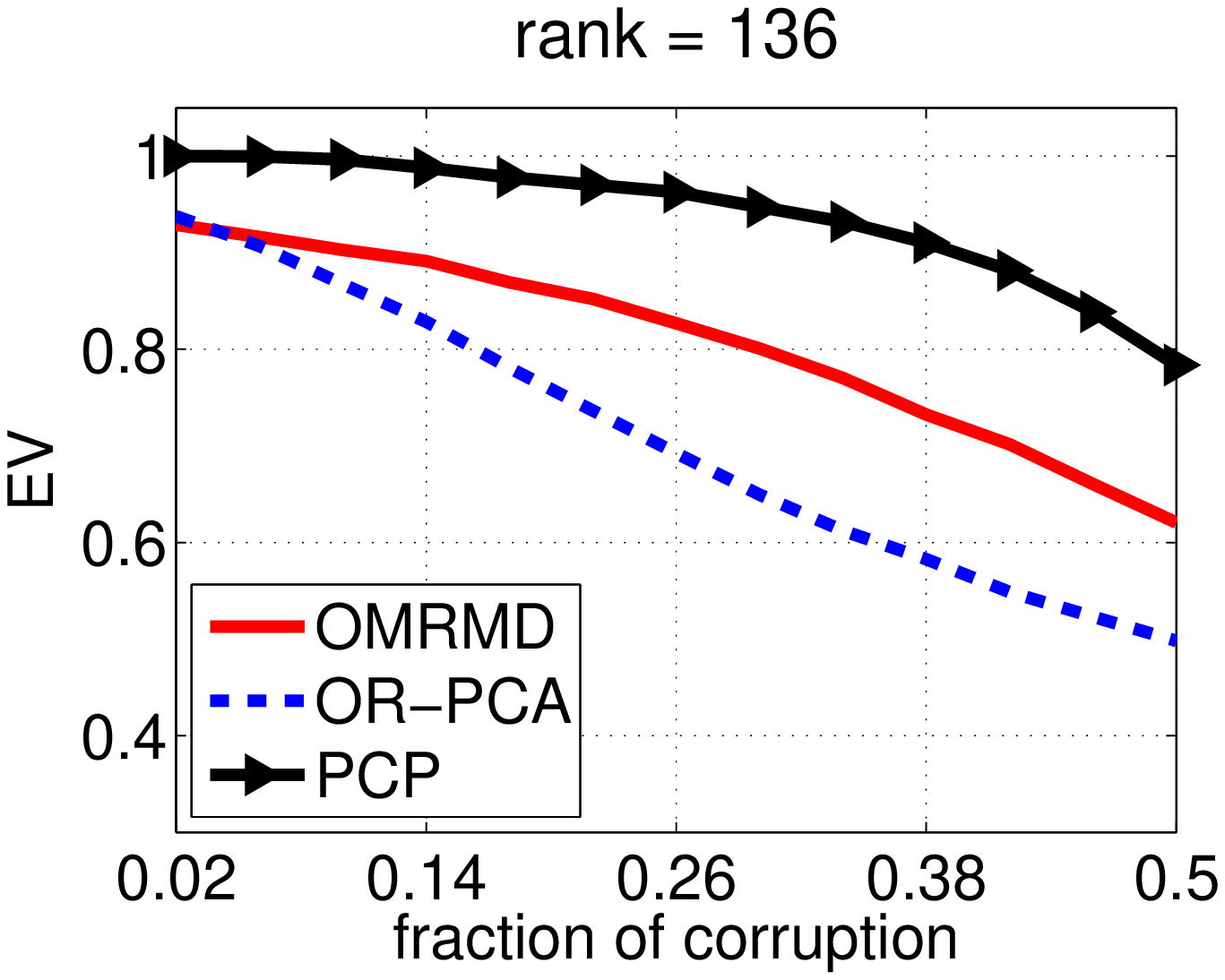}
		\label{fig:rank 136}
	}
	
	\subfloat[]{
		\includegraphics[width=0.32\linewidth]{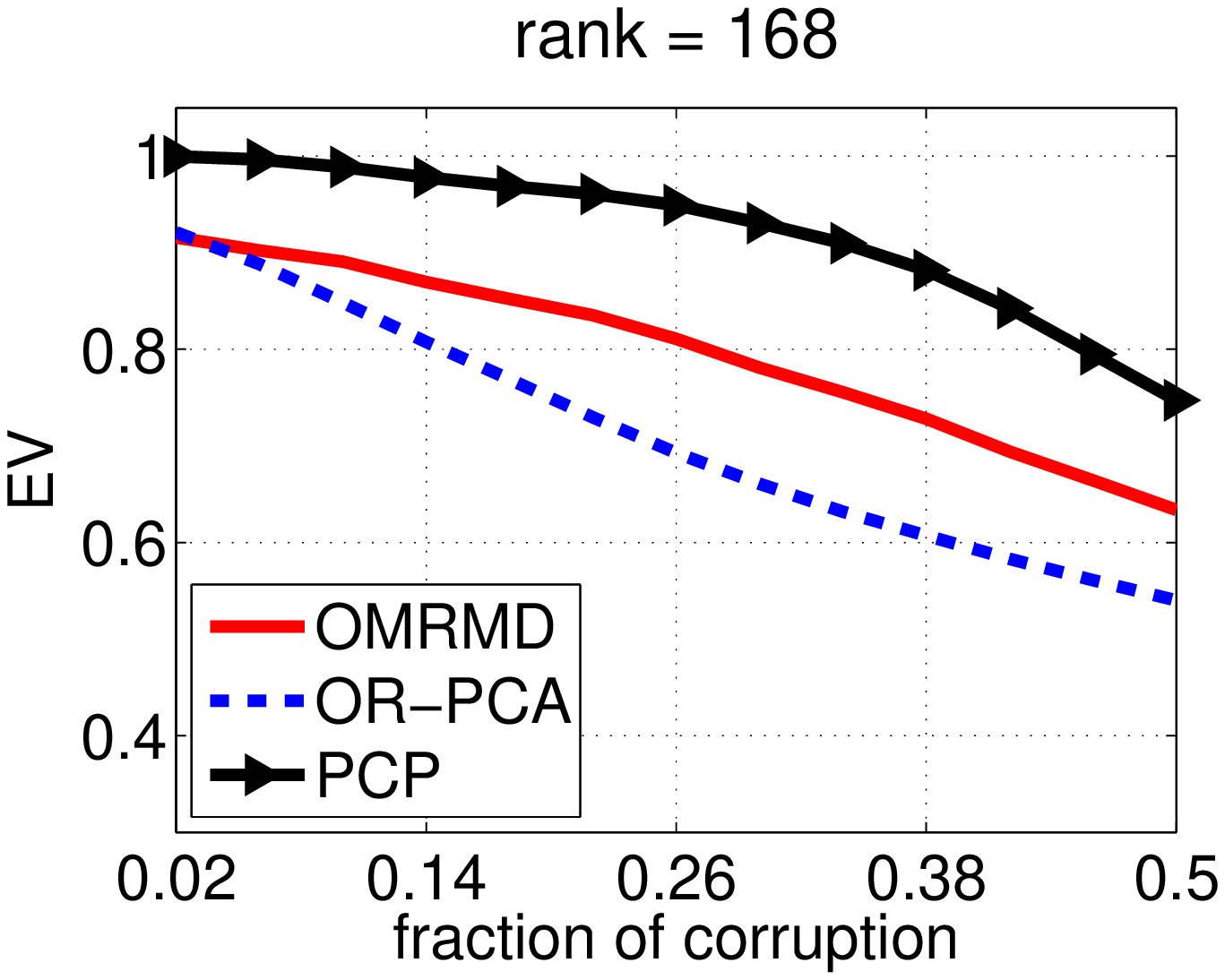}
		\label{fig:rank 168}
	}
	\subfloat[]{
		\includegraphics[width=0.32\linewidth]{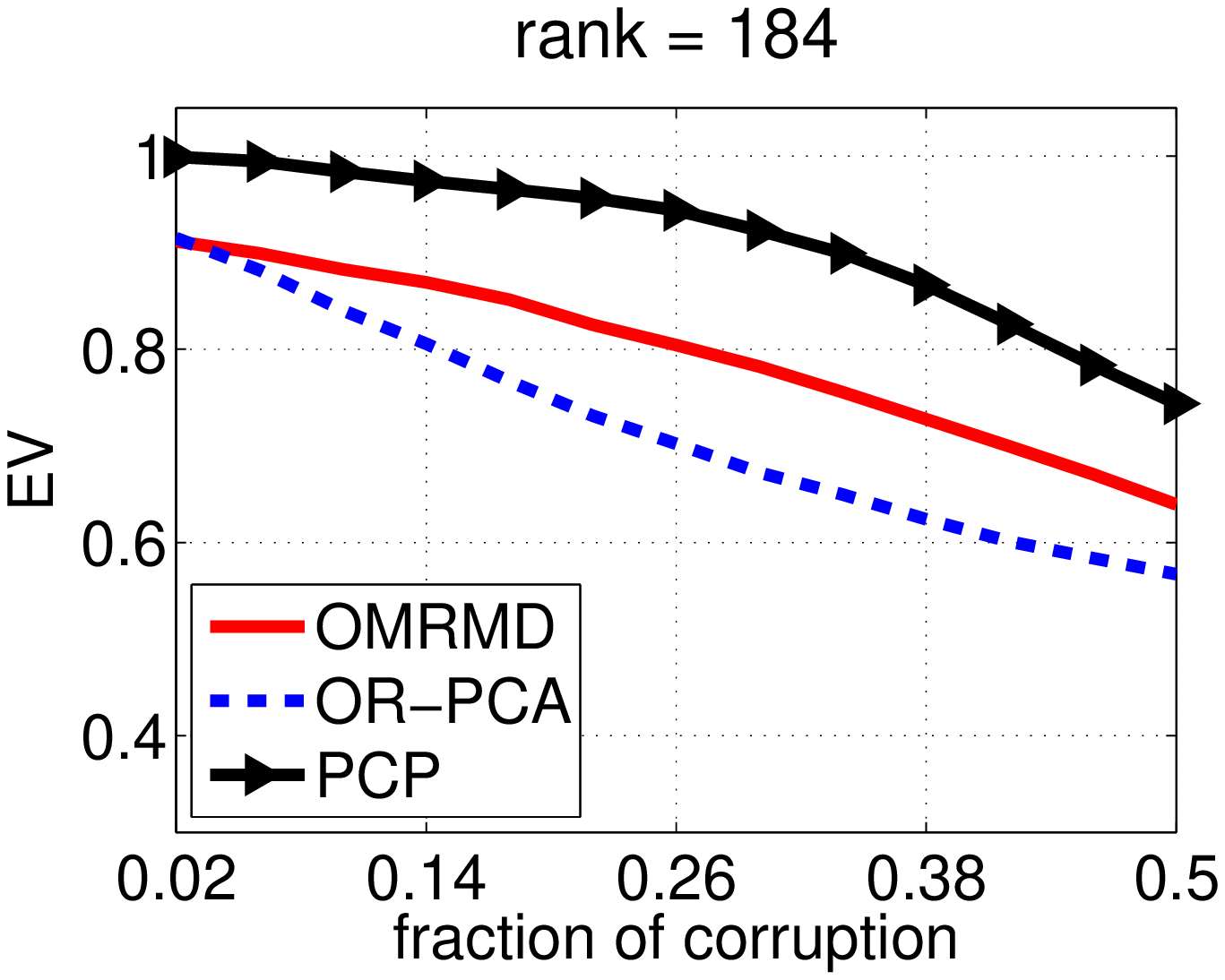}
		\label{fig:rank 184}
	}
	\subfloat[]{
		\includegraphics[width=0.32\linewidth]{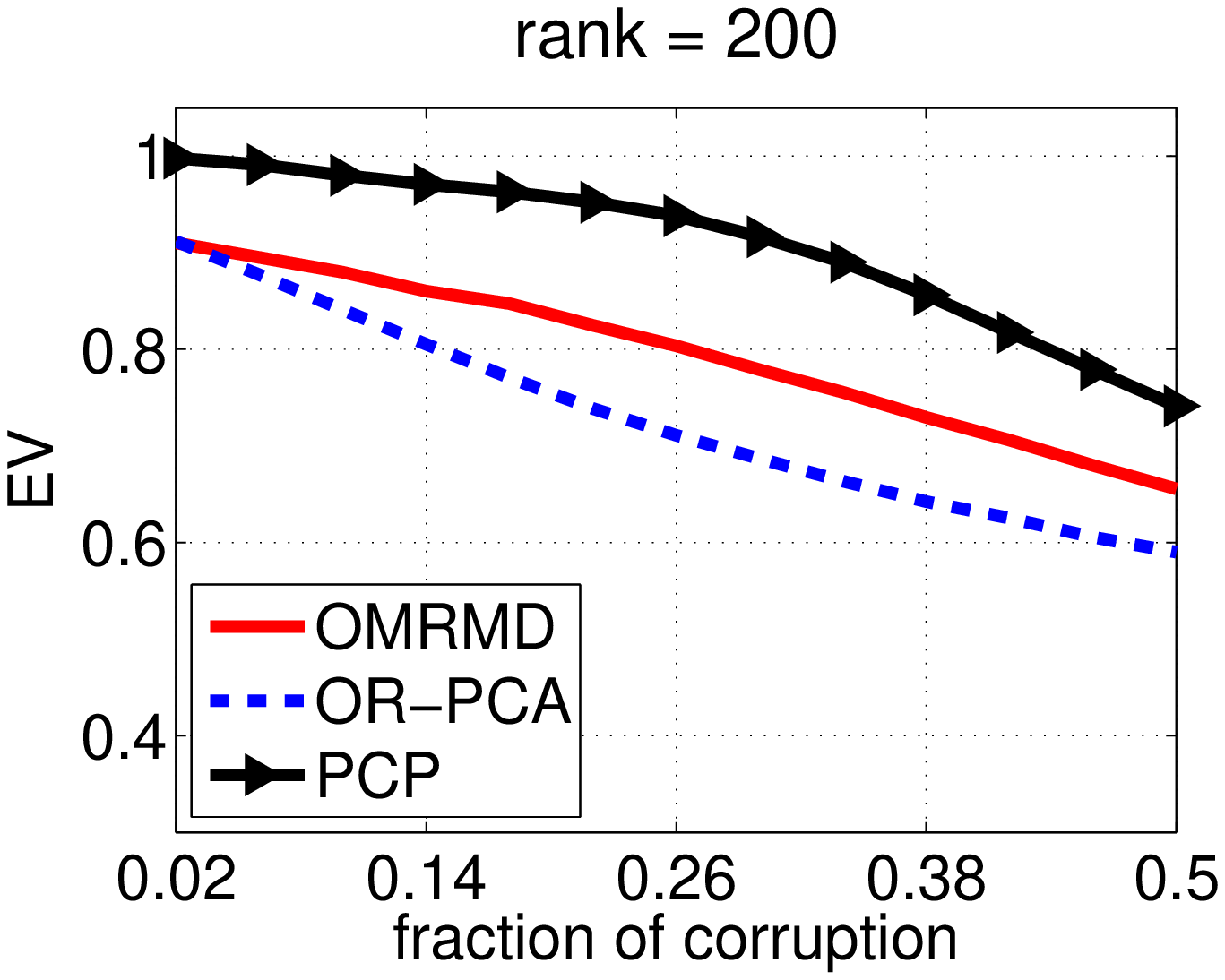}
		\label{fig:rank 200}
	}
	\caption{EV value against corruption fractions when the matrix has a middle level of rank (note that the ambient dimension $p$ is 400). The EV value is computed by the basis after accessing the last sample. In these cases, OR-PCA degrades as soon as the corruption is tuned to be higher than 0.02.}
	\label{fig:diff_rank_rho_midrank}
\end{figure*}

\subsection{Robustness}
We first study the robustness of OMRMD, measured by the EV value of its output after accessing the last sample, and compare it to the nuclear norm based OR-PCA and the batch algorithm PCP. In order to make a detailed examination, we vary the intrinsic dimension $d$ from  $0.02p$ to  $0.5p$, with a step size $0.04p$, and the corruption fraction $\rho$ from  $0.02$ to $0.5$, with a step size $0.04$. 

The general results are reported in Figure~\ref{fig:diff_rank_rho} where a brighter color means a higher EV (hence better performance). We observe that for easy tasks (\emph{i.e.}, when corruption and rank are low), both OMRMD and OR-PCA perform comparably. On the other hand, for more difficult cases, OMRMD outperforms OR-PCA. In order to further investigate this phenomenon, we plot the EV curve against the fraction of corruption under a given matrix rank. In particular, we group the results into two parts, one with relatively low rank (Figure~\ref{fig:diff_rank_rho_lowrank}) and the other with middle level of rank (Figure~\ref{fig:diff_rank_rho_midrank}). Figure~\ref{fig:diff_rank_rho_lowrank} indicates that when manipulating a low-rank matrix, OR-PCA works as well as OMRMD under a low level of noise. For instance, the EV produced by OR-PCA is as close as that of OMRMD for rank less than 40 and $\rho$ no more than 0.26. However, when the rank becomes larger, OR-PCA degrades quickly compared to OMRMD. This is possibly because the max-norm is a tighter approximation to the matrix rank. Since PCP is a batch formulation and accesses all the data in each iteration, it always achieves the best recovery performance.

\begin{figure*}[!htb]
	\centering
	\subfloat[$\rho = 0.01$]{
		\includegraphics[width=0.4\linewidth]{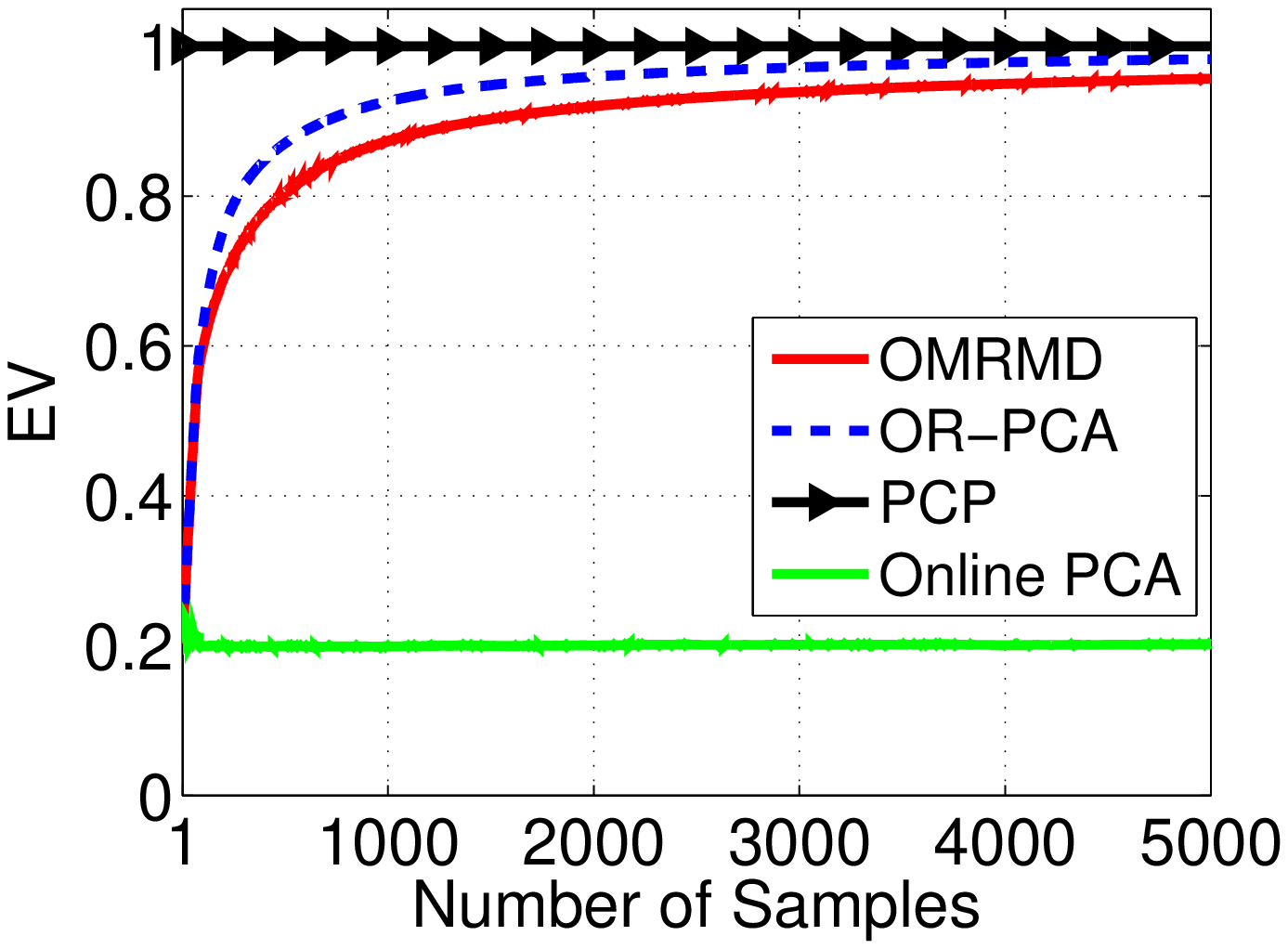}
		\label{fig:ev:0.01}
	}
	\subfloat[$\rho = 0.1$]{
		\includegraphics[width=0.4\linewidth]{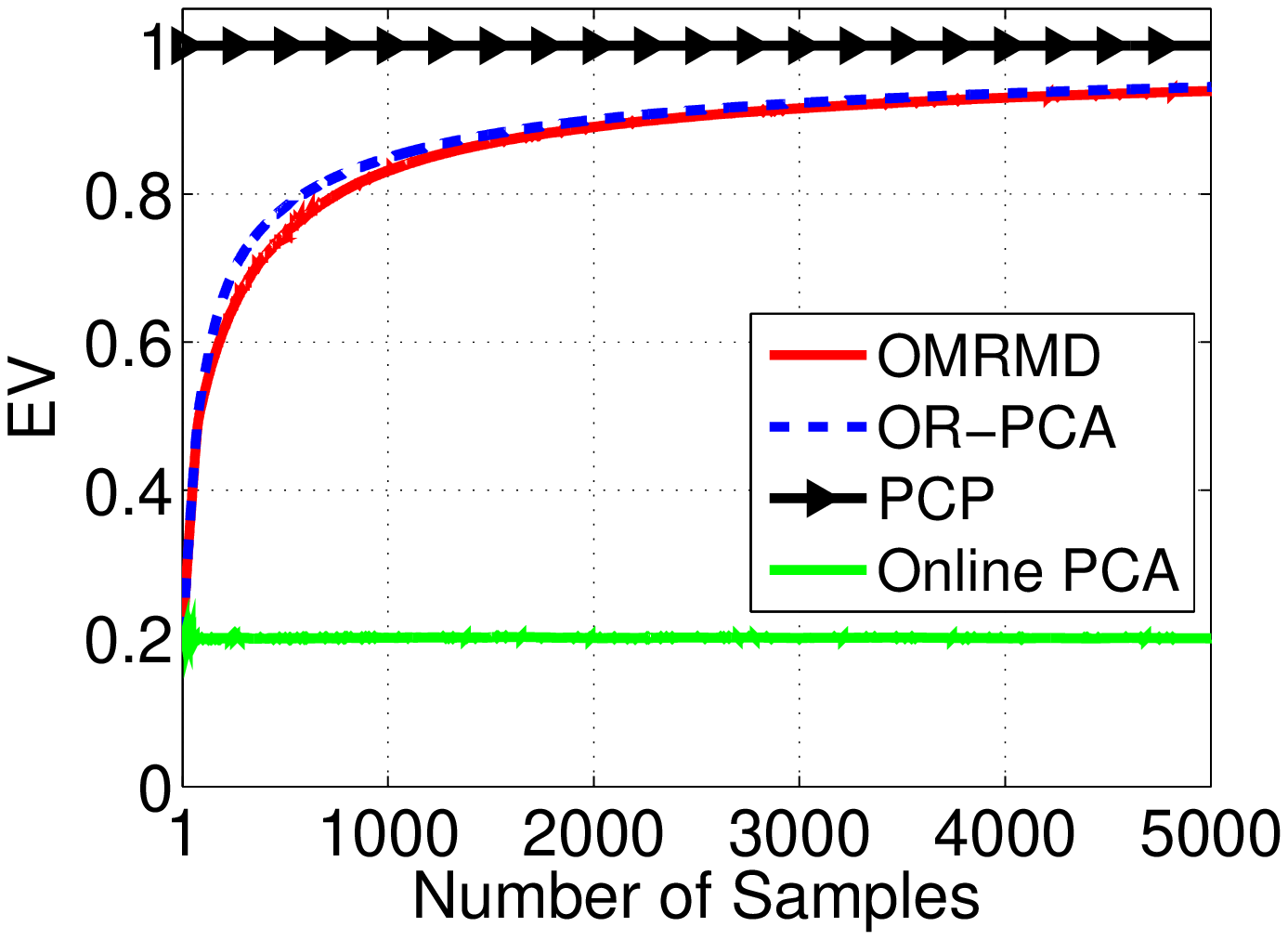}
		\label{fig:ev:0.1}
	}
	
	\subfloat[$\rho = 0.3$]{
		\includegraphics[width=0.4\linewidth]{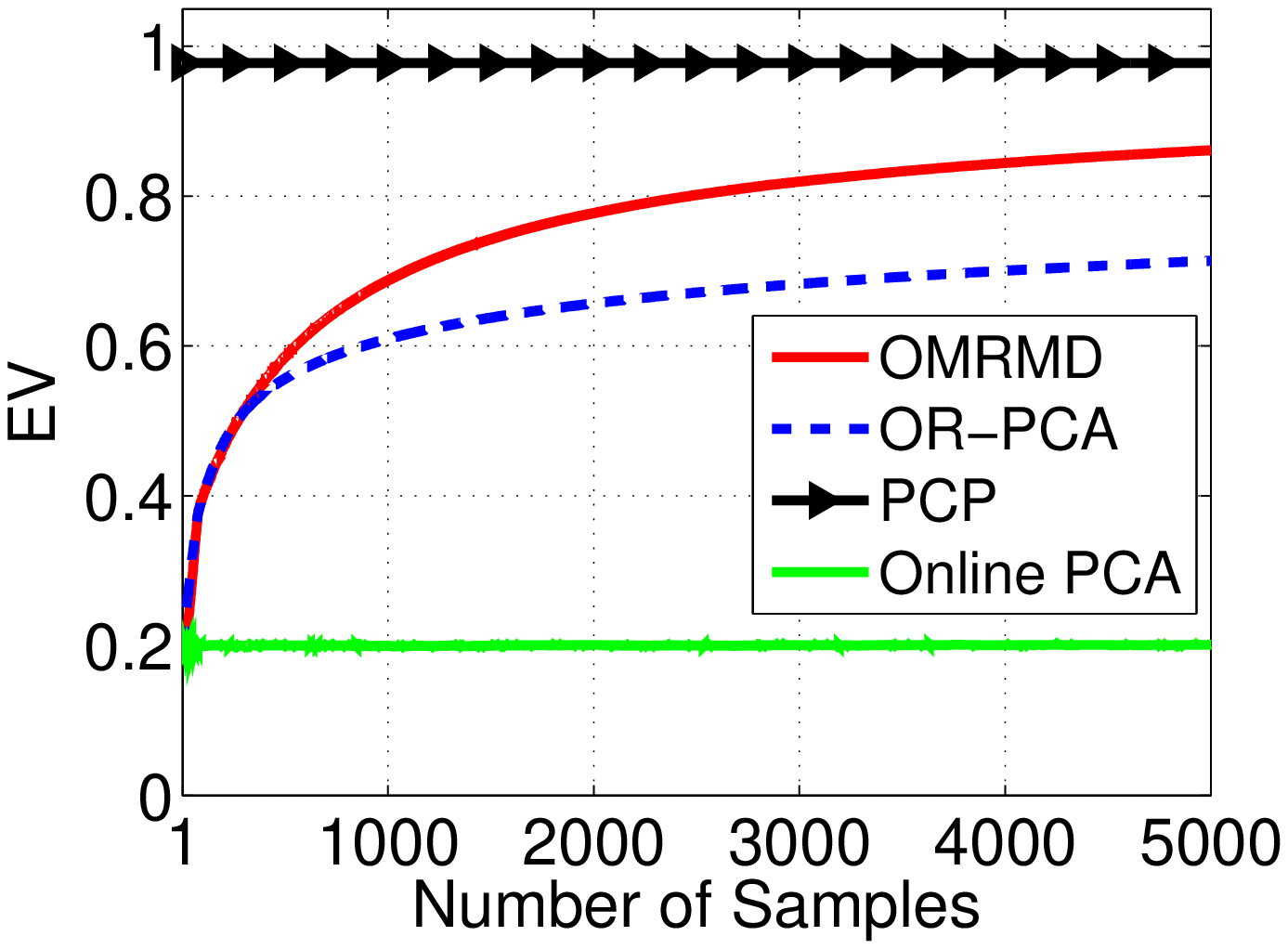}
		\label{fig:ev:0.3}
	}
	\subfloat[$\rho = 0.5$]{
		\includegraphics[width=0.4\linewidth]{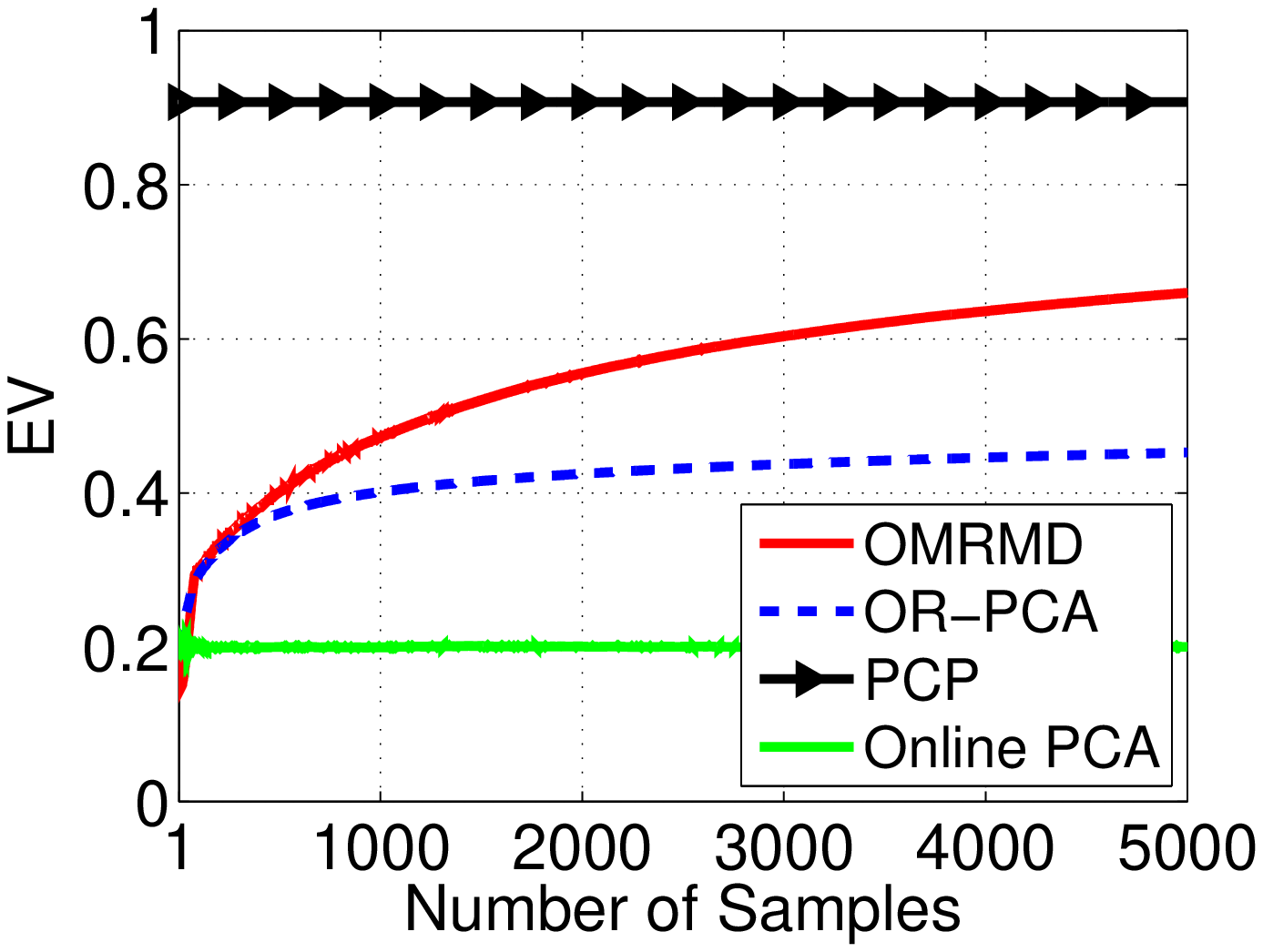}
		\label{fig:ev:0.5}
	}
	\caption{ EV value against number of samples under different corruption fractions. PCP outperforms all the online algorithms before they converge since PCP accesses all the data to estimate the basis. The performance of Online PCA is significantly degraded even when there is little corruption. For hard tasks ($\rho$ equal to 0.3 or higher), we again observe the superiority of the max-norm over the nuclear norm.}
	\label{fig:ev_samples}
\end{figure*}

\begin{figure*}[!htb]
	\centering
	\subfloat[$p = 400$]{
		\includegraphics[width=0.33\linewidth]{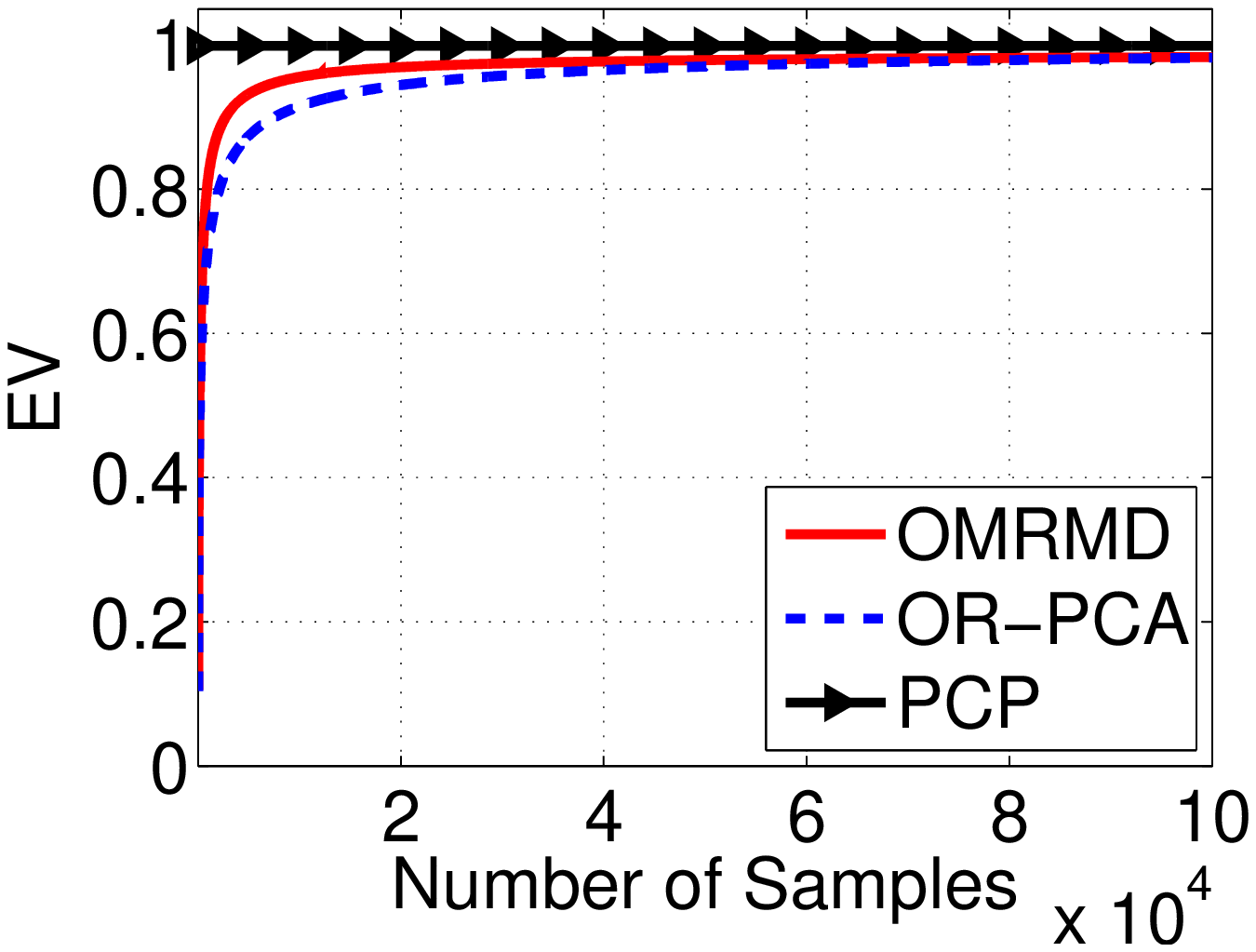}
		\label{fig:large_sample400}
	}
	\subfloat[$p = 1000$]{
		\includegraphics[width=0.33\linewidth]{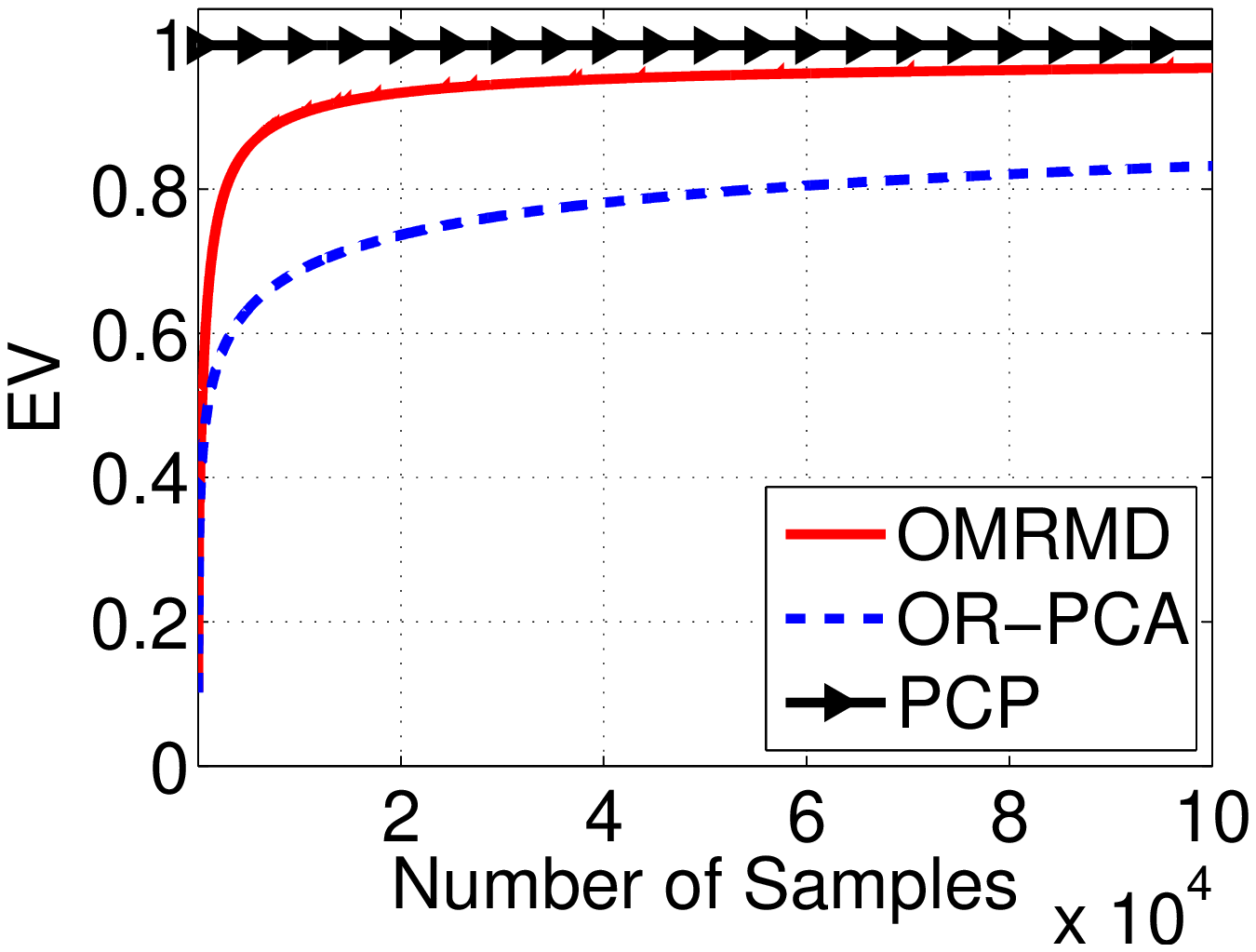}
		\label{fig:large_sample1000}
	}
	\subfloat[$p = 3000$]{
		\includegraphics[width=0.33\linewidth]{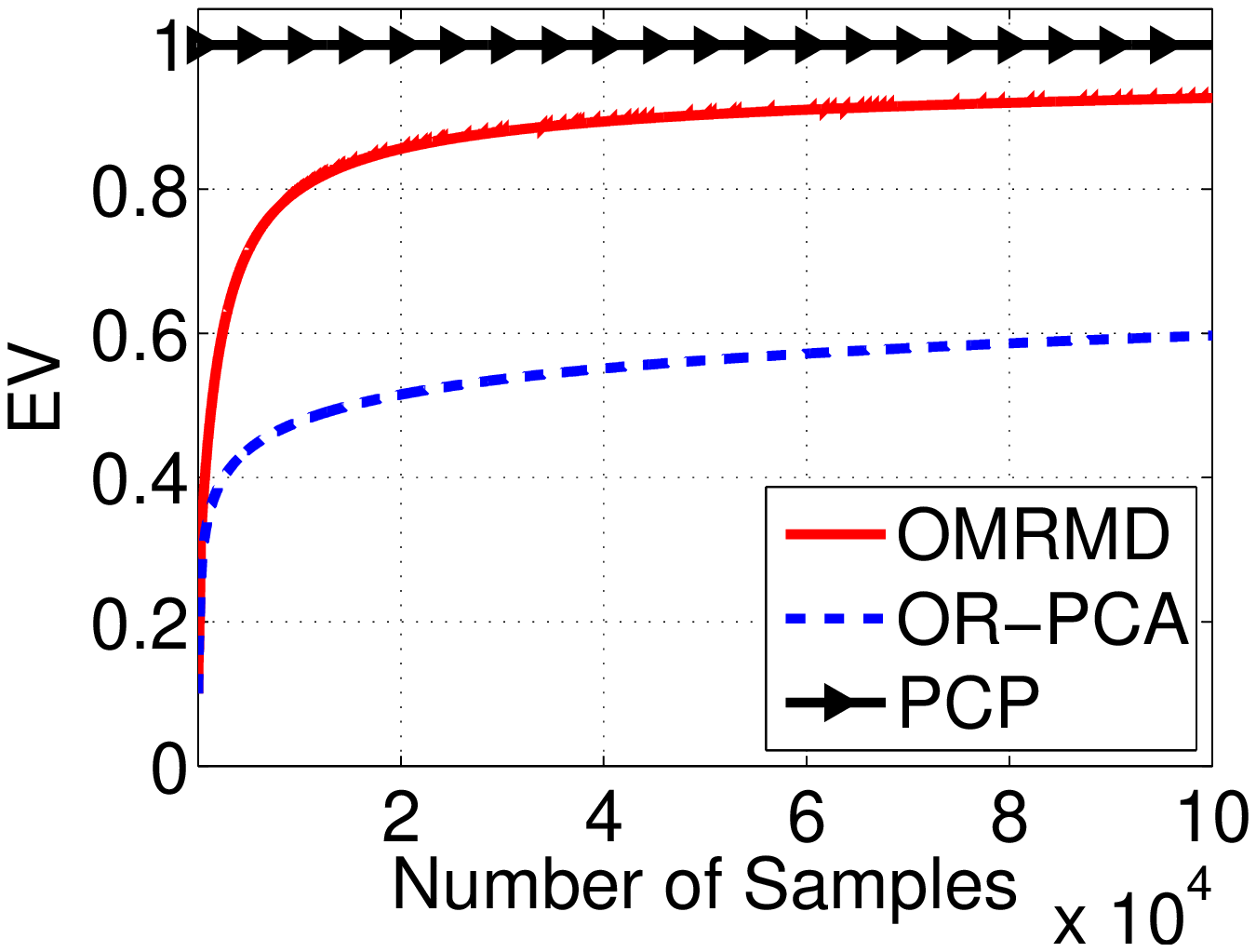}
		\label{fig:large_sample3000}
	}
	\caption{EV value against number of samples under different ambient dimensions. The intrinsic dimension $d = 0.1p$ and the corruption fraction $\rho = 0.3$.
	}
	\label{fig:large_sample}
\end{figure*}

\subsection{Convergence Rate}
We next study the convergence of OMRMD by plotting the EV curve against the number of samples. Besides OR-PCA and PCP, we also add online PCA~\cite{artac2002incremental} as a baseline algorithm. The results are illustrated in Figure~\ref{fig:ev_samples}. As expected, PCP achieves the best performance since it is a batch method and needs to access all the data throughout the algorithm. Online PCA degrades significantly even with low corruption (Figure~\ref{fig:ev:0.01}).  OMRMD is comparable to OR-PCA when the corruption is low (Figure~\ref{fig:ev:0.01}), and converges significantly faster when the corruption is high (Figure~\ref{fig:ev:0.3} and~\ref{fig:ev:0.5}). This observation agrees with Figure~\ref{fig:diff_rank_rho}, and again suggests that for large corruption, max-norm may be a better fit than the nuclear norm.

Indeed, it is true that OMRMD converges much faster even in large scale problems. In Figure~\ref{fig:large_sample}, we compare the convergence rate of OMRMD and OR-PCA under different ambient dimensions. The intrinsic dimensions are set with $0.1p$, indicating a low-rank structure of the underlying data. The error corruption $\rho$ is fixed with $0.3$~--~a difficult task for recovery. We observe that for high dimensional cases ($p = 1000$ and $p = 3000$), OMRMD significantly outperforms OR-PCA. For example, in Figure~\ref{fig:large_sample1000}, OMRMD achieves the EV value of $0.8$ only with accessing about 2000 samples, while OR-PCA needs to access $60, 000$ samples to obtain the same accuracy!

\begin{figure*}[!htb]
	\centering
	\subfloat[$p = 400$]{
		\includegraphics[width=0.33\linewidth]{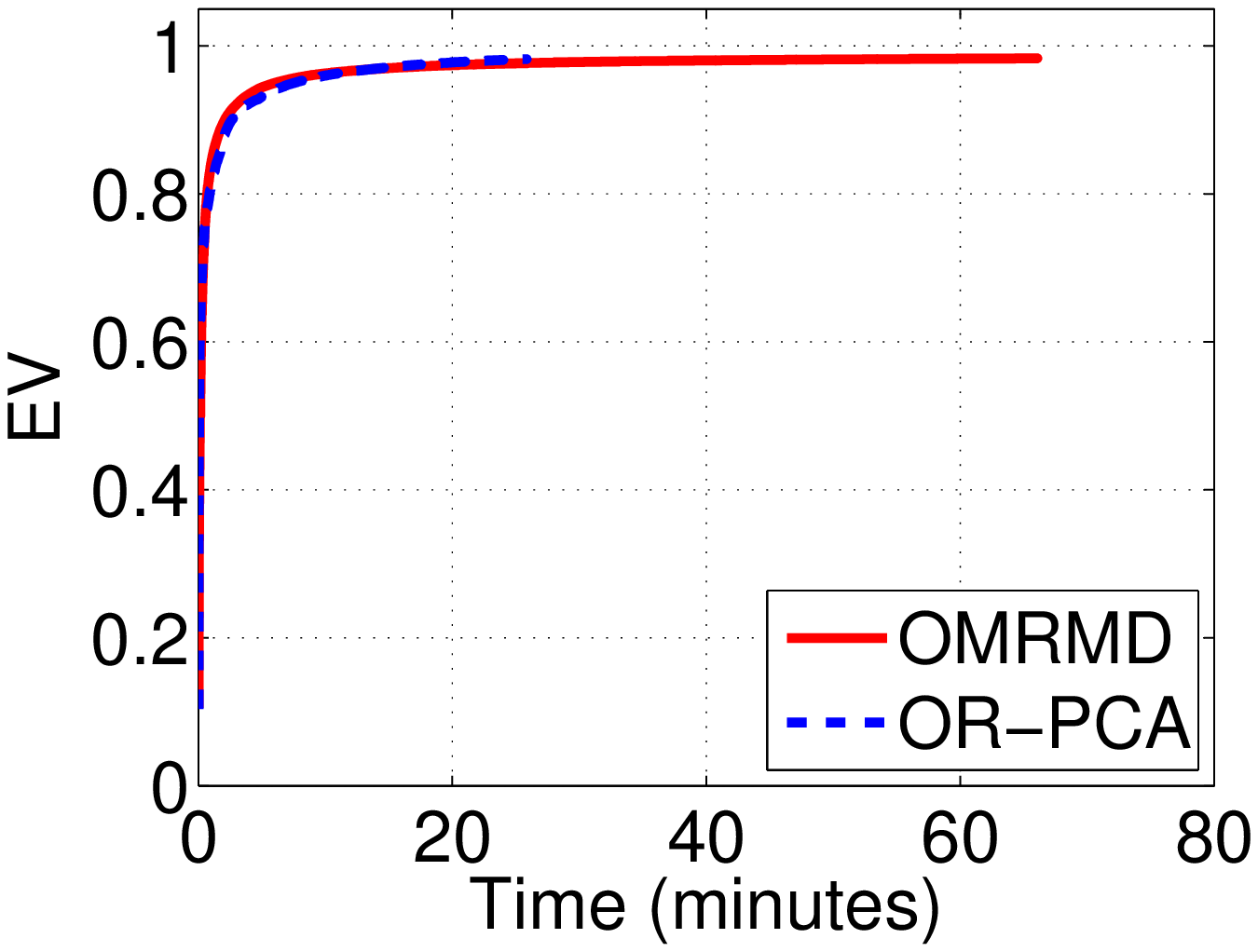}
		\label{fig:large_time400}
	}
	\subfloat[$p = 1000$]{
		\includegraphics[width=0.33\linewidth]{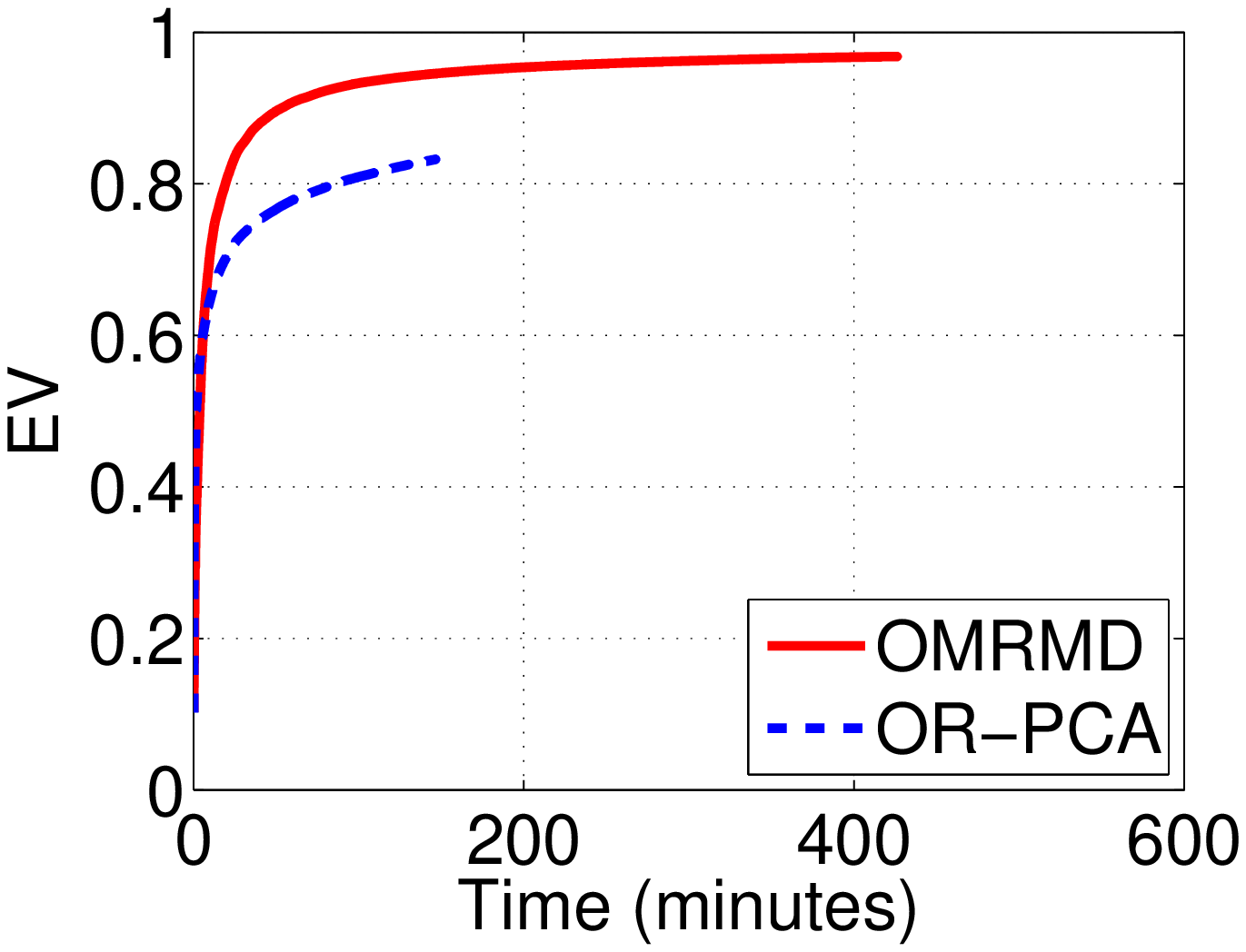}
		\label{fig:large_time1000}
	}
	\subfloat[$p = 3000$]{
		\includegraphics[width=0.33\linewidth]{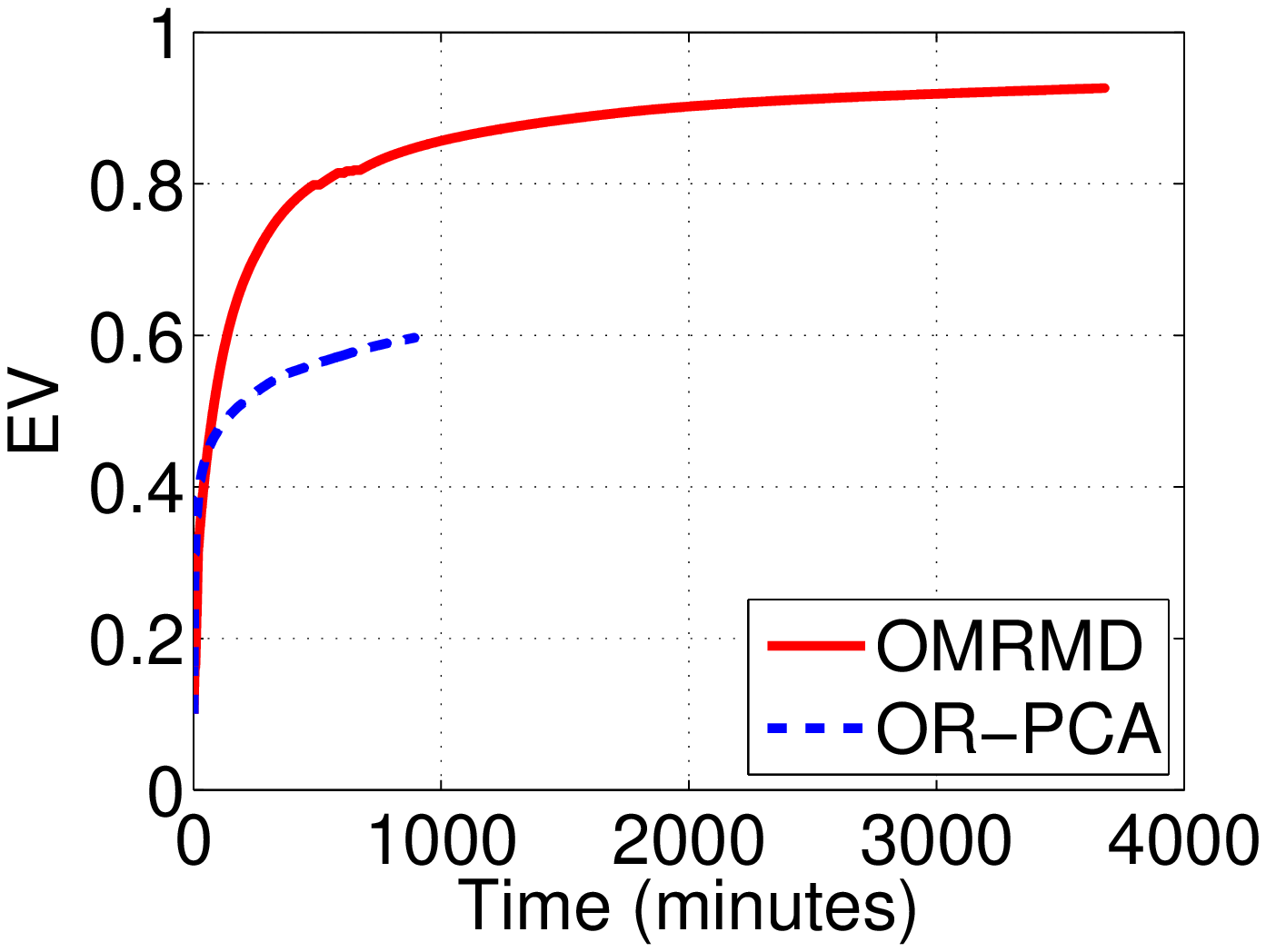}
		\label{fig:large_time3000}
	}
	\caption{EV value against time under different ambient dimensions. The intrinsic dimension $d$ is set as $0.1p$ and the corruption fraction $\rho$ equals 0.3.}
	\label{fig:large_time}
\end{figure*}

\subsection{Computational Complexity}
We note that our OMRMD is a bit inferior to OR-PCA in terms of computation in each iteration, as our algorithm may solve a dual problem to optimize $\br$~(see Algorithm~\ref{alg:re}). Therefore, our algorithm will spend more time to process an instance if the initial solution $\br_0$ violates the constraint. We plot the EV curve with respect to the running time in Figure~\ref{fig:large_time}. It shows that basically, OR-PCA is about $3$ times faster than OMRMD per sample. However, we point out here that we mainly emphasize on the convergence rate. That is, given an EV value, how much time the algorithm will cost to achieve it. In Figure~\ref{fig:large_time3000}, for example, OMRMD takes $50$ minutes to achieve the EV value of 0.6, while OR-PCA uses nearly $900$ minutes. From Figure~\ref{fig:large_sample} and Figure~\ref{fig:large_time}, it is safe to say that OMRMD is superior to OR-PCA in terms of convergence rate in the price of a little more computation per sample.

%% file: proof.tex
\section{Proof Details}
\label{supp:sec:proof}
\subsection{Proof for Stage \textbf{I}}
First we prove that all the stochastic variables are uniformly bounded.
\begin{proposition}
Let $\br_t$, $\be_t$ and $L_t$ be the optimal solutions produced by Algorithm~\ref{alg:all}. Then,

\begin{enumerate}
\item The optimal solutions $\br_t$ and $\be_t$  are uniformly bounded.
%\label{prop:bound:re}

\item The matrices $\frac{1}{t}A_t$ and $\frac{1}{t}B_t$ are uniformly bounded.
%\label{prop:bound:AB}

\item There exists a compact set $\mathcal{L}$, such that for all $L_t$ produced by Algorithm~\ref{alg:all}, $L_t \in \mathcal{L}$. Namely, there exists a positive constant $L_{\max}$ that is uniform over $t$, such that for all $t > 0$,
\begin{equation*}
\fronorm{L_t} \leq L_{\max}.
\end{equation*}
\end{enumerate}
\end{proposition}
\begin{proof}
Note that for each $t>0$, $\twonorm{\br_t} \leq 1$. Thus $\br_t$ is uniformly bounded. Let us consider the optimization problem~\eqref{eq:solve_re}. As the trivial solution $\br_t = \mathbf{0}$ and $\be_t = \mathbf{0}$ are feasible, we have
\begin{equation*}
\tilde{\ell}(\bz_t, L_{t-1}, \mathbf{0}, \mathbf{0}) = \frac{1}{2} \twonorm{\bz_t}^2.
\end{equation*}
Therefore, the optimal solution should satisfy:
\begin{equation*}
\frac{1}{2} \twonorm{\bz_t - L_{t-1}\br_t - \be_t}^2 + \lambda_2 \onenorm{\be_t} \leq \frac{1}{2} \twonorm{\bz_t}^2,
\end{equation*}
which implies
\begin{equation*}
\onenorm{\be_t} \leq \frac{1}{2 \lambda_2} \twonorm{\bz_t}^2.
\end{equation*}
Since $\bz_t$ is uniformly bounded (Assumption~\ref{as:z}), $\be_t$ is uniformly bounded.

To examine the uniform bound for $\frac{1}{t}A_t$ and $\frac{1}{t}B_t$, note that
\begin{equation*}
\begin{split}
&\frac{1}{t}A_t = \frac{1}{t} \sum_{i=1}^t \br_i \br_i\trans,\\
&\frac{1}{t}B_t = \frac{1}{t} \sum_{i=1}^t \(\bz_i - \be_i\) \br_i\trans.
\end{split}
\end{equation*}
Since for each $i$, $\br_i$, $\be_i$ and $\bz_i$ are uniformly bounded, $\frac{1}{t} A_t $ and $\frac{1}{t}B_t$ are uniformly bounded. 

Based on Claim 1 and Claim 2, we prove that $L_t$ can be uniformly bounded.
First let us denote $\frac{1}{t}A_t$ and $\frac{1}{t}B_t$ by $\widetilde{A}_t$ and $\widetilde{B}_t$ respectively.

\textbf{Step 1:}
According to Claim 2, there exist constants $a_1$ and $b$ that are uniform over $t$, such that
\begin{equation*}
\label{eq:prop:bound:L:A<a}
\begin{split}
\fronorm{\widetilde{A}_t} &\leq a_1,\\
\fronorm{\widetilde{B}_t} &\leq b.
\end{split}
\end{equation*}

On the other hand, from Assumption~\ref{as:g_t(L)}, the eigenvalues of $\widetilde{A}_t$ is lower bounded by a positive constant $\beta_1$ that is uniform over $t$, implying the trace norm (sum of the singular values) of $\widetilde{A}_t$ is uniformly lower bounded by a positive constant. As all norms are equivalent, we can show that
\begin{equation*}
\fronorm{\widetilde{A}_t} \geq a_0 > 0,
\end{equation*}
where $a_0$ is a positive constant which is uniform over $t$.

Recall that $L_t$ is the optimal basis for Eq.~\eqref{eq:alg:all:solve L}. Thus, the subgradient of the objective function taken on $L_t$ should contain zero, that is,
\begin{equation*}
\begin{split}
L_t \widetilde{A}_t - \widetilde{B}_t + \frac{\lambda_1}{t} U_t = 0,\\
L_t \widetilde{A}_t = \widetilde{B}_t - \frac{\lambda_1}{t} U_t,
\end{split}
\end{equation*}
where $U_t$ is the subgradient of $\frac{1}{2}\| L_t \|_{\tinf}^2$ produced by Eq.\eqref{eq:subgradient}. Note that, as all of the eigenvalues of $\widetilde{A}_t$ are lower bounded by a positive constant, $\widetilde{A}_t$ is invertible. Thus,
\begin{equation*}
L_t = \(\widetilde{B}_t - \frac{\lambda_1}{t} U_t\) \widetilde{A}_t^{-1},
\end{equation*}
where $\widetilde{A}_t^{-1}$ is the inverse of $\widetilde{A}_t$.
Now we derive the bound for $L_t$:
\begin{equation*}
\begin{split}
\fronorm{L_t} =& \fronorm{\(\widetilde{B}_t - \frac{\lambda_1}{t} U_t\) \widetilde{A}_t^{-1}}\\
\leq& \fronorm{\widetilde{B}_t - \frac{\lambda_1}{t} U_t \Vert_F \cdot \Vert \widetilde{A}_t^{-1}}\\
\leq& \( \fronorm{\widetilde{B}_t}  + \frac{\lambda_1}{t} \fronorm{U_t} \) \fronorm{\widetilde{A}_t^{-1}}\\
=& \fronorm{\widetilde{A}_t^{-1}} \fronorm{\widetilde{B}_t} + \frac{\lambda_1}{t} \fronorm{\widetilde{A}_t^{-1}} \fronorm{U_t} \\
\leq& \fronorm{\widetilde{A}_t^{-1}} \fronorm{\widetilde{B}_t} + \frac{\lambda_1}{t} \fronorm{\widetilde{A}_t^{-1}} \fronorm{L_t}.
\end{split}
\end{equation*}
It follows that
\begin{equation*}
\(1- \frac{\lambda_1}{t}\fronorm{\widetilde{A}_t^{-1}} \) \fronorm{L_t} \leq  \fronorm{\widetilde{A}_t^{-1}} \fronorm{\widetilde{B}_t}.
\end{equation*}
As all of the eigenvalues of $\widetilde{A}_t$ are uniformly lower bounded, those of $\widetilde{A}_t^{-1}$ are uniformly upper bounded. Thus the trace norm of $\widetilde{A}_t^{-1}$ are uniformly upper bounded. As all norms are equivalent, $\Vert \widetilde{A}_t^{-1} \Vert_F$ is also uniformly upper bounded by a constant, say $a_2$. Thus,
\begin{equation*}
\(1 - \frac{\lambda_1}{t}a_2 \) \fronorm{L_t} \leq \(1 - \frac{\lambda_1}{t}\fronorm{\widetilde{A}_t^{-1}}  \) \fronorm{L_t} \leq \fronorm{\widetilde{A}_t^{-1}} \fronorm{\widetilde{B}_t} \leq a_2 b
\end{equation*}
Particularly, let 
\begin{equation*}
t_0 = \min_t\left\{t \geq 2\lambda_1 a_2, t\ \mathrm{is\ an\ integer}\right\}.
\end{equation*}
Then, for all $t \geq t_0$,
\begin{equation}
\lV L_t \rV_F \leq 2 a_2b.
\end{equation}

\textbf{Step 2:}
Now let us consider a uniform bound for $L_t$, with $0<t < t_0$.
Recall that $L_t$ is the minimizer for $g_t(L)$, that is
\begin{equation*}
\begin{split}
L_t =& \argmin_L g_t(L) \\
=& \argmin_L \frac{1}{t} \sum_{i=1}^t \( \tildeli \) + \frac{\lambda_1}{2t} \twoinfnorm{L}^2 \\
=& \argmin_L \sum_{i=1}^t \frac{1}{2} \twonorm{\bz_i - L \br_i -\be_i}^2 + \frac{\lambda_1}{2} \twoinfnorm{L}^2 \\
\defeq& \argmin_L \tilde{g}_t(L).
\end{split}
\end{equation*}
Consider a trivial but feasible solution with $L = 0$,
\begin{equation*}
\tilde{g}_t(0) = \sum_{i=1}^t \frac{1}{2} \twonorm{\bz_i -\be_i}^2.
\end{equation*}
Thus,
\begin{equation*}
\begin{split}
&\tilde{g}_t(L_t) \leq \tilde{g}_t(0),\\
\Rightarrow&\sum_{i=1}^t \frac{1}{2} \twonorm{\bz_i - L_t \br_i -\be_i}^2 + \frac{\lambda_1}{2} \twoinfnorm{L_t}^2 \leq \sum_{i=1}^t \frac{1}{2} \twonorm{\bz_i -\be_i}^2 \\
\Rightarrow& \frac{\lambda_1}{2} \twoinfnorm{L_t}^2 \leq \sum_{i=1}^t \frac{1}{2} \twonorm{\bz_i -\be_i}^2 \\
\Rightarrow& \twoinfnorm{L_t}^2 \leq \frac{1}{\lambda_1}\sum_{i=1}^t \twonorm{\bz_i -\be_i}^2 \\
\Rightarrow& \fronorm{L_t}^2 \leq p \twoinfnorm{L_t}^2 \leq \frac{p}{\lambda_1}\sum_{i=1}^t \twonorm{\bz_i -\be_i}^2 \\
\Rightarrow& \fronorm{L_t} \leq  \sqrt{\frac{p}{\lambda_1} \sum_{i=1}^t \twonorm{\bz_i -\be_i}^2}.
\end{split}
\end{equation*}
For all $0 < t < t_0$,
\begin{equation}
\fronorm{L_t} \leq  \sqrt{\frac{p}{\lambda_1} \sum_{i=1}^t \twonorm{\bz_i -\be_i}^2} \leq  \sqrt{\frac{p}{\lambda_1} \sum_{i=1}^{t_0} \twonorm{\bz_i -\be_i}^2}.
\end{equation}
Note that each term, particularly $t_0$, can be uniformly upper bounded, thus $\sqrt{\frac{p}{\lambda_1} \sum_{i=1}^{t_0} \twonorm{\bz_i -\be_i}^2}$ can also be uniformly upper bounded. Namely, for all $0 <t <t_0$, $L_t$ is also uniformly upper bounded.

\textbf{Step 3:}
Now let us define
\begin{equation*}
L_{\max} = \max\left\{ 2a_2b, \sqrt{\frac{p}{\lambda_1} \sum_{i=1}^{t_0} \twonorm{\bz_i -\be_i}^2}\right\}.
\end{equation*}

Then, for all $t > 0$,
\begin{equation*}
\lV L_t \rV_F \leq L_{\max}.
\end{equation*}

\end{proof}

\begin{remark}
We remark some critical points in the third claim of Proposition~\ref{prop:bound:reABL}. All the constants, $a_0$, $a_1$, $a_2$ and $b$ are independent from $t$, making them uniformly bounded. Also, $t_0$ is a constant that is uniform over $t$. Thus, $L_t$ can be uniformly bounded.
\end{remark}

\begin{corollary}
%\label{coro:bound l lip gt}
Let $\br_t$, $\be_t$ and $L_t$ be the optimal solutions produced by Algorithm~\ref{alg:all}. We show some uniform boundedness property here.

\begin{enumerate}
\item  $\tilde{\ell}\(\bz_t, L_t, \br_t, \be_t\)$ defined in Eq.~\eqref{eq:tildel} and $\ell\(\bz_t, L_t\)$ defined in Eq.~\eqref{eq:l(z,L)} are both uniformly bounded.

\item The surrogate function, \emph{i.e.}, $g_t(L_t)$ defined in Eq.~\eqref{eq:g_t(L)} is uniformly bounded.

\item Moreover, $g_t(L)$ is uniformly Lipschitz over the compact set $\mathcal{L}$.
\end{enumerate}
\end{corollary}
\begin{proof}
The uniform bound of $\br_t$, $\be_t$ and $\bz_t$, combined with the uniform bound of $L_t$, implies the uniform boundedness for $\tilde{\ell}\(\bz_t, L_t, \br_t, \be_t\)$ and $\ell\(\bz_t, L_t\)$. Thus, $g_t(L_t)$ and $f_t(L_t)$ are also uniformly bounded. 

To show that $g_t(L)$ is uniformly Lipschitz, we compute its subgradient at any $L \in \mathcal{L}$:
\begin{equation*}
\begin{split}
\fronorm{\nabla^{}_L g_t(L)} =& \fronorm{\frac{1}{t}(LA_t - B_t) + \frac{\lambda_1}{t}U}\\
\leq& \fronorm{\frac{1}{t}(LA_t - B_t)} + \frac{\lambda_1}{t} \fronorm{L} \\
\leq& \fronorm{\frac{1}{t}(LA_t - B_t)} + \lambda_1 \fronorm{L} \\
\end{split}
\end{equation*}
where $U \in \partial \frac{1}{2} \twoinfnorm{L}$. Since $L$, $\frac{1}{t}A_t$ and $\frac{1}{t}B_t$ are all uniformly bounded, the subgradient of $g_t(L)$ is uniformly bounded. This implies that $g_t(L)$ is uniformly Lipschitz.

\end{proof}

\subsection{Proof for Stage \textbf{II}}

\begin{lemma}[A corollary of Donsker theorem~\cite{van2000asymptotic}]
\label{lem:donsker}
Let $F = \{f_{\theta}: \mathcal{X} \rightarrow \mathbb{R}, \theta \in \Theta\}$ be a set of measurable functions indexed by a bounded subset $\Theta$ of $\mathbb{R}^d$. Suppose that there exists a constant $K$ such that
\begin{equation*}
\lv f_{\theta_1}(x) - f_{\theta_2}(x) \rv \leq K \twonorm{\theta_1 - \theta_2},
\end{equation*}

for every $\theta_1$ and $\theta_2$ in $\Theta$ and $x$ in $\mathcal{X}$. Then, $F$ is P-Donsker. For any $f$ in $F$, let us define $\mathbb{P}_nf$, $\mathbb{P}f$ and $\mathbb{G}_nf$ as
\begin{equation*}
\mathbb{P}_nf = \frac{1}{n}\sum_{i=1}^n f(X_i),\ \mathbb{P}f = \EXP[f(X)],\ \mathbb{G}_nf = \sqrt{n}(\mathbb{P}_nf-\mathbb{P}f).
\end{equation*}

Let us also suppose that for all $f$, $\mathbb{P}f^2 < \delta^2$ and $\lV f \rV_{\infty} < M$ and that the random elements $X_1, X_2, \cdots$ are Borel-measurable. Then, we have
\begin{equation*}
\EXP \fronorm{\mathbb{G}} = O(1),
\end{equation*}
where $\fronorm{\mathbb{G}} = \sup_{f\in F} \lv \mathbb{G}_nf \rv$.
\end{lemma}

Now let us verify that the set of functions $\{ \ell(\bz, L), L \in \mathcal{L} \}$ indexed by $L$ suffices the hypotheses in the corollary of Donsker Theorem. In particular, we should verify that:
\begin{itemize}
\item The index set $\mathcal{L}$ is uniformly bounded (see Proposition~\ref{prop:bound:reABL}).
\item Each $\ell(\bz, L)$ can be uniformly bounded (see Corollary~\ref{coro:bound l lip gt}).
\item Any of the functions $\ell(\bz, L)$ in the family is uniformly Lipschitz (see Proposition~\ref{prop:l:Lipschtiz}).
\end{itemize}

Next, we show that the family of functions $\ell(\bz, L)$ is uniformly Lipschitz w.r.t. $L$. We introduce the following lemma as it will be useful for our discussion.

\begin{lemma}[Corollary of Theorem 4.1 from~\cite{bonnans1998optimization}] 
\label{lem:1}
Let $f: \Rp \times \Rq \rightarrow \mathbb{R}$. Suppose that for all $\bx \in \Rp$ the function $f(\bx, \cdot)$ is differentiable, and that $f$ and $\nabla_{\bu}f(\bx, \bu)$ are continuous on $\Rp \times \Rq$. Let $\bv(\bu)$ be the optimal value function $\bv(\bu) = \min_{\bx \in \mathcal{C}}f(\bx, \bu)$, where $\mathcal{C}$ is a compact subset of $\Rp$. Then $\bv(\bu)$ is directionally differentiable. Furthermore, if for $\bu_0 \in \Rq$, $f(\cdot, \bu_0)$ has unique minimizer $\bx_0$ then $\bv(\bu)$ is differentiable in $\bu_0$ and $\nabla_{\bu} \bv(\bu_0) = \nabla_{\bu} f(\bx_0, \bu_0)$.
\end{lemma}

\begin{proposition}
%\label{prop:l:Lipschtiz}
Let $L \in \mathcal{L}$ and denote the minimizer of $\tilde{\ell}(\bz, L, \br, \be)$ defined in \eqref{eq:l(z,L)} as:
\begin{equation*}
\{\br^*, \be^*\} = \argmin_{\br, \be, \twonorm{\br} \leq 1} \tildel.
\end{equation*}
 Then, the function $\ell(\bz, L)$ defined in Problem~\eqref{eq:l(z,L)} is continuously differentiable and
\begin{equation*}
\nabla_L \ell(\bz, L) = (L\br^* + \be^* - \bz) \br^{*\top}.
\end{equation*}
Furthermore, $\ell(\bz, \cdot)$ is uniformly Lipschitz.
\end{proposition}
\begin{proof}
By fixing the variable $\bz$, the function $\tilde{\ell}$ can be seen as a mapping: 
\begin{equation*}
\begin{split}
\mathbb{R}^{d+p} \times \mathcal{L} \rightarrow \mathbb{R}\\
\( [\br;\ \be], L \) \mapsto \tilde{\ell}(\bz, L, \br, \be)
\end{split}
\end{equation*}
It is easy to show that $\forall [\br;\ \be] \in \mathbb{R}^{d+p}$, $\tilde{\ell}(\bz, \cdot, \br, \be)$ is differentiable. Also $\tilde{\ell}(\bz, \cdot, \cdot, \cdot)$ is continuous on $\mathbb{R}^{d+p} \times \mathcal{L}$. $\nabla^{}_L \tilde{\ell}(\bz, L, \br, \be) = (L\br + \be - \bz)\br\trans$ is continuous on $\mathbb{R}^{d+p} \times \mathcal{L}$. $\forall L \in \mathcal{L}$, according to Assumption~\ref{as:unique}, $\tilde{\ell}(\bz, L, \cdot, \cdot)$ has a unique minimizer. Thus Lemma~\ref{lem:1} applies and we prove that $\ell(\bz, L)$ is differentiable in $L$ and
\begin{equation*}
\nabla_L \ell(\bz, L) = (L\br^* + \be^* - \bz) \br^{*\top}.
\end{equation*}
Since every term in $\nabla_L \ell(\bz, L)$ is uniformly bounded (Assumption~\ref{as:z} and Proposition~\ref{prop:bound:reABL}), we conclude that the gradient of $\ell(\bz, \cdot)$ is uniformly bounded, implying that $\ell(\bz, L)$ is uniformly Lipschitz w.r.t. $L$.

\end{proof}

\begin{corollary}
%\label{coro:bound lip ft}
Let $f_t(L)$ be the empirical loss function defined in Eq.~\eqref{eq:f_n(L)}. Then $f_t(L)$ is uniformly bounded and Lipschitz.
\end{corollary}
\begin{proof}
As $\ell(\bz, L)$ can be uniformly bounded (Corollary~\ref{coro:bound l lip gt}), we derive the uniform boundedness of $f_t(L)$. Let $ U \in \frac{1}{2}\twoinfnorm{L}$. By computing the subgradient of $f_t(L)$ at $L$, we have
\begin{equation*}
\begin{split}
\fronorm{\nabla^{}_L f_t(L)} =& \fronorm{\frac{1}{t}  \sum_{i=1}^{t}  \nabla^{}_L \ell(\bz_i, L) + \frac{\lambda_1}{t} U}\\
\leq& \frac{1}{t} \sum_{i=1}^{t} \fronorm{(L\br_i + \be_i - \bz_i) \br\trans_i} + \frac{\lambda_1}{t} \fronorm{L}\\
=& \frac{1}{t} \sum_{i=1}^{t} \fronorm{L\br_i\br\trans_i + (\be_i - \bz_i)\br\trans_i} + \frac{\lambda_1}{t} \fronorm{L}\\
\leq& \frac{1}{t} \sum_{i=1}^{t} \( \fronorm{L} \cdot \fronorm{\br_i\br\trans_i} + \fronorm{(\be_i - \bz_i)\br\trans_i} \) + \frac{\lambda_1}{t} \fronorm{L}.
\end{split}
\end{equation*}
Note that all the terms (i.e. $\bz_i$, L, $\br_i$, $\be_i$) in the right hand inequality are uniformly bounded. Thus, we say that the subgradient of $f_t(L)$ is uniformly bounded and $f_t(L)$ is uniformly Lipschitz.

\end{proof}

\begin{proposition}
%\label{prop:l donsker}
Let $f_t(L)$ and $f(L)$ be the empirical and expected loss functions we defined in Eq.~\eqref{eq:f_n(L)} and Eq.~\eqref{eq:f(L)}. Then we have
\begin{equation*}
\EXP[\sqrt{t}\lV f_t - f \rV_{\infty}] = O(1).
\end{equation*}
\end{proposition}
\begin{proof}
Based on Proposition~\ref{prop:bound:reABL} and Proposition~\ref{prop:l:Lipschtiz}, we argue that the set of measurable functions $\{ \ell(\bz, L), L \in \mathcal{L} \}$ is P-Donsker (defined in Lemma~\ref{lem:donsker}). From Corollary~\ref{coro:bound l lip gt}, we know that $\ell(\bz, L)$ can be uniformly bounded by a constant, say $\kappa_c$. Also note that from the definition of $\ell(\bz, L)$ (see Eq.\eqref{eq:l(z,L)}), it is always non-negative. Thus, we have
\begin{equation*}
\ell^2(\bz, L) \leq \kappa_c^2,
\end{equation*}
implying the uniform boundedness of $\EXP[\ell^2(\bz, L)]$. Thus, Lemma~\ref{lem:donsker} applies and we have
\begin{equation*}
\EXP[\sup_{\ell}\vert \sqrt{t}(f_t - f) \vert] = O(1).
\end{equation*}

\end{proof}

Now we are ready to prove the convergence of $g_t(L_t)$, which requires to justify that the stochastic process $\{ g_t(L_t) \}_{t=1}^{\infty}$ is a quasi-martingale, defined as follows:

\begin{lemma}[Sufficient condition of convergence for a stochastic process~\cite{bottou1998online}]
\label{lem:bottou}
Let $\( \Omega, \mathcal{F}, P \)$ be a measurable probability space, $u_t$, for $t \geq 0$, be the realization of a stochastic process and $\mathcal{F}_t$ be the filtration by the past information at time $t$. Let
\begin{equation*}
\delta_t = 
\begin{cases}
1\quad if\ \EXP[u_{t+1}-u_t \mid \mathcal{F}_t] > 0,\\
0\quad otherwise.
\end{cases}
\end{equation*}
If for all $t$, $u_t \geq 0 $ and $\sum_{t=1}^{\infty} \EXP[\delta_t(u_{t+1}-u_t)] < \infty$, then $u_t$ is a quasi-martingale and converges almost surely. Moreover,
\begin{equation*}
\sum_{t=1}^{\infty} \lv \EXP[u_{t+1} - u_t \mid \mathcal{F}_t] \rv < +\infty\ a.s.
\end{equation*}
\end{lemma}

\begin{theorem}[Convergence of the surrogate function $g_t(L_t)$]
%\label{thm:convergence gt(Lt)}
The surrogate function $g_t(L_t)$ we defined in Eq.~\eqref{eq:g_t(L)} converges almost surely, where $L_t$ is the solution produced by Algorithm~\ref{alg:all}.
\end{theorem}
\begin{proof}
For convenience, let us first define the stochastic positive process
\begin{equation*}
u_t = g_t(L_t) \geq 0.
\end{equation*}

We consider the difference between $u_{t+1}$ and $u_t$:
\begin{equation}
\label{eq:diff_ut}
\begin{split}
u_{t+1} - u_t &= g_{t+1}(L_{t+1}) - g_t(L_t)\\
&= g_{t+1}(L_{t+1}) - g_{t+1}(L_t) + g_{t+1}(L_t) - g_t(L_t)\\
&= g_{t+1}(L_{t+1}) - g_{t+1}(L_t) + \frac{1}{t+1}\ell(\bz_{t+1}, L_t) - \frac{1}{t+1}g_t(L_t)\\
&= g_{t+1}(L_{t+1}) - g_{t+1}(L_t) + \frac{f_t(L_t)-g_t(L_t)}{t+1} + \frac{\ell(\bz_{t+1}, L_t) - f_t(L_t)}{t+1}.\\
\end{split}
\end{equation}
As $L_{t+1}$ minimizes $g_{t+1}(L)$, we have
\begin{equation*}
g_{t+1}(L_{t+1}) - g_{t+1}(L_t) \leq 0.
\end{equation*}

As $g_t(L_t)$ is the surrogate function of $f_t(L_t)$, we have
\begin{equation*}
f_t(L_t)-g_t(L_t) \leq 0.
\end{equation*}

Thus,
\begin{equation}
\label{eq:diff_u_t}
u_{t+1}- u_t \leq \frac{\ell(\bz_{t+1}, L_t) - f_t(L_t)}{t+1}.
\end{equation}

Let us consider the filtration of the past information $\mathcal{F}_t$ and take the expectation of Eq.~\eqref{eq:diff_u_t} conditioned on $\mathcal{F}_t$:
\begin{equation}
\label{eq:E(diff_u)}
\begin{split}
\EXP[u_{t+1}-u_t \mid \mathcal{F}_t] & \leq \frac{\EXP[\ell(\bz_{t+1}, L_t) \mid \mathcal{F}_t] - f_t(L_t)}{t+1}\\
& \leq \frac{f(L_t) - f_t(L_t)}{t+1}\\
& =\frac{f(L_t) - f'_t(L_t) -{\lambdat} \Vert L_t \Vert_{\tinf}^2}{t+1}\\
& \leq \frac{\Vert f - f'_t \Vert_{\infty}}{t+1} - \frac{\lambda_1}{2t(t+1)}\twoinfnorm{L_t}^2\\
& \leq \frac{\Vert f - f'_t \Vert_{\infty}}{t+1},
\end{split}
\end{equation}
where
\begin{equation*}
f'_t(L) = \frac{1}{t} \sum_{i=1}^t \ell(\bz_i, L).
\end{equation*}
Note that
\begin{equation*}
f'(L) = \lim_{t \rightarrow \infty} f'_t(L) = \EXP_{\bz}[\ell(\bz, L)] = f(L).
\end{equation*}

From Proposition~\ref{prop:concentration of ft}, we have
\begin{equation*}
\EXP[\Vert \sqrt{t}(f'_t - f') \Vert_{\infty}] = O(1).
\end{equation*}
Also note that according to Proposition~\ref{prop:bound:reABL}, we have $\fronorm{L_t} \leq L_{\max}$. Thus, considering the positive part of $\EXP[u_{t+1}-u_t \mid \mathcal{F}_t]$ in Eq.~\eqref{eq:E(diff_u)} and taking the expectation, we have
\begin{equation*}
\EXP[\EXP[u_{t+1}-u_t \mid \mathcal{F}_t]^+] = \EXP[\max\{\EXP[u_{t+1}-u_t \mid \mathcal{F}_t], 0\}] \leq \frac{\kappa}{\sqrt{t}(t+1)},
\end{equation*}
where $\kappa$ is a constant.

Therefore, defining the set $\mathcal{T} = \{t \mid \EXP[u_{t+1}-u_t \mid \mathcal{F}_t] >0\}$ and 
\begin{equation*}
\delta_t = 
\begin{cases}
1\quad \textrm{if}\ t \in \mathcal{T},\\
0\quad \textrm{otherwise},
\end{cases}
\end{equation*}
we have
\begin{equation*}
\begin{split}
\sum_{t=1}^{\infty} \EXP[\delta_t(u_{t+1}-u_t)] &= \sum_{t \in \mathcal{T}} \EXP[(u_{t+1}-u_t)]\\
&= \sum_{t \in \mathcal{T}} \EXP[\EXP[u_{t+1}-u_t \mid \mathcal{F}_t]]\\
&= \sum_{t=1}^{\infty}\EXP[\EXP[u_{t+1}-u_t \mid \mathcal{F}_t]^+]\\
&< +\infty
\end{split}
\end{equation*}

According to Lemma~\ref{lem:bottou}, we conclude that $g_t(L_t)$ is a quasi-martingale and converges almost surely. Moreover,
\begin{equation}
\label{eq:sum diff_ut}
\sum_{t=1}^{\infty} \lvert \EXP[u_{t+1} - u_t \mid \mathcal{F}_t] \rvert < +\infty \ a.s.
\end{equation}

\end{proof}

\subsection{Moving to Stage III}
We now show that $g_t(L_t)$ and $f(L_t)$ converge  to the same limit almost surely. Consequently, $f(L_t)$ converges almost surely. First, we prove that $b_t \defeq g_t(L_t) - f_t(L_t)$ converges to 0 almost surely. We utilize the lemma from~\cite{mairal2010online} for the proof.

\begin{lemma}[Lemma 8 from~\cite{mairal2010online}]
\label{lem:mairal}
Let $a_t$, $b_t$ be two real sequences such that for all $t$, $a_t \geq 0$, $b_t \geq 0$, $\sum_{t=1}^{\infty}a_t = \infty$, $\sum_{t=1}^{\infty} a_t b_t < \infty$, $\exists K > 0$, such that $\lv b_{t+1} - b_t \rv < K a_t$. Then, $\lim_{t \rightarrow +\infty} b_t = 0$.
\end{lemma}

We notice that another sequence $\{a_t\}_{t=1}^{\infty}$ should be constructed in Lemma~\ref{lem:mairal}. Here, we take the $a_t = \frac{1}{t} \geq 0$, which satisfies the condition $\sum_{t=1}^{\infty}a_t = \infty$. Next, we need to show that $\lv b_{t+1} - b_t \rv < K a_t$, where $K$ is a constant. To do this, we alternatively show that $\lv b_{t+1} - b_t \rv$ can be upper bounded by $\fronorm{L_{t+1} - L_t}$, which can be further bounded by $Ka_t$.

\begin{proposition}
%\label{prop:diff_L}
Let \{$L_t$\} be the basis sequence produced by the Algorithm~\ref{alg:all}. Then,
\begin{equation*}
\fronorm{L_{t+1} - L_t} = O(\frac{1}{t}).
\end{equation*}
\end{proposition}
\begin{proof}
Let us define
\begin{equation}
\hg_t(L) = \frac{1}{t} \(\frac{1}{2} \mathrm{Tr} \( L\trans L A_t \) - \mathrm{Tr} \( L\trans B_t \) \) + \frac{\lambda_1}{2t} \twoinfnorm{L}^2.
\end{equation}
According the strong convexity of $A_t$ (Assumption~\ref{as:g_t(L)}), and the convexity of $\twoinfnorm{L}^2$, we can derive the strong convexity of $\hg_t(L)$. That is,
\begin{equation}
\hg_t(L_{t+1}) - \hg_t(L_t) \geq \langle U_t, L_{t+1} - L_t \rangle + \frac{\beta_1}{2} \fronorm{L_{t+1} - L_t}^2,
\end{equation}
where $U_t \in \partial \hg_t(L_t)$. As $L_t$ is the minimizer of $\hg_t$, we have
\begin{equation*}
0 \in \partial \hg_t(L_t).
\end{equation*}
Let $U_t$ be the zero matrix. Then we have
\begin{equation}
\label{eq:dif_hg1}
\hg_t(L_{t+1}) - \hg_t(L_t) \geq {\frac{\beta_1}{2}} \fronorm{L_{t+1} - L_t}^2.
\end{equation}
On the other hand,
\begin{equation}
\label{eq:dif_hg}
\begin{split}
\hg_t(L_{t+1}) - \hg_t(L_t) =& \hg_t(L_{t+1}) - \hg_{t+1}(L_{t+1}) + \hg_{t+1}(L_{t+1}) - \hg_{t+1}(L_t) + \hg_{t+1}(L_t) - \hg_t(L_t)\\
\leq& \hg_t(L_{t+1}) - \hg_{t+1}(L_{t+1}) + \hg_{t+1}(L_t) - \hg_t(L_t).
\end{split}
\end{equation}
Note that the inequality is derived by the fact that $\hg_{t+1}(L_{t+1}) - \hg_{t+1}(L_t) \leq 0$, as $L_{t+1}$ is the minimizer of $\hg_{t+1}(L)$. Let us denote $\hg_t(L) - \hg_{t+1}(L)$ by $G_t(L)$. We have
\begin{equation*}
\begin{split}
G_t(L) = &\frac{1}{t} \(\frac{1}{2} \mathrm{Tr} \( L\trans L A_t \) - \mathrm{Tr} \( L\trans B_t \) \) - \frac{1}{t+1} \(\frac{1}{2} \mathrm{Tr} \( L\trans L A_{t+1} \) - \mathrm{Tr} \( L\trans B_{t+1} \) \)\\
& + \frac{\lambda_1}{2t} \twoinfnorm{L}^2 - \frac{\lambda_1}{2(t+1)}\twoinfnorm{L}^2.
 \end{split}
\end{equation*}
By a simple calculation, we have the gradient of $G_t(L)$:
\begin{equation*}
\begin{split}
\nabla G_t(L) =& \frac{1}{t}\(LA_t - B_t\) - \frac{1}{t+1}\(LA_{t+1}- B_{t+1}\) + \(\frac{1}{t} - \frac{1}{t+1} \) \lambda_1 U\\
=& \frac{1}{t}\(L(A_t - \frac{t}{t+1} A_{t+1}) + \frac{t}{t+1}B_{t+1} - B_t + \frac{\lambda_1}{t+1} U \),
\end{split}
\end{equation*}
where $U \in  \partial  \twoinfnorm{L}^2$. We then compute  the Frobenius norm of the gradient of $G_t(L)$:
\begin{equation}
\begin{split}
\fronorm{\nabla G_t(L)} \leq & \frac{1}{t} \( \fronorm{L(A_t - \frac{t}{t+1} A_{t+1})} + \fronorm{\frac{t}{t+1}B_{t+1} - B_t} + \frac{\lambda_1}{t+1} \fronorm{L} \)\\
\leq & \frac{1}{t} \( \fronorm{L} \cdot \fronorm{A_t - \frac{t}{t+1} A_{t+1}} + \fronorm{ \frac{t}{t+1}B_{t+1} - B_t} + \frac{\lambda_1}{t+1} \fronorm{L} \)\\
=& \frac{1}{t} \{ \fronorm{L} \cdot \fronorm{\frac{1}{t+1}A_t - \frac{t}{t+1}\br_{t+1}\br_{t+1}\trans} \\
&+ \fronorm{\frac{1}{t+1}B_t - \frac{t}{t+1}\(\bz_{t+1}-\be_{t+1}\)\br_{t+1}\trans} + \frac{\lambda_1}{t+1} \fronorm{L} \}.
\end{split}
\end{equation}
According to the first order Taylor expansion,
\begin{equation*}
\begin{split}
G_t(L_{t+1}) - G_t(L_t) = & \tr\( \(L_{t+1} - L_t\)\trans \nabla G_t \(\alpha L_t + \(1- \alpha\)L_{t+1}\) \)\\
\leq & \fronorm{L_{t+1} - L_t} \cdot \fronorm{\nabla G_t \(\alpha L_t + \(1- \alpha\)L_{t+1}\)},
\end{split}
\end{equation*}
where $\alpha$ is a constant between 0 and 1. According to Proposition~\ref{prop:bound:reABL}, $L_t$ and $L_{t+1}$ are uniformly bounded, so $\alpha L_t + \(1- \alpha\)L_{t+1}$ is uniformly bounded. According to Proposition~\ref{prop:bound:reABL}, $\frac{1}{t+1}A_t$, $\frac{t}{t+1}\br_{t+1}\br_{t+1}\trans$, $\frac{1}{t+1}B_t$ and $ \frac{t}{t+1}\(\bz_{t+1}-\be_{t+1}\)\br_{t+1}\trans$ are all uniformly bounded. Thus, there exists a constant $c$, such that
\begin{equation*}
\fronorm{\nabla G_t \(\alpha L_t + \(1- \alpha\)L_{t+1}\)} \leq \frac{c}{t},
\end{equation*}
resulting that
\begin{equation*}
G_t(L_{t+1}) - G_t(L_t) \leq \frac{c}{t} \fronorm{L_{t+1} - L_t}.
\end{equation*}
Applying this property in Eq.~\eqref{eq:dif_hg}, we have
\begin{equation}
\label{eq:dif_hg2}
\hg_t(L_{t+1}) - \hg_t(L_t) \leq G_t(L_{t+1}) - G_t(L_t) \leq \frac{c}{t} \fronorm{L_{t+1} - L_t}.
\end{equation}
From Eq.~\eqref{eq:dif_hg1} and Eq.~\eqref{eq:dif_hg2}, we conclude that
\begin{equation}
\fronorm{L_{t+1} - L_t} \leq \frac{2c}{\beta_1}\cdot \frac{1}{t}.
\end{equation}

\end{proof}

\begin{theorem}[Convergence of the empirical and expected loss]
%	\label{thm:convergence of f(L_t)}
Let $\{f(L_t)\}_{t=1}^{\infty}$ be the sequence of the expected loss where $\{L_t\}_{t=1}^{\infty}$ be the sequence of the solutions produced by the Algorithm~\ref{alg:all}. Also for any $t>0$, denote $g_t(L_t) - f_t(L_t)$ by $b_t$. Then,
\begin{enumerate}
	\item The sequence $\{b_t\}_{t=1}^{\infty}$ converges almost surely to 0.
		
	\item The sequence of the empirical loss $\{f_t(L_t)\}_{t=1}^{\infty}$ converges almost surely.
		
	\item The sequence of the expected loss $\{f(L_t)\}_{t=1}^{\infty}$ converges almost surely to the same limit of the surrogate $\{g_t(L_t)\}_{t=1}^{\infty}$.
		
\end{enumerate}
\end{theorem}
\begin{proof}
We start our proof by deriving an upper bound for $g_t(L_t) - f_t(L_t)$.

\textbf{Step 1:}
According to Eq.~\eqref{eq:diff_ut},
\begin{equation*}
\begin{split}
 \frac{b_t}{t+1} &= g_{t+1}(L_{t+1}) - g_{t+1}(L_t) + \frac{\ell(\bz_{t+1}, L_t) - f_t(L_t)}{t+1} + u_t - u_{t+1}\\
 &\leq \frac{\ell(\bz_{t+1}, L_t) - f_t(L_t)}{t+1} + u_t - u_{t+1}.
 \end{split}
\end{equation*}
Taking the expectation conditioned on the past information $\mathcal{F}_t$ in the above equation, and note that
\begin{equation*}
\begin{split}
\EXP[\frac{b_t}{t+1} \mid \mathcal{F}_t] &= \frac{g_t(L_t) - f_t(L_t)}{t+1},\\
\EXP[\frac{\ell(\bz_{t+1}, L_t) - f_t(L_t)}{t+1} \mid \mathcal{F}_t] &= \frac{f(L_t) - f_t(L_t)}{t+1},
\end{split}
\end{equation*}
we have
\begin{equation*}
\frac{b_t}{t+1} \leq  \frac{f(L_t) - f_t(L_t)}{t+1} + \EXP[u_t - u_{t+1} \mid \mathcal{F}_t].
\end{equation*}
Thus,
\begin{equation*}
\begin{split}
\sum_{t=1}^{\infty} \frac{b_t}{t+1} &\leq \sum_{t=1}^{\infty} \frac{f(L_t) - f_t(L_t)}{t+1} + \sum_{t=1}^{\infty} \EXP[u_t - u_{t+1} \mid \mathcal{F}_t]\\
\end{split}
\end{equation*}
According to the central limit theorem, $\sqrt{t}(f(L_t) - f_t(L_t))$ is bounded almost surely. Also, from Eq.~\eqref{eq:sum diff_ut},
\begin{equation*}
\sum_{t=1}^{\infty} \EXP[u_t - u_{t+1} \mid \mathcal{F}_t] \leq \sum_{t=1}^{\infty} \lv \EXP[u_t - u_{t+1} \mid \mathcal{F}_t] \rv < +\infty.
\end{equation*}
Thus,
\begin{equation*}
\sum_{t=1}^{\infty} \frac{b_t}{t+1} < +\infty.
\end{equation*}

\textbf{Step 2:}
We examine the difference between $b_{t+1}$ and $b_t$:
\begin{equation*}
\begin{split}
&\lv b_{t+1} - b_t \rv\\
 =& \lv g_{t+1}(L_{t+1}) -  f_{t+1}(L_{t+1}) - g_t(L_t) + f_t(L_t) \rv \\
\leq& \lv  g_{t+1}(L_{t+1}) - g_t(L_t) \rv + \lv f_{t+1}(L_{t+1}) - f_t(L_t) \rv\\
=& \lv g_{t+1}(L_{t+1}) - g_t(L_{t+1}) + g_t(L_{t+1}) - g_t(L_t) \rv + \lv f_{t+1}(L_{t+1}) - f_t(L_{t+1}) + f_t(L_{t+1}) - f_t(L_t) \rv \\
\leq&  \lv g_{t+1}(L_{t+1}) - g_t(L_{t+1}) \rv + \lv g_t(L_{t+1}) - g_t(L_t) \rv + \lv f_{t+1}(L_{t+1}) - f_t(L_{t+1}) \rv + \lv f_t(L_{t+1}) - f_t(L_t) \rv \\
=& \lv  \frac{1}{t+1}\ell(\bz_{t+1}, L_{t+1}) - \frac{1}{t+1}g_t(L_{t+1}) \rv + \lv g_t(L_{t+1}) - g_t(L_t) \rv \\
&+ \lv \frac{1}{t+1}\ell(\bz_{t+1}, L_{t+1}) - \frac{1}{t+1}f_t(L_{t+1}) \rv + \lv f_t(L_{t+1}) - f_t(L_t) \rv. \\
%\leq& \frac{1}{t+1}\( g_t(L_{t+1}) + f_t(L_t+1) + 2\ell(\bz_{t+1}, L_{t+1}) \) + \lv g_t(L_{t+1}) - g_t(L_t) \rv + \lv f_t(L_{t+1}) - f_t(L_t) \rv.
\end{split}
\end{equation*}
According to Corollary~\ref{coro:bound l lip gt} and Corollary~\ref{coro:bound lip ft}, we know that there exist constant $\kappa_1$ and $\kappa_2$ that are uniformly over $t$, such that
\begin{equation*}
\begin{split}
\lv g_t(L_{t+1}) - g_t(L_t) \rv &\leq \kappa_1 \fronorm{L_{t+1} - L_t},\\
\lv f_t(L_{t+1}) - f_t(L_t) \rv &\leq \kappa_2 \fronorm{L_{t+1} - L_t}.
\end{split}
\end{equation*}
Combing with Proposition~\ref{prop:diff_L}, there exists a constant $\kappa_3$ that is uniformly over $t$, such that
\begin{equation*}
\lv g_t(L_{t+1}) - g_t(L_t) \rv + \lv f_t(L_{t+1}) - f_t(L_t) \rv \leq \frac{\kappa_3}{t}.
\end{equation*}
As we shown, $\ell(\bz_{t+1}, L_{t+1})$, $g_t(L_{t+1})$ and $f_t(L_{t+1})$ are all uniformly bounded. Therefore, there exists a constant $\kappa_4$, such that
\begin{equation*}
\lvert \ell(\bz_{t+1}, L_{t+1}) - g_t(L_{t+1}) \rvert + \lvert \ell(\bz_{t+1}, L_{t+1}) - f_t(L_t+1) \rvert  \leq \kappa_4.
\end{equation*}
Finally, we have
\begin{equation*}
b_{t+1} - b_t \leq \frac{\kappa_4}{t+1} + \frac{\kappa_3}{t} \leq \frac{\kappa_5}{t},
\end{equation*}
where $\kappa_5$ is a constant that is uniformly over $t$.

Applying Lemma~\ref{lem:mairal}, we conclude that $\{b_t\}$ converges to zero. That is,
\begin{equation}
\lim_{t \rightarrow +\infty} g_t(L_t) - f_t(L_t) = 0.
\end{equation}

In Theorem~\ref{thm:convergence gt(Lt)}, we have shown that $g_t(L_t)$ converges almost surely. This implies that $f_t(L_t)$ also converges almost surely to the same limit of $g_t(L_t)$.

According to the central limit theorem, $\sqrt{t}(f(L_t)- f_t(L_t)$ is bounded, implying
\begin{equation*}
\lim_{t \rightarrow +\infty} f(L_t)- f_t(L_t) = 0, \quad a.s.
\end{equation*}
Thus, we conclude that $f(L_t)$ converges almost surely to the same limit of $f_t(L_t)$ (or, $g_t(L_t)$).

\end{proof}

%
%Equipped with Proposition~\ref{prop:diff_L} Proposition~\ref{prop:gt ft Lipschitz}, we can verify that the difference of the sequence $b_t = g_t(L_t) - f_t(L_t)$ can be upper bounded by $O(\frac{1}{t})$. The convergence of the summation of the serial $\{a_t b_t\}_{t=1}^{\infty}$ can be examined by the expectation convergence property of quasi-martingale $g_t(L_t)$, stated in~\cite{bottou1998online}. Applying the Lemma 8 from~\cite{mairal2010online}, we conclude that $g_t(L_t) - f_t(L_t)$ converges to zero almost surely.
%
%After the first claim of Theorem~\ref{thm:convergence of f(L_t)} being proved, the second claim follows immediately, as $g_t(L_t)$ converges almost surely (Theorem~\ref{thm:convergence gt(Lt)}). By taking the central limit theorem, the third claim can easily be verified.

\subsection{Finalizing the Proof}
According to Theorem~\ref{thm:convergence of f(L_t)}, we can  see that $g_t(L_t)$ and $f(L_t)$ converge to the same limit almost surely. Let $t$ tends to infinity, as $L_t$ is uniformly bounded (Proposition~\ref{prop:bound:reABL}), the term $\frac{\lambda_1}{2t} \twoinfnorm{L_t}^2$ in $g_t(L_t)$ vanishes. Thus $g_t(L_t)$ becomes differentiable. On the other hand, we have the following proposition about the gradient of $f(L)$.

\begin{proposition}[Subgradient of $f(L)$]
\label{prop:gradient f}
Let $f(L)$ be the expected loss function defined in Eq.~\eqref{eq:f(L)}. Then, $f(L)$ is continuously differentiable and $\nabla f(L) = \EXP_{\bz}[\nabla_L \ell(\bz, L)]$. Moreover, $\nabla f(L)$ is uniformly Lipschitz on $\mathcal{L}$.
\end{proposition}
\begin{proof}
Since $\ell(\bz, L)$ is continuously differentiable (Proposition~\ref{prop:l:Lipschtiz}), $f(L)$ is continuously differentiable and $\nabla f(L) = \EXP_{\bz}[\nabla_L \ell(\bz, L)]$.

Now we prove the second claim. Let us consider a matrix $L$ and a sample $\bz$, and denote $\br^*(\bz, L)$ and $\be^*(\bz, L)$ as the optimal solutions for Eq.~\eqref{eq:l(z,L)}.

\textbf{Step 1:}
First, $\tilde{\ell}(\bz, L, \br, \be)$ is continuous in $\bz$, $L$, $\br$ and $\be$, and has a unique minimizer. This implies that $\br^*(\bz, L)$ and $\be^*(\bz, L)$ is continuous in $\bz$ and $L$.

Let us denote $\varLambda$ as the set of the indices such that $\forall j \in \varLambda$, $\be_j^* \neq 0$. According to the first order optimal condition for Eq.~\eqref{eq:solve_re} w.r.t $\be$, we have
\begin{equation*}
\begin{split}
\bz - L \br - \be \in \lambda_2 \partial \onenorm{\be}, \\
\Rightarrow \lv (\bz - L \br - \be)_j \rv = \lambda_2,\ \forall j \in \varLambda.
\end{split}
\end{equation*}

Since $\bz - L \br - \be$ is continuous in $\bz$ and $L$, we consider a small perturbation of $(\bz, L)$ in one of their open neighborhood $V$, such that for all $(\bz', L') \in V$, we have if $ j \notin \varLambda$, then $\lv (\bz' - L' {\br^*}' - {\be^*}')_j \rv < \lambda_2$ and ${\be^*}'_j = 0$, where ${\br^*}' = \br^*(\bz', L')$ and ${\be^*}' = \be^*(\bz', L')$. That is, the support set of $\be^*$ does not change.

Let us denote $D = [L\ I]$ and $\bb=[\br;\ \be]$ and consider the function

\begin{equation*}
\tilde{\ell}(\bz, L^{}_{\varLambda}, \bb^{}_{\varLambda}) \defeq \frac{1}{2} \twonorm{\bz - D^{}_{\varLambda}\bb^{}_{\varLambda}}^2 + \lambda_2 \onenorm{[0\ I] \bb^{}_{\varLambda}}.
\end{equation*}

According to Assumption~\ref{as:unique}, $\tilde{\ell}(\bz,  L^{}_{\varLambda}, \cdot)$ is strongly convex with a Hessian lower-bounded by a positive constant $\kappa_1$. Thus,
\begin{equation}
\label{eq:geq}
\begin{split}
\tilde{\ell}(\bz, L^{}_{\varLambda}, \bb'^*_{\varLambda}) - \tilde{\ell}(\bz, L^{}_{\varLambda}, \bb^*_{\varLambda}) \geq & \kappa_1 \twonorm{\bb^{}_{\varLambda} - \bb'_{\varLambda}}^2 \\
= & \kappa_1 \( \twonorm{\br^* - \br'^*}^2 + \twonorm{\be^*_{\varLambda} - \be'^*_{\varLambda}}^2 \)
\end{split}
\end{equation}

\textbf{Step 2:}
We shall prove that $\tilde{\ell}(\bz, L, \cdot) - \tilde{\ell}(\bz', L', \cdot)$ is Lipschitz w.r.t. $\bb$.
\begin{equation*}
\begin{split}
& 2\times \( \tilde{\ell}(\bz, L, \bb) - \tilde{\ell}(\bz', L', \bb) \) - 2\times \(\tilde{\ell}(\bz, L, \bb') - \tilde{\ell}(\bz', L', \bb') \)\\
=& \Vert \bz - D\bb \Vert_2^2 - \Vert \bz - D\bb' \Vert_2^2 + \Vert \bz' - D'\bb' \Vert_2^2 - \Vert \bz' - D'\bb \Vert_2^2 \\
=& 2\bz\trans D (\bb' - \bb) + \bb\trans D\trans D \bb - \bb'^{\top} D\trans D \bb' -2 \bz'^{\top} D'(\bb'-\bb) - \bb\trans D'^{\top} D'\bb + \bb'^{\top} D'^{\top} D'\bb' \\
=& 2 [ (\bz\trans D - \bz'^{\top} D')(\bb'-\bb) ] + [ \bb\trans D\trans D\bb - \bb\trans D'^{\top} D'\bb + \bb'^{\top} D'^{\top} D'\bb' - \bb'^{\top} D\trans D\bb' ] \\
\end{split}
\end{equation*}
For the first term,
\begin{equation*}
\begin{split}
&(\bz\trans D - \bz'^{\top} D')(\bb'-\bb) \\
=& (\bz\trans D - \bz\trans D' + \bz\trans D' - \bz'^{\top} D'^{\top})(\bb'-\bb) \\
=& \(\bz\trans(D-D') + (\bz\trans - \bz'^{\top})D' \)(\bb'-\bb)
\end{split}
\end{equation*}
As each sample is bounded, $D$ is bounded (as $L$ is bounded), so the $\ell_2$-norm of the first term can be bounded as follows:
\begin{equation}
\label{eq:t1}
\begin{split}
& \Vert (\bz\trans D - \bz'^{\top} D')(\bb'-\bb) \Vert_2 \\
=& \Vert \(\bz\trans(D-D') + (\bz\trans - \bz'^{\top})D' \)(\bb'-\bb) \Vert_2\\
\leq& \( \Vert \bz \Vert_2 \Vert D - D' \Vert_F + \Vert \bz - \bz' \Vert_2 \Vert D' \Vert_F \) \cdot \Vert \bb'-\bb \Vert_2 \\
\leq& \( c_1   \Vert D - D' \Vert_F + c_2 \Vert \bz - \bz' \Vert_2  \) \cdot \Vert \bb'-\bb \Vert_2 
\end{split}
\end{equation}

For the second term,
\begin{equation*}
\begin{split}
&  \bb\trans D\trans D\bb - \bb\trans D'^{\top} D'\bb + \bb'^{\top} D'^{\top} D'\bb' - \bb'^{\top} D\trans D\bb'\\
=& \bb\trans\(D\trans D - D'^{\top} D'\)\bb - \bb'^{\top}\(  D\trans D - D'^{\top} D' \)\bb' \\
=& \bb\trans\(D\trans D - D'^{\top} D'\)\bb -  \bb\trans\(D\trans D - D'^{\top} D'\)\bb' +  \bb\trans\(D\trans D - D'^{\top} D'\)\bb' - \bb'^{\top}\(  D\trans D - D'^{\top} D' \)\bb' \\
=& \bb\trans\( D\trans D - D'^{\top} D' \)\( \bb - \bb' \) + \( \bb - \bb' \)\trans \( D\trans D - D'^{\top} D' \)\bb' \\
=& \bb\trans\( D\trans D -D\trans D' +D\trans D'  - D'^{\top} D' \)\( \bb - \bb' \) + \( \bb - \bb' \)\trans \( D\trans D -D\trans D' +D\trans D' - D'^{\top} D' \)\bb' \\
=& \bb\trans\( D\trans\(D-D'\) + \(D\trans - D'\)D' \)\( \bb - \bb' \) +  \( \bb - \bb' \)\trans \( D\trans\(D-D'\) + \(D\trans - D'\)D' \)\bb' 
 \end{split}
\end{equation*}
Since $D$ is bounded, $\bb$ is bounded, the second term can be bounded as follows:
\begin{equation}
\label{eq:t2}
\begin{split}
& \Vert \bb\trans D\trans D\bb - \bb\trans D'^{\top} D'\bb + \bb'^{\top} D'^{\top} D'\bb' - \bb'^{\top} D\trans D\bb' \Vert_2\\
=& \Vert  \bb\trans\( D\trans\(D-D'\) + \(D\trans - D'^{\top}\)D' \)\( \bb - \bb' \) +  \( \bb - \bb' \)\trans \( D\trans\(D-D'\) + \(D\trans - D'^{\top}\)D' \)\bb'  \Vert_2 \\
\leq& c_3 \Vert D-D' \Vert_F \cdot \Vert \bb - \bb' \Vert_2
\end{split}
\end{equation}
Combining \eqref{eq:t1} and \eqref{eq:t2}, we prove that $\tilde{\ell}(\bz, L, \cdot) - \tilde{\ell}(\bz', L', \cdot)$ is Lipschitz with constant $\(c_1 + c_3 \)  \Vert D - D' \Vert_F + c_2 \Vert \bz - \bz' \Vert_2 $:
\begin{equation}
\label{eq:leq}
\begin{split}
 & \( \tilde{\ell}(\bz, L, \bb) - \tilde{\ell}(\bz', L', \bb) \) -  \(\tilde{\ell}(\bz, L, \bb') - \tilde{\ell}(\bz', L', \bb') \) \\
 \leq& \(\(c_1 + c_3 \)  \Vert D - D' \Vert_F + c_2 \Vert \bz - \bz' \Vert_2 \) \Vert \bb - \bb' \Vert_2 \\
 =& \(\(c_1 + c_3 \)  \Vert D - D' \Vert_F + c_2 \Vert \bz - \bz' \Vert_2 \) \sqrt{\Vert \br - \br' \Vert_2^2 + \Vert \be - \be' \Vert_2^2}
\end{split}
\end{equation}

\textbf{Step 3:}
According to Eq.~\eqref{eq:geq} and Eq.~\eqref{eq:leq}, and notice that $\bb'^*$ minimizes $\tilde{\ell}(\bz', L', \cdot)$, we have
\begin{equation*}
\begin{split}
&\kappa_1 \( \Vert \br^* - \br'^* \Vert_2^2 + \Vert \be^*_{\varLambda} - \be'^*_{\varLambda} \Vert_2^2 \)\\
\leq& \tilde{\ell}(\bz, L^{}_{\varLambda}, \bb'^*_{\varLambda}) - \tilde{\ell}(\bz, L^{}_{\varLambda}, \bb^*_{\varLambda})\\
=&  \tilde{\ell}(\bz, L^{}_{\varLambda}, \bb'^*_{\varLambda}) - \tilde{\ell}(\bz', L'_{\varLambda}, \bb^*_{\varLambda})  + \tilde{\ell}(\bz', L'_{\varLambda}, \bb^*_{\varLambda})  - \tilde{\ell}(\bz, L^{}_{\varLambda}, \bb^*_{\varLambda}) \\
\leq& \tilde{\ell}(\bz, L^{}_{\varLambda}, \bb'^*_{\varLambda}) - \tilde{\ell}(\bz', L'_{\varLambda}, \bb'^*_{\varLambda})  + \tilde{\ell}(\bz', L'_{\varLambda}, \bb^*_{\varLambda})  - \tilde{\ell}(\bz, L^{}_{\varLambda}, \bb^*_{\varLambda}) \\
\leq&  \(\(c_1 + c_3 \)  \Vert D - D' \Vert_F + c_2 \Vert \bz - \bz' \Vert_2 \) \sqrt{\Vert \br^* - \br'^* \Vert_2^2 + \Vert \be^*_{\varLambda} - \be'^*_{\varLambda} \Vert_2^2}
\end{split}
\end{equation*}

Therefore, $\br^*(\bz, L)$ and $\be^*(\bz, L)$ are Lipschitz, which concludes the proof.

\end{proof}

Finally, taking a first order Taylor expansion for $f(L_t)$ and $g_t(L_t)$, we can show that the gradient of $f(L_t)$ equals to that of $g_t(L_t)$ when $t$ tends to infinity. Since  $L_t$ is the minimizer for $g_t(L)$, we know that the gradient of $f(L_t)$ vanishes. Therefore, we have proved Theorem~\ref{thm:stationary}.

\begin{proof}
According to Proposition~\ref{prop:bound:reABL}, the sequences $\{\frac{1}{t}A_t\}$ and $\{\frac{1}{t}B_t\}$ are uniformly bounded. Then, there exist sub-sequences of $\{\frac{1}{t}A_t\}$ and $\{\frac{1}{t}B_t\}$ that converge to $A_{\infty}$ and $B_{\infty}$ respectively.
In that case, $L_t$ converges to $L_{\infty}$. Let $V$ be an arbitrary matrix in $\Rpd$, and $\{h_k\}$ be a positive sequence that converges to zero.

Since $g_t$ is the surrogate function of $f_t$, for all $t$ and $k$, we have
\begin{equation*}
g_t(L_t + h_k V) \geq f_t(L_t + h_k V).
\end{equation*}
Let $t$ tend to infinity:
\begin{equation*}
g_{\infty}(L_{\infty} + h_k V) \geq f(L_{\infty} + h_k V).
\end{equation*}
Since $L_t$ is uniformly bounded, when $t$ tends to infinity, the term $\lambdat \Vert L_t \Vert_{\infty}^2$ will vanish. In this way, $g_t(\cdot)$ becomes differentiable. Also, the Lipschitz of $\nabla f(L)$ (proved in Proposition~\ref{prop:gradient f}) implies that the second derivative of $f(L_t)$ can be uniformly bounded. And by a simple calculation, this also holds for $g_t(L_t)$. Thus, we can take the first order Taylor expansion even when $t$ tends to infinity.
Using a first order Taylor expansion, and note the fact that $g_{\infty}(L_{\infty}) = f(L_{\infty})$, we have
\begin{equation*}
\tr(h_k V\trans \nabla g_{\infty}(L_{\infty})) + o(h_k V) \geq \tr(h_k V\trans \nabla f(L_{\infty})) + o(h_k V).
\end{equation*}
Since $\{h_k\}$ is a positive sequence, by multiplying $\frac{1}{h_k \Vert  V \Vert_F}$ on both side, it follows that
\begin{equation*}
\tr(\frac{1}{\Vert V \Vert_F} V\trans \nabla g_{\infty}(L_{\infty})) + \frac{o(h_k V)}{h_k {\Vert V \Vert_F} } \geq \tr(\frac{1}{\Vert V \Vert_F}  V\trans \nabla f(L_{\infty})) + \frac{o(h_k V)}{h_k {\Vert V \Vert_F} } .
\end{equation*}
Now let $k$ tend to infinity:
\begin{equation*}
\tr(\frac{1}{\Vert V \Vert_F} V\trans \nabla g_{\infty}(L_{\infty}))  \geq \tr(\frac{1}{\Vert V \Vert_F}  V\trans \nabla f(L_{\infty})).
\end{equation*}

Since the inequality holds for all matrix $V \in \Rpd$, it can easily show that
\begin{equation*}
\nabla g_{\infty}(L_{\infty}) = \nabla f(L_{\infty}).
\end{equation*}
Since $L_t$ always minimizes $g_t(\cdot)$, we have
\begin{equation*}
\nabla f(L_{\infty}) = \nabla g_{\infty}(L_{\infty}) = 0,
\end{equation*}
which implies that when $t$ tend to infinity, $L_t$ is a stationary point of $f(\cdot)$.

\end{proof}